\documentclass{article}

\usepackage[preprint]{main}

\usepackage[utf8]{inputenc} % allow utf-8 input
\usepackage[T1]{fontenc}    % use 8-bit T1 fonts
\usepackage{hyperref}       % hyperlinks
\usepackage{url}            % simple URL typesetting
\usepackage{booktabs}       % professional-quality tables
\usepackage{amsfonts}       % blackboard math symbols
\usepackage{nicefrac}       % compact symbols for 1/2, etc.
\usepackage{microtype}      % microtypography
\usepackage{xcolor}         % colors

\usepackage{graphicx}
\usepackage{subcaption}

\usepackage{wrapfig}

\usepackage{amsmath} 
\usepackage{algorithm}
\usepackage{algorithmic}

\usepackage{amsmath}
\usepackage{amssymb}
\usepackage{mathtools}
\usepackage{amsthm}

\theoremstyle{plain}
\newtheorem{theorem}{Theorem}[section]
\newtheorem{proposition}[theorem]{Proposition}
\newtheorem{lemma}[theorem]{Lemma}
\newtheorem{corollary}[theorem]{Corollary}
\theoremstyle{definition}
\newtheorem{definition}[theorem]{Definition}
\newtheorem{assumption}[theorem]{Assumption}
\theoremstyle{remark}
\newtheorem{remark}[theorem]{Remark}

\newtheorem{example}[theorem]{Example}

\newcommand{\R}{\mathbb{R}}
\newcommand{\E}{\mathbb{E}}

\newcommand{\cS}{\mathcal{S}}
\newcommand{\cA}{\mathcal{A}}
\newcommand{\cX}{\mathcal{X}}
\newcommand{\cG}{\mathcal{G}}

\newcommand{\cV}{\mathcal{V}}
\newcommand{\cE}{\mathcal{E}}

\newcommand{\cN}{\mathcal{N}}
\newcommand{\cL}{\mathcal{L}}

\newcommand{\ip}[2]{\left\langle #1,#2\right\rangle}
\newcommand{\norm}[1]{\left\|#1\right\|}

\newcommand{\Circ}{\mathbin{\circ}}

\newcommand{\diag}{\operatorname{diag}}
\newcommand{\Proj}{\operatorname{Proj}}
\newcommand{\Retr}{\operatorname{Retr}}
\newcommand{\dist}{\operatorname{Dist}}

\newcommand{\Gap}{\operatorname{Gap}}

\usepackage{graphics}

\title{
Metric-Gradient Projection for Stable Multi-Agent Policy Learning}

\author{%
    Zuyuan Zhang\\
    The George Washington University\\
    \texttt{zuyuan.zhang@gwu.edu}\\
    \And
    Sizhe Tang\\
    The George Washington University\\
    \texttt{s.tang1@gwu.edu}\\
    \And
    Mahdi Imani\\
    Northeastern University\\
    \texttt{m.imani@northeastern.edu}
    \And
    Tian Lan\\
    The George Washington University\\
    \texttt{tlan@gwu.edu}
}

\begin{document}

\maketitle

\begin{abstract}
General-sum multi-agent learning is often governed by a stacked update field in which each agent's policy update changes the optimization landscape faced by the others. This coupling can entangle an integrable component of collective improvement with cyclic interaction dynamics, leading to slow or unstable multi-agent learning. Existing approaches, such as regularization, credit assignment, and consensus methods, stabilize MARL through local or algorithmic modifications; HPML complements them by projecting the joint update field onto a metric-gradient component. We introduce \textbf{HPML} (\textbf{H}odge-\textbf{P}rojected \textbf{M}ulti-agent \textbf{L}earning), which views the joint update field of a multi-agent system as an element of an $L^2$ space of vector fields and computes a Hodge-type projection onto the closest metric-gradient potential flow. HPML follows the projected component as the update direction, yielding the closest metric-gradient field under the chosen metric and sampling measure. The projection is defined variationally, characterized by a Poisson-type equation, and implemented through graph-based and amortized neural realizations that recover projected directions from samples. We show that the projected dynamics admit a Lyapunov potential and yield equilibrium-gap bounds with an explicit additive non-potentiality term. Controlled experiments validate the geometric mechanism, and CTDE benchmarks show improved stability and normalized return when HPML is used as a plug-in projection layer in MARL pipelines.
\end{abstract}

\section{Introduction}
\label{sec:intro}

Reinforcement learning (RL) is a standard framework for sequential decision-making, and its multi-agent extension addresses networked control, robotic coordination, and strategic interaction \citep{puterman2014markov,sutton1998reinforcement,shapley1953stochastic,littman1994markov,bucsoniu2010multi,zhang2025learning,qiao2024br,zhang2025network,zhang2024modeling}. 
In single-agent RL, optimization is typically organized around a scalar objective, which supports Lyapunov-style analysis under standard smoothness conditions \citep{puterman2014markov,bertsekas1997nonlinear}. 
General-sum multi-agent learning is driven by a stacked update field that may contain non-potential components and cyclic interaction dynamics even near stationary points \citep{balduzzi2018mechanics,monderer1996potential,zhang2025lipschitz,zhang2026geometry}.

This issue is particularly visible in the centralized-training/decentralized-execution (CTDE) paradigm. PPO-style MARL methods produce decentralized actor updates from centralized critics or advantages~\citep{schulman2015trust,schulman2017proximal,schulman2015high,yu2022surprising}, and actor-critic or value-factorization baselines such as MADDPG, COMA, QMIX, and VDN induce coupled evaluations and updates across agents~\citep{lowe2017multi,foerster2018counterfactual,rashid2020monotonic,sunehag2017value}. Since each agent experiences non-stationary dynamics due to the concurrent learning and
policy updates of other agents, the stacked actor-update field can contain asymmetric cross-agent response terms. Even when each local update is well estimated, the stacked field can contain a non-potential component that carries closed-loop circulation. Local Jacobian antisymmetry diagnoses this effect pointwise; HPML instead targets a global projected update over the region visited by learning ~\citep{balduzzi2018mechanics,mescheder2017numerics,gidel2018variational}. Such cyclic components can slow learning, destabilize discrete-time updates, and drive recurrent policy profiles or failed coordination patterns~\citep{DBLP:conf/nips/ZinkevichGL05,balduzzi2018mechanics,monderer1996potential}.

We propose \textbf{HPML}, a Hodge-inspired projection layer for multi-agent update fields. HPML views the supplied joint update field as an element of an $L^2(M,\mu)$ space of vector fields and projects it onto the closest metric-gradient potential flow. The orthogonal residual is the non-potential component left unexplained by any metric-gradient flow, and it carries closed-loop circulation~\citep{hodge1989theory,eckmann1944harmonische,zhang2024modeling}. The projected field can be used as a plug-in update direction in MARL actor-update pipelines. This gives a geometric intervention complementary to regularization, credit assignment, consensus, extra-gradient, Mirror-Prox, optimistic, and momentum-based methods \citep{korpelevich1977extragradient,nemirovski2004prox,popov1980modification,mertikopoulos2018optimistic,mescheder2017numerics,gidel2018variational,daskalakis2017training}.

The projection admits a Poisson-type characterization, a graph-Laplacian realization from samples~\citep{chung1997spectral,von2007tutorial,spielman2010algorithms,hawrylycz2012discrete,lim2020hodge,zhang2026cochain,zhang2026structuring,zhang2026operator}, and an amortized neural realization that recovers projected directions by automatic differentiation. 
The projected dynamics admit a Lyapunov potential, and the residual non-potential component contributes an explicit additive term to the Variational Inequality (VI) gap \citep{facchinei2003finite,nemirovski2004prox}. In the CTDE paradigm, this residual measures the obstruction to a joint greedy-improvement interpretation. The main contributions are a global potential-flow projection for MARL, continuous/discrete/neural realizations, VI-gap guarantees with explicit residual terms, and a CTDE integration compatible with PPO/MAPPO-style actor-update pipelines. Controlled and Melting Pot experiments show improved stability and final normalized return.

\section{Preliminaries}
\label{sec:prelim}

This section provides the policy-space operator notation used throughout the paper.
We define the feasible policy set, the VI-based equilibrium metric, and a local Jacobian-based integrability diagnostic that motivates the global projection developed later.
Additional remarks on operator choice, unconstrained parameterizations, VI variants, and the canonical bilinear example are deferred to Appendix~\ref{app:prelim_notes}; see in particular Remarks~\ref{rem:operator_app}, \ref{rem:param_app}, \ref{rem:minty_app}, and Example~\ref{ex:bilinear_app}.

\subsection{Markov Games and Policy Space}
\label{subsec:mg}

We consider a discounted finite Markov game with $n$ agents
$
\cG = \big(\cS,\{\cA_i\}_{i=1}^n,P,\{r_i\}_{i=1}^n,\gamma\big),
\qquad \gamma\in(0,1),
$
where $\cS$ is a finite state space, $\cA_i$ is agent $i$'s finite action set, $P(\cdot\mid s,a)$ is the transition kernel, and $r_i(s,a)$ is agent $i$'s per-stage reward.
A stationary policy for agent $i$ is a mapping $\pi_i(\cdot\mid s)\in\Delta(\cA_i)$, and we write $\pi=(\pi_1,\ldots,\pi_n)$ for a joint policy.
Let
$
\Pi := \prod_{i=1}^n \prod_{s\in\cS}\Delta(\cA_i)
$
denote the stationary-policy set, viewed through the standard coordinate embedding as a compact convex set $\cX\subset\R^d$.
Each agent's discounted return under $\pi\in\Pi$ is
$
J_i(\pi)
=
\E_{\pi}\Big[\sum_{t\ge 0}\gamma^t r_i(s_t,a_t)\Big].
$

We write $F:\cX\to\R^d$ for a continuous joint field on policy space.
A canonical game-theoretic choice is the negative pseudo-gradient operator
$
F(\pi)
:=
\big(-\nabla_{\pi_1}J_1(\pi),\ldots,-\nabla_{\pi_n}J_n(\pi)\big),
$
which is the operator used in the VI formulation below.
More generally, HPML can be applied to any continuous joint update field defined on $\cX$.
To avoid sign ambiguity, we distinguish the VI operator $F_{\mathrm{VI}}$ from an ascent/update field $U$.
For the canonical game operator above, $U=-F_{\mathrm{VI}}=(\nabla_{\pi_i}J_i)_i$.
The projection layer is applied to the supplied update field $U$; if one instead starts from a VI operator, the projected update direction is obtained by applying HPML to $-F_{\mathrm{VI}}$.
In general-sum games, such a field need not be the gradient of any single scalar potential on $\cX$; this lack of global integrability is the source of the cyclic behavior that HPML targets.
See Remark~\ref{rem:operator_app} and Remark~\ref{rem:param_app} for the distinction between the equilibrium operator and implementation-level parameter updates.

\subsection{Variational Inequality and Duality Gap}
\label{subsec:vi}

Many first-order equilibrium notions for differentiable games can be expressed as a variational inequality (VI).
Let $\cX\subset\R^d$ be nonempty, closed, and convex, and let $F:\cX\to\R^d$.
Unless otherwise stated, $\ip{u}{v}:=u^\top v$ denotes the Euclidean inner product.

\begin{definition}[Stampacchia VI]
\label{def:vi}
The variational inequality $\mathrm{VI}(\cX,F)$ is to find $x^\star\in\cX$ such that
$
\ip{F(x^\star)}{x-x^\star}\ge 0,
\qquad \forall x\in\cX.
$
\end{definition}

When $x$ parametrizes a joint policy $\pi$ and
$
F(\pi)=(-\nabla_{\pi_i}J_i(\pi))_{i=1}^n,
$
a solution of $\mathrm{VI}(\cX,F)$ satisfies the standard first-order Nash condition for a differentiable game.
It is generally necessary for Nash equilibrium, and it becomes sufficient under additional structure such as playerwise concavity.

\begin{definition}[Duality gap (Stampacchia gap)]
\label{def:gap}
For $\bar x\in\cX$, define
$
\Gap(\bar x)
:=
\max_{x\in\cX}\ip{F(\bar x)}{\bar x-x}.
$
\end{definition}

We use $\Gap(\bar x)$ as the equilibrium-quality metric throughout.
By construction $\Gap(\bar x)\ge 0$, and $\Gap(\bar x)\le \varepsilon$ is equivalent to
$
\ip{F(\bar x)}{x-\bar x}\ge -\varepsilon,
\qquad \forall x\in\cX.
$

\begin{lemma}[Zero gap implies a VI solution]
\label{lem:gap_zero_vi}
If $\Gap(\bar x)=0$, then $\bar x$ solves $\mathrm{VI}(\cX,F)$ in the sense of Definition~\ref{def:vi}.
\end{lemma}

See Remark~\ref{rem:minty_app} and Remark~\ref{rem:monotone_gap_app} for the relation to the Minty gap and to monotonicity-based merit-function interpretations.

\subsection{Jacobian Symmetry and Local Integrability}
\label{subsec:jac_split}
Assume $F$ is differentiable on $\mathrm{ri}(\cX)$, and let $J(x):=\nabla F(x)\in\R^{d\times d}$.
We decompose
\[
J(x)=S(x)+A(x),
\qquad
S(x)=\tfrac12\big(J(x)+J(x)^\top\big),
\qquad
A(x)=\tfrac12\big(J(x)-J(x)^\top\big).
\]
The symmetric part $S(x)$ is the locally potential-like component of the linearization, while the antisymmetric part $A(x)$ is a local obstruction to integrability and a diagnostic for cyclic interaction dynamics.

\begin{definition}[Antisymmetric energy]
\label{def:rot_energy}
The antisymmetric energy at $x$ is
$
\mathcal R(x):=\norm{A(x)}_F^2,
$
where $\|\cdot\|_F$ denotes the Frobenius norm.
\end{definition}

\begin{proposition}[Gradient fields have symmetric Jacobians; a converse on open sets]
\label{prop:grad_symm}
(i) If there exists a twice-differentiable potential $\phi$ such that $F=\nabla\phi$ on an open set containing $\cX$, then $J(x)=\nabla^2\phi(x)$ is symmetric for all $x$ in that open set; hence $A(x)\equiv 0$ and $\mathcal R(x)\equiv 0$ on $\cX$.

(ii) Conversely, if $U\subset\R^d$ is open and simply connected, $F\in C^1(U)$, and $A(x)\equiv 0$ for all $x\in U$, then there exists a scalar potential $\phi$ on $U$ such that $F=\nabla\phi$ (unique up to an additive constant).
\end{proposition}

Proposition~\ref{prop:grad_symm} shows that antisymmetry in $J(x)$ is a local obstruction to integrability.
HPML will later replace this local diagnostic by a global $L^2$ projection at the vector-field level.
See Example~\ref{ex:bilinear_app} for the canonical bilinear cyclic-interaction example and Remark~\ref{rem:localglobal_app} for the local/global distinction.

\section{Hodge-Projected Multi-Agent Learning}
\label{sec:hpml}

This section defines HPML as a metric-aware least-squares projection of the joint update field onto the class of potential flows.
The main objects used later are the projected field, its orthogonality property, and the induced non-potential residual.
A Poisson-type characterization, circulation interpretation, and implementation-level details are deferred to Appendix~\ref{app:hpml_variational}; see Lemma~\ref{lem:adjoint}, Proposition~\ref{prop:poisson}, Proposition~\ref{prop:circulation_residual}, and Remark~\ref{rem:variable_metric}.

\subsection{Metric-Gradient Projection}
\label{subsec:closest_grad}
Let $\cX\subset\R^d$ be a nonempty closed convex set.
Throughout Sections~\ref{sec:hpml}--\ref{sec:algs}, $F$ denotes the supplied update field to be projected.
All continuous projection statements are understood in coordinates on the affine hull of $\cX$; for product-simplex policy spaces this can be implemented by dropping redundant simplex coordinates or by working in an unconstrained parameterization such as logits.
Fix a symmetric positive definite matrix $M\succ 0$ and a reference probability measure $\mu$ on $\cX$.
For vectors $u,v\in\R^d$, define
$
\langle u,v\rangle_M := u^\top M v,
\qquad
\|u\|_M := \sqrt{u^\top M u}.
$
For vector fields $U,V\in L^2(\cX,\mu;\R^d)$, define
$
\langle U,V\rangle_{L^2(M,\mu)}
:=
\E_{x\sim\mu}\big[\langle U(x),V(x)\rangle_M\big].
$

For a scalar potential $\Phi:\cX\to\R$, the metric gradient is
$
\nabla_M \Phi(x) := M^{-1}\nabla \Phi(x).
$

\begin{definition}[Metric-gradient least-squares projection]
\label{def:proj}
Let $F:\cX\to\R^d$ be square-integrable under $\mu$.
Define
$
\cG_{\mathrm{grad}}^M
:=
\{\nabla_M\Phi:\ \Phi\in H^1(\cX)\},
$
and define $\Pi_{\mathrm{grad}}^{M}F$ as the minimizer-induced field of
$
\inf_{\Phi\in H^1(\cX)}
\E_{x\sim\mu}\big[\|F(x)-\nabla_M\Phi(x)\|_M^2\big].
$
If $\Phi^\star$ is any minimizer, the projected field is
$
\Pi_{\mathrm{grad}}^{M}F := \nabla_M\Phi^\star,
$
which is unique $\mu$-a.e., while $\Phi^\star$ itself is defined up to an additive constant.
\end{definition}

Definition~\ref{def:proj} is the core object of HPML: it replaces the original joint field by its closest potential-flow approximation under the geometry induced by $M$ and the sampling distribution $\mu$.
We assume the continuous variational problem admits a minimizer; otherwise, $\Pi_{\mathrm{grad}}^M F$ can be interpreted as the projection onto the $L^2(M,\mu)$-closure of metric-gradient fields.
The finite graph and neural implementations used later are finite-dimensional approximations of this same objective.
A weak Poisson-type characterization of the same projection is given in Proposition~\ref{prop:poisson}.

\subsection{Orthogonality and Residual Energy}
\label{subsec:residual}

Let
$
G^\star := \Pi_{\mathrm{grad}}^{M}F,
\qquad
R := F-G^\star.
$
The projected field and residual satisfy the following variational orthogonality property.

\begin{lemma}[Normal equation / orthogonality]
\label{lem:orth}
For every test potential $\Phi\in H^1(\cX)$,
$
\E_{x\sim\mu}\big[\langle R(x),\nabla_M\Phi(x)\rangle_M\big]=0.
$
Equivalently, $R$ is $L^2(M,\mu)$-orthogonal to all metric-gradient fields.
\end{lemma}

\begin{corollary}[Pythagorean identity]
\label{cor:pythag}
With $G^\star$ and $R$ defined above,
$
\E_{x\sim\mu}\big[\|F(x)\|_M^2\big]
=
\E_{x\sim\mu}\big[\|G^\star(x)\|_M^2\big]
+
\E_{x\sim\mu}\big[\|R(x)\|_M^2\big].
$
\end{corollary}
Corollary~\ref{cor:pythag} separates the field into a potential component and an orthogonal non-potential residual.
This is the Hodge-type structure used by HPML: an $L^2(M,\mu)$-orthogonal split into a metric-gradient component and a residual. On domains or complexes with richer topology, the residual can be further refined into curl-like and harmonic components; on the sample graph, the same idea appears as a cut-space/cycle-space decomposition of edge flows. HPML uses the coarser projected-gradient/residual split because the update only requires the integrable direction. Both components are relative to the chosen metric $M$ and sampling measure $\mu$, so $E$ measures non-potentiality over the region represented by the data.

\subsection{Non-Potentiality}
\label{subsec:nonpot}

\begin{definition}[Non-potentiality]
\label{def:nonpot}
Define the residual field
$
R(x):=F(x)-\Pi_{\mathrm{grad}}^{M}F(x)
$
and its $L^2(M,\mu)$ magnitude
$
E
:=
\Big(
\E_{x\sim\mu}\big[\|R(x)\|_M^2\big]
\Big)^{1/2}.
$
We call $E$ the \emph{non-potentiality} of $F$ under metric $M$ and sampling measure $\mu$.
\end{definition}

The quantity $E$ is a global, geometry-aware measure of how far the joint update field deviates from a potential flow over the region emphasized by $\mu$.
If $E=0$, then $F$ agrees $\mu$-a.e.\ with a metric-gradient field; if $E>0$, the residual $R$ captures the irreducible non-potential component that can sustain cyclic interaction dynamics.
Closed-loop circulation is carried entirely by this residual; see Proposition~\ref{prop:circulation_residual} and Remark~\ref{rem:failed_loops}.

\subsection{HPML Update}
\label{subsec:update}

Let $\Retr:\R^d\to\cX$ be a feasibility operator.
HPML updates by following only the projected component:
\begin{equation}
x_{t+1}
=
\Retr\Big(x_t+\eta_t\,\Pi_{\mathrm{grad}}^{M}F(x_t)\Big).
\label{eq:hpml-update}
\end{equation}
In practice, the projection is computed approximately from samples by solving the same variational problem on a discrete structure or by fitting an amortized potential model.
The concrete procedures are given in Algorithm~\ref{alg:hodge-proj} and Algorithm~\ref{alg:neural-hpml}, with additional implementation remarks in Remark~\ref{rem:discrete_impl_app}.

\section{Sample-Graph Projection}
\label{sec:discrete}

Section~\ref{sec:hpml} defines the continuous projection variationally.
In practice, however, the field is observed only at sampled iterates, so we solve the same least-squares problem on a finite sample graph.
The main outputs of the discrete construction are a projected potential edge flow, an orthogonal residual, and a computable proxy for non-potentiality.
Detailed graph construction, lifting formulas, and the full procedure are deferred to Appendix~\ref{app:discrete_details}; see Proposition~\ref{prop:metric_discrete}, Remark~\ref{rem:graph_cycle_space}, Remark~\ref{rem:convention_graph}, and Algorithm~\ref{alg:hodge-proj}.

\subsection{Graph Samples and Edge Flow}
\label{subsec:sampling_graph}

Let $\{x^1,\dots,x^N\}\subset\cX$ be sampled joint-policy points.
We build a weighted undirected $k$-NN graph $G=(\cV,\cE)$ on these samples, with node set $\cV=\{1,\dots,N\}$ and symmetric edge weights $w_{ij}>0$.
Choose an orientation for each edge and let $\cE^\rightarrow$ denote the oriented edge set.

We represent scalar potentials as node values $\phi\in\R^N$.
Let $B\in\R^{|\cE^\rightarrow|\times N}$ be the oriented incidence matrix, so that the discrete gradient is the edge flow $B\phi$.

Given field estimates $F_i:=F(x^i)\in\R^d$, define the edge $1$-form
\begin{equation}
\omega_e
:=
\Big\langle
M\,\tfrac12(F_i+F_j),
\,x^j-x^i
\Big\rangle,
\qquad
e=(i\to j)\in\cE^\rightarrow.
\label{eq:edge_flow}
\end{equation}
This is the discrete analogue of the continuous $1$-form $\langle MF,dx\rangle$.
Its metric consistency with potential differences is formalized in Proposition~\ref{prop:metric_discrete}.

\subsection{Discrete Projection and Non-Potentiality}
\label{subsec:discrete_projection}

Let $W=\diag(w_e)$ be the diagonal matrix of edge weights and define $\|u\|_W^2:=u^\top W u$.

\begin{definition}[Discrete potential-flow projection]
\label{def:disc_proj}
Given an edge flow $\omega$, define
$
\phi^\star
\in
\arg\min_{\phi\in\R^N}\|\omega-B\phi\|_W^2.
$
The projected potential edge flow is
$
\omega_{\mathrm{pot}}:=B\phi^\star,
$
and the residual component is
$
\omega_{\mathrm{cyc}}:=\omega-\omega_{\mathrm{pot}}.
$
\end{definition}

The optimality condition yields the graph Poisson system
\begin{equation}
L\phi^\star=B^\top W\omega,
\qquad
L:=B^\top W B,
\label{eq:graph_poisson}
\end{equation}
with $\phi^\star$ unique up to an additive constant on each connected component.

\begin{lemma}[Orthogonality of the discrete residual]
\label{lem:disc_orth}
Let $\omega_{\mathrm{cyc}}:=\omega-B\phi^\star$. Then
$
(B\psi)^\top W\,\omega_{\mathrm{cyc}}=0,
\qquad
\forall \psi\in\R^N,
$
i.e., the residual is orthogonal to all discrete gradient flows.
\end{lemma}

\begin{corollary}[Discrete energy decomposition]
\label{cor:disc_pythag}
With $\omega_{\mathrm{pot}}=B\phi^\star$ and $\omega_{\mathrm{cyc}}=\omega-\omega_{\mathrm{pot}}$,
$
\|\omega\|_W^2
=
\|\omega_{\mathrm{pot}}\|_W^2
+
\|\omega_{\mathrm{cyc}}\|_W^2.
$
\end{corollary}

The graph interpretation of $\omega_{\mathrm{cyc}}$ as a cycle-space component is recorded in Remark~\ref{rem:graph_cycle_space}.

\subsection{Discrete Non-Potentiality}
\label{subsec:disc_nonpot}

\begin{definition}[Discrete non-potentiality proxy]
\label{def:disc_nonpot}
Define
$
\mathrm{NonPot}(\omega)
:=
\frac{\|\omega_{\mathrm{cyc}}\|_W^2}{\|\omega\|_W^2+\varepsilon},
\qquad
\varepsilon>0.
$
This quantity measures the fraction of edge-flow energy unexplained by any potential flow, and serves as a computable proxy for the continuous residual energy $\E_\mu[\|R(x)\|_M^2]$.
\end{definition}

\begin{lemma}[Basic properties of $\mathrm{NonPot}$]
\label{lem:nonpot_props}
For any $\omega$ and $\varepsilon>0$, one has
$
0\le \mathrm{NonPot}(\omega)\le 1.
$
Moreover, $\mathrm{NonPot}(\omega)\approx 0$ iff $\omega$ is approximately potential on the graph, and it increases with the energy of circulation components supported by cycles.
\end{lemma}

When samples cluster around recurring joint behaviors, $\mathrm{NonPot}(\omega)$ can be read as an empirical certificate of convention-level cyclic interaction dynamics; see Remark~\ref{rem:convention_graph}.
% When samples cluster around recurring joint behaviors, $\mathrm{NonPot}(\omega)$ can be read as an empirical certificate of convention cycling; see Remark~\ref{rem:convention_graph}.

\subsection{Node-Wise Direction Recovery}
\label{subsec:lifting}

The graph solve returns a node potential $\phi^\star$ or, equivalently, a potential edge flow $B\phi^\star$.
To update a particular iterate, we recover a node-wise direction by locally fitting potential differences and mapping the result back through $M^{-1}$.
The exact least-squares formulas are given in \eqref{eq:lifting_ls_h}--\eqref{eq:lifting_metric}, and the full graph-based procedure appears in Algorithm~\ref{alg:hodge-proj}.

\section{Amortized and CTDE Realizations of HPML}
\label{sec:algs}
This section records the amortized potential model and the CTDE parameter-space instantiation; full procedures appear in Algorithms~\ref{alg:hodge-proj}--\ref{alg:ctde-hpml}.

\subsection{Amortized Potential Projection}
\label{sec:neural-projection}
To avoid repeated large linear solves, we also use an amortized alternative that parameterizes the scalar potential and recovers the projected direction by automatic differentiation.

Let $\Phi_\theta:\cX\to\R$ be a neural potential model.
Its induced metric-gradient field is
$
g_\theta(x)
:=
\nabla_M \Phi_\theta(x)
=
M^{-1}\nabla_x \Phi_\theta(x).
$
Given samples $\{x_k\}_{k=1}^B$ and field estimates $\{\widehat F(x_k)\}$, we fit $\Phi_\theta$ by minimizing
\begin{equation}
\label{eq:nn-proj-loss}
\min_\theta\;
\widehat{\mathcal L}(\theta)
:=
\frac{1}{B}\sum_{k=1}^B
\bigl\|
\widehat F(x_k)-g_\theta(x_k)
\bigr\|_M^2
+
\lambda_{\mathrm g}
\Bigl(
\frac{1}{B}\sum_{k=1}^B \Phi_\theta(x_k)
\Bigr)^2
+
\lambda_{\mathrm{wd}}\|\theta\|_2^2.
\end{equation}
The gauge term removes the additive-constant ambiguity of the potential.
Once trained, the amortized projected update is
$
x_{t+1}
=
\Retr\bigl(x_t+\eta_t\,g_\theta(x_t)\bigr).
$
The full amortized procedure is given in Algorithm~\ref{alg:neural-hpml}; implementation schedules and metric variants are summarized in Remark~\ref{rem:neural_practical_app}.

\subsection{CTDE Parameter-Space Realization}
\label{sec:ctde-hpml}
Let
$
\theta := (\theta_1,\dots,\theta_n)\in\R^d
$
denote the concatenated actor parameters, and let
$
\widehat F(\theta)
:=
\bigl(\widehat g_1(\theta),\dots,\widehat g_n(\theta)\bigr)
$
be the stacked CTDE actor-update field.
A canonical example is a centralized-advantage estimator of the form
$
\widehat g_i(\theta)
:=
\E\Bigl[
\nabla_{\theta_i}\log\pi_{\theta_i}(a_i\mid o_i)\,
\widehat A(s,a)
\Bigr].
$
HPML projects the stacked field before the actor update:
$
g_t \approx \Pi_{\mathrm{grad}}^{M}\widehat F(\theta_t),
\qquad
\theta_{i,t+1}
=
\Retr\bigl(\theta_{i,t}+\eta_t[g_t]_i\bigr),
\quad i=1,\dots,n.
$
This realizes HPML as a plug-in projection layer on top of standard CTDE updates without changing decentralized execution.
The full training loop appears in Algorithm~\ref{alg:ctde-hpml}; the relation between policy-space theory and parameter-space implementation is discussed in Remark~\ref{rem:ctde_param_space_app}.
In implementation, HPML changes only the actor-update direction. Rollout collection, critic training, observation preprocessing, reward normalization, and decentralized execution remain the same as in the underlying MAPPO-style pipeline. The graph variant periodically builds a local sample graph from recent stacked update-field estimates, solves the graph Poisson system, and lifts node potentials to actor directions. The neural variant amortizes the same projection objective through a scalar potential network. Thus the additional cost comes from projection refresh and graph-lifting or neural inner optimization, while the critic, reward signal, policy architecture, and execution interface remain unchanged.

\subsection{CTDE Integrability Interpretation}
\label{subsec:bellman_bridge}

The decomposition
$
F
=
\Pi_{\mathrm{grad}}^{M}F + R
$
gives a direct CTDE interpretation.
The projected term is the component of the joint update field compatible with ascent of a single scalar potential, and the residual $R$ is the non-potential part outside that integrable structure.

In CTDE, let $z$ denote the chosen joint coordinates, either policy coordinates or actor-parameter coordinates, and let
$
F(z)=(g_1(z),\dots,g_n(z))
$
denote the stacked actor-update field induced by critic-based local improvement directions.
The next proposition states that the residual is exactly the obstruction to joint integrability. The key distinction below is between \emph{centralized evaluation} and \emph{exact centralized optimization}: a centralized critic or a single shared reward does not by itself eliminate the residual unless the realized joint actor-update field equals the exact gradient of one common scalar objective in the full joint parameter space.

\begin{proposition}[Residual as an obstruction to joint greedy integrability]
\label{prop:residual_bellman_obstruction}
Let $U\subset \mathrm{ri}(\cX)$ be a connected open set on which $F$ is continuously differentiable.

(i) If $R\equiv 0$ on $U$, then there exists a scalar potential $\Phi$ on $U$ such that
$
F=\Pi_{\mathrm{grad}}^{M}F=\nabla_M\Phi
\qquad\text{on }U.
$
Hence the joint actor-update field is integrable on $U$, and sufficiently small projected updates admit $\Phi$ as a Lyapunov function.
(ii) If there exists a closed piecewise-$C^1$ loop $C\subset U$ such that
$
\oint_C \langle M F(x),dx\rangle
=
\oint_C \langle M R(x),dx\rangle
\neq 0,
$
then no scalar potential $\Phi$ on $U$ can satisfy $F=\nabla_M\Phi$ on $U$.
Consequently, the simultaneous actor-update field has no exact representation as ascent of a single common scalar objective on $U$.
%Consequently, the simultaneous actor-update field cannot be interpreted as ascent of a single common scalar objective on $U$.
\end{proposition}

The local source of this circulation is asymmetry in the one-form $MF$. In the two-agent Euclidean case, it appears as a cross-response mismatch
$
\nabla_{x_2}g_1(x)\neq \nabla_{x_1}g_2(x)^\top,
$

\section{Projected Ascent and Equilibrium-Gap Bounds}
\label{sec:theory}

We analyze HPML in a clean Euclidean setting ($M=I$ with Euclidean projection) to isolate the main effect of the projection layer.
The theory has two steps: projected ascent along a potential flow yields Lyapunov improvement, and the residual non-potential component enters the original VI gap as an explicit additive term.
Technical extensions and auxiliary variants are deferred to Appendix~\ref{app:theory_details}; see Lemma~\ref{lem:proj_vi}, Theorem~\ref{thm:gap_L2}, Remark~\ref{rem:gap_interp_app}, and Remark~\ref{rem:rate_ext_app}.

\paragraph{Sign convention.}
Throughout this section, $F$ denotes the VI operator in Definition~\ref{def:gap}.
If $U(x)$ is an ascent direction, then $F(x)=-U(x)$; HPML constructs a potential $\Phi$ by projecting $U=-F$ so that the projected ascent direction is $\nabla\Phi$.
Equivalently, the VI operator decomposes as $F=-\nabla\Phi+R$, where $R$ is the residual obstruction to exact potential ascent.

\subsection{Assumptions}
\label{subsec:assumptions}

\begin{assumption}[Geometry and smoothness]
\label{ass:smooth}
$\cX\subset\R^d$ is nonempty, closed, convex, and compact with diameter
$
D:=\max_{x,y\in\cX}\|x-y\|.
$
The potential $\Phi:\cX\to\R$ is differentiable and $L$-smooth:
$
\|\nabla\Phi(x)-\nabla\Phi(y)\|\le L\|x-y\|,
\qquad \forall x,y\in\cX.
$
By compactness and continuity of $\nabla\Phi$, define
$
G:=\max_{x\in\cX}\|\nabla\Phi(x)\|<\infty,
\qquad
\Phi_{\max}:=\max_{x\in\cX}\Phi(x),
\qquad
\Phi_{\min}:=\min_{x\in\cX}\Phi(x).
$
\end{assumption}

\begin{assumption}[Residual bound]
\label{ass:residual}
The VI operator admits a decomposition
$
F(x)=-\nabla\Phi(x)+R(x),
$
where the residual satisfies
$
\|R(x)\|\le \varepsilon,
\qquad \forall x\in\cX.
$
An $L^2$ version is given in Theorem~\ref{thm:gap_L2}.
\end{assumption}

\begin{assumption}[Projection accuracy]
\label{ass:proj-acc}
At each step, HPML uses an ascent direction $g_t$ such that
$
\|g_t-\nabla\Phi(x_t)\|\le \delta_t,
$
where $\nabla\Phi$ is the exact projected ascent direction.
\end{assumption}

\subsection{Lyapunov Improvement Under Potential Flow}
\label{subsec:lyapunov}

We first isolate the ideal projected dynamics.

\begin{theorem}[Monotone Lyapunov improvement for projected gradient ascent]
\label{thm:lyapunov}
Under Assumption~\ref{ass:smooth}, consider the exact projected ascent step
$
x_{t+1}=\Proj_{\cX}\bigl(x_t+\eta\nabla\Phi(x_t)\bigr),
\qquad
\eta\le \frac{1}{L}.
$
Then
$
\Phi(x_{t+1})
\ge
\Phi(x_t)
+
\Bigl(\frac{1}{\eta}-\frac{L}{2}\Bigr)\|x_{t+1}-x_t\|^2
\ge
\Phi(x_t)+\frac{1}{2\eta}\|x_{t+1}-x_t\|^2.
$
\end{theorem}

Define the projected-step mapping
$
\cG_\eta(x_t):=\frac{1}{\eta}(x_{t+1}-x_t),
\qquad
x_{t+1}=\Proj_{\cX}(x_t+\eta\nabla\Phi(x_t)).
$
Then $\cG_\eta(x)=0$ iff $x$ is first-order stationary for the constrained maximization of $\Phi$.
Moreover, Theorem~\ref{thm:lyapunov} implies
$
\sum_{t=0}^{T-1}\|\cG_\eta(x_t)\|^2
\le
\frac{2}{\eta}(\Phi_{\max}-\Phi_{\min}),
$
so the exact projected dynamics cannot sustain a nontrivial cycle.
The projection inequality used in the proof is recorded separately in Lemma~\ref{lem:proj_vi}.

\subsection{Gap Bounds with a Residual Non-Potential Component}
\label{subsec:gap}

We now relate the projected dynamics to the VI gap of the original operator
$
F=-\nabla\Phi+R.
$

\begin{lemma}[Residual contribution to the VI gap]
\label{lem:gap_residual}
For any $x\in\cX$,
$
\Gap(x)
=
\max_{u\in\cX}\langle -\nabla\Phi(x)+R(x),x-u\rangle
\le
\Gap_\Phi(x)+D\|R(x)\|,
$
where
$
\Gap_\Phi(x)
:=
\max_{u\in\cX}\langle -\nabla\Phi(x),x-u\rangle
=
\max_{u\in\cX}\langle \nabla\Phi(x),u-x\rangle.
$
\end{lemma}

\begin{lemma}[Gap bound via projected-step mapping]
\label{lem:gap_mapping}
Under Assumption~\ref{ass:smooth}, let
$
x^+=\Proj_{\cX}(x+\eta\nabla\Phi(x)),
\qquad
\cG_\eta(x)=\frac{1}{\eta}(x^+-x).
$
Then
$
\Gap_\Phi(x)\le (D+\eta G)\,\|\cG_\eta(x)\|.
$
\end{lemma}

\begin{theorem}[Gap bound with additive non-potentiality]
\label{thm:gap}
Under Assumptions~\ref{ass:smooth} and~\ref{ass:residual}, run the exact projected ascent
$
x_{t+1}=\Proj_{\cX}(x_t+\eta\nabla\Phi(x_t)),
\qquad
\eta\le \frac{1}{L}.
$
Let
$
\hat t\in\arg\min_{0\le t\le T-1}\|\cG_\eta(x_t)\|,
\qquad
\hat x_T:=x_{\hat t}.
$
Then
$
\Gap(\hat x_T)
\le
(D+\eta G)\sqrt{\frac{2(\Phi_{\max}-\Phi_{\min})}{T\eta}}
+
D\varepsilon.
$
\end{theorem}

An $L^2$ output-law version of the same statement is given in Theorem~\ref{thm:gap_L2}, and the circulation interpretation of the additive residual term is summarized in Remark~\ref{rem:gap_interp_app}.

\paragraph{Reading the residual term.}
The bound certifies the original VI operator, not only the projected potential flow. The first term is the stationarity error of the projected dynamics, controlled by Lyapunov improvement, while the additive term $D\varepsilon$ is the contribution of the non-potential residual in the original field. Thus HPML chooses an integrable update direction, but any certificate for the original game must still account for residual non-potentiality. The $L^2$ version replaces the uniform residual bound by residual energy under the output-iterate law, matching the non-potentiality measure in Definition~\ref{def:nonpot} and the graph proxy in Definition~\ref{def:disc_nonpot}.

\paragraph{Inexact projection.}
Graph and neural projections are approximate in practice. Appendix~\ref{app:theory_details} extends Theorem~\ref{thm:gap} to directions $g_t$ satisfying $\|g_t-\nabla\Phi(x_t)\|\le\bar\delta$, adding projection-error terms while preserving the separation between structural non-potentiality and projection approximation error.

We finally quantify the effect of using an approximate projected direction $g_t$.
Let
$
x_{t+1}=\Proj_{\cX}(x_t+\eta g_t),
\widetilde{\cG}_\eta(x_t):=\frac{1}{\eta}(x_{t+1}-x_t).
$

\begin{theorem}[Effect of projection error]
\label{thm:inexact}
Under Assumptions~\ref{ass:smooth}, \ref{ass:residual}, and~\ref{ass:proj-acc},
run
$
x_{t+1}=\Proj_{\cX}(x_t+\eta g_t),
\qquad
\eta\le \frac{1}{L},
$
and suppose $\delta_t\le \bar\delta$ for all $t$.
Let
$
\hat t\in\arg\min_{0\le t\le T-1}\|\widetilde{\cG}_\eta(x_t)\|,
\hat x_T:=x_{\hat t}.
$
Then
$
\Gap(\hat x_T)
\le
\bigl(D+\eta(G+\bar\delta)\bigr)
\sqrt{\frac{4(\Phi_{\max}-\Phi_{\min})}{T\eta}+4\bar\delta^2}
+
D(\varepsilon+\bar\delta).
$
\end{theorem}

The auxiliary inexact gap-mapping estimate and a term-by-term interpretation of the bound are recorded in Lemma~\ref{lem:gap_mapping_inexact} and Remark~\ref{rem:inexact_interp_app}.

\section{Experiments}
\label{sec:exp}

\begin{figure}[t]
  \centering
  \begin{subfigure}[t]{0.24\textwidth}
    \centering
    \includegraphics[width=\linewidth]{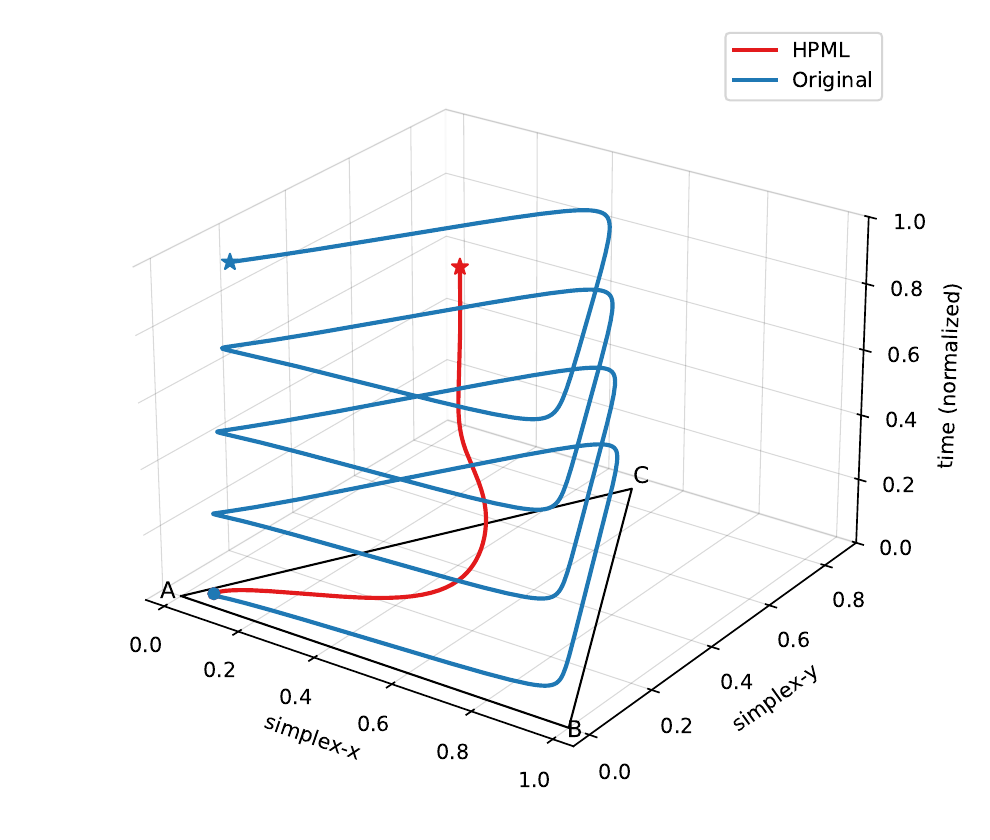}
    \caption{RPS simplex}
    \label{fig:interp_simplex_p}
  \end{subfigure}\hfill
  \begin{subfigure}[t]{0.24\textwidth}
    \centering
    \includegraphics[width=\linewidth]{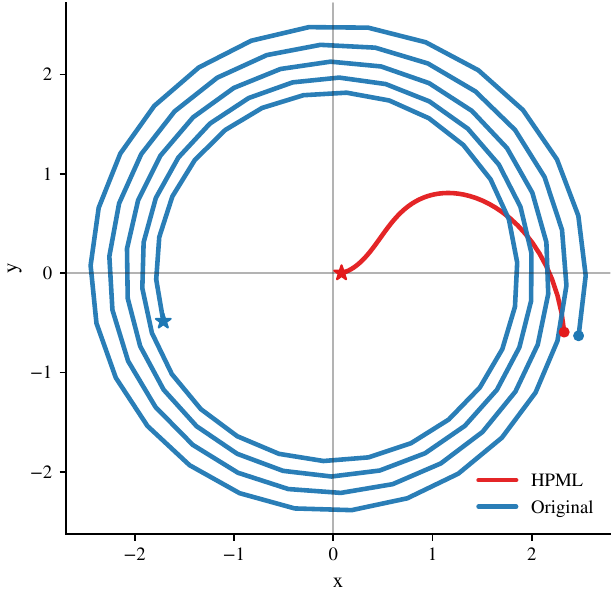}
    \caption{2D field}
    \label{fig:interp_bilinear_phase}
  \end{subfigure}\hfill
  \begin{subfigure}[t]{0.24\textwidth}
    \centering
    \includegraphics[width=\linewidth]{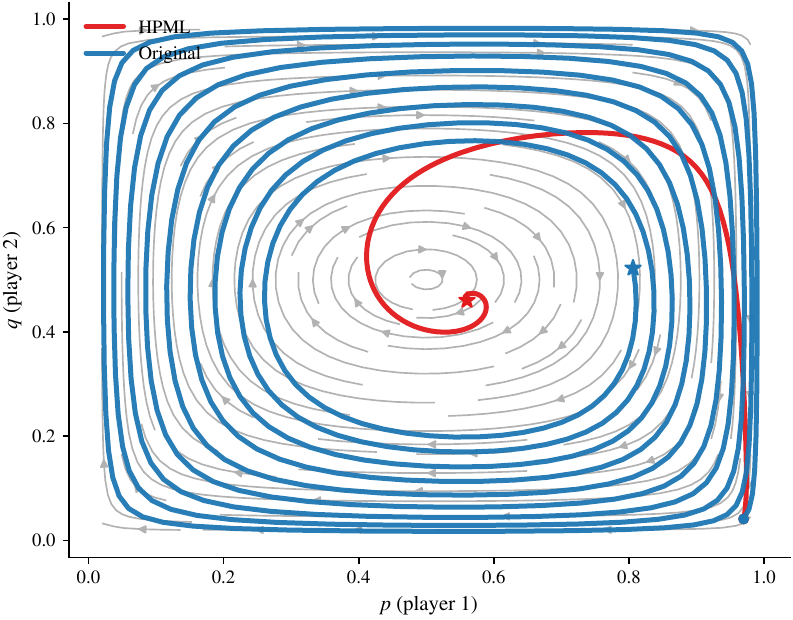}
    \caption{Logistic game}
    \label{fig:interp_logistic2x2}
  \end{subfigure}\hfill
  \begin{subfigure}[t]{0.24\textwidth}
    \centering
    \includegraphics[width=\linewidth]{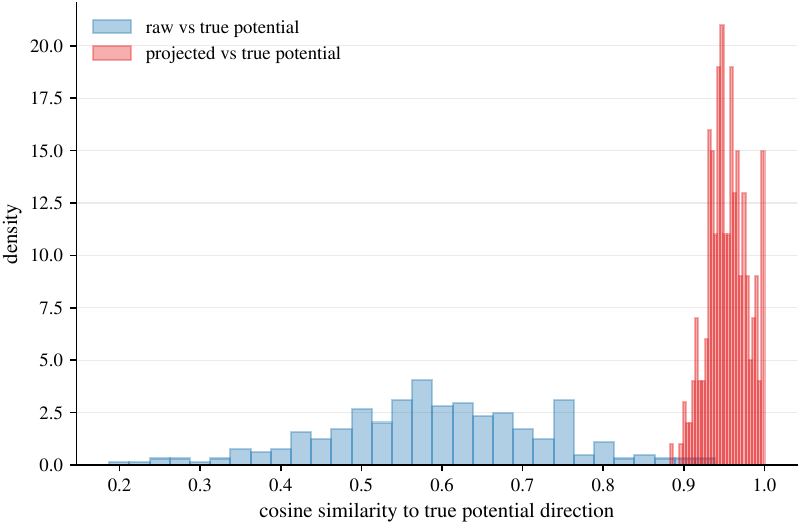}
    \caption{3D projection}
    \label{fig:interp_linear3d_correct}
  \end{subfigure}
  \caption{\textbf{Mechanism tests.} HPML suppresses circulation in controlled games and recovers the known potential direction in the linear 3D test. The first three panels compare raw and projected trajectories or fields; the fourth reports cosine similarity to the exact potential component $g_{\mathrm{pot}}(z)=-z$.}
  \label{fig:part1}
\end{figure}

\providecommand{\curveincludeorplaceholder}[1]{%
  \IfFileExists{#1}{\includegraphics[width=\linewidth]{#1}}{\fbox{\rule{0pt}{1.15in}\rule{0.95\linewidth}{0pt}}}%
}

\begin{figure}[h]
  \centering
  \begin{subfigure}[t]{0.32\textwidth}
    \centering
    \curveincludeorplaceholder{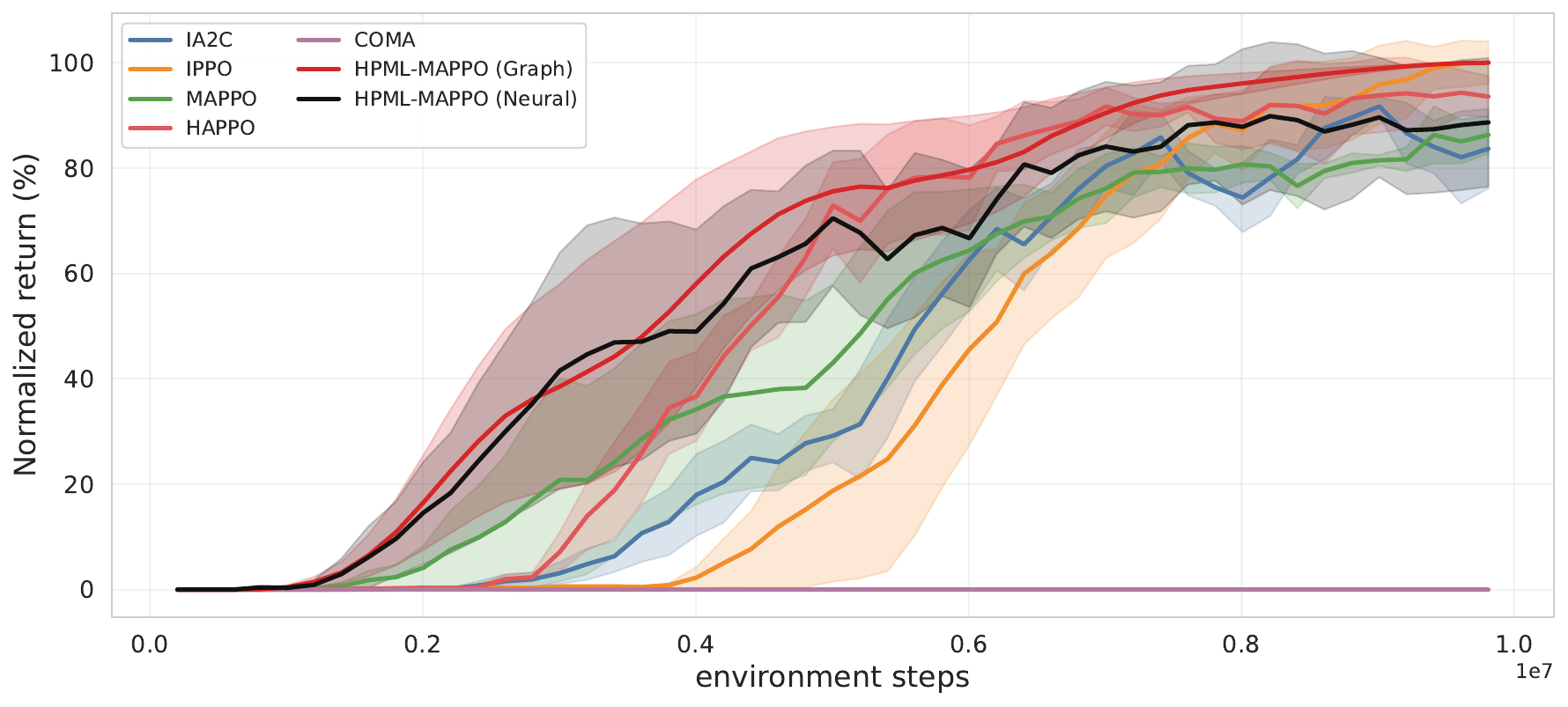}
    \caption{\texttt{bos\_repeated}}
    \label{fig:mp_bos_curve}
  \end{subfigure}\hfill
  \begin{subfigure}[t]{0.32\textwidth}
    \centering
    \curveincludeorplaceholder{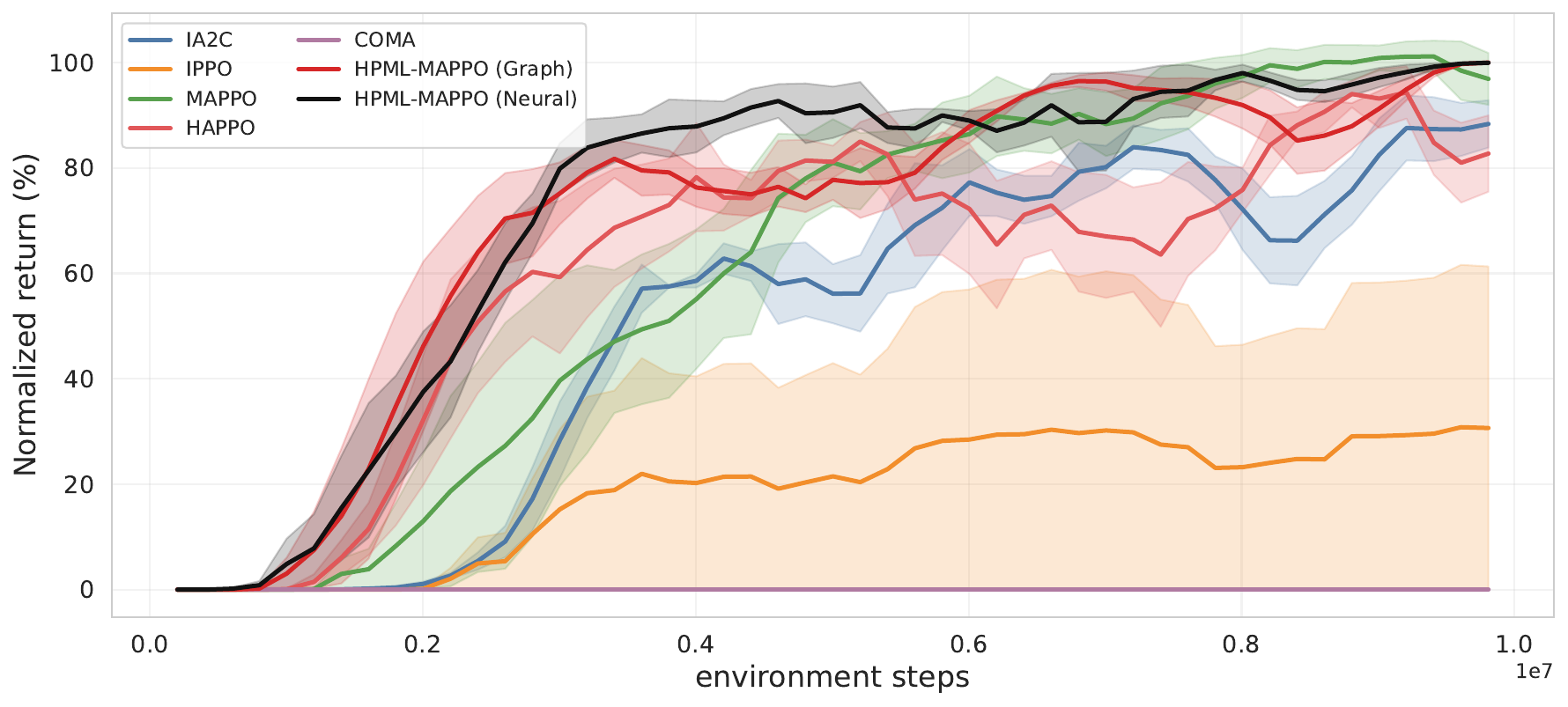}
    \caption{\texttt{collaborative\_cooking\_circuit}}
    \label{fig:mp_cc_circuit_curve}
  \end{subfigure}\hfill
  \begin{subfigure}[t]{0.32\textwidth}
    \centering
    \curveincludeorplaceholder{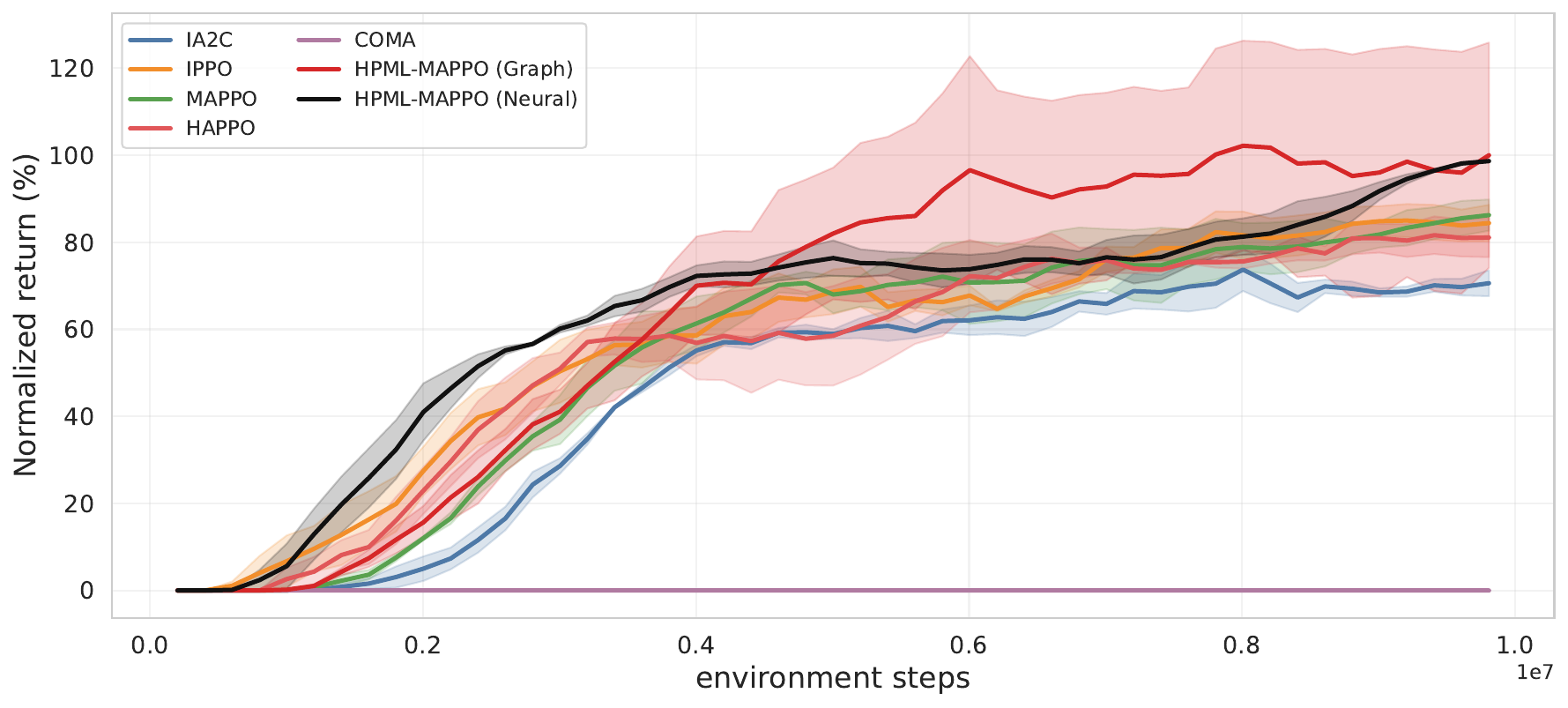}
    \caption{\texttt{clean\_up}}
    \label{fig:mp_cleanup_curve}
  \end{subfigure}
  \caption{\textbf{Representative Melting Pot learning curves.} Normalized return versus environment steps on three scenarios covering convention selection, embodied coordination, and sequential social dilemmas. Shaded bands denote standard error over $S=5$ seeds; remaining scenarios are reported in Appendix~\ref{app:exp_curves}.}
  \label{fig:meltingpot_main_curves}
\end{figure}

\begin{table}[h]
\centering
\small
\caption{\textbf{Melting Pot summary.} Final normalized return (\%) on three representative scenarios, macro average over the listed scenarios, and training-average raw cyclic residual energy (lower is better). Full per-scenario results are given in Appendix~\ref{app:exp_table_full}.}
\label{tab:meltingpot_summary}
\begin{tabular}{lccccc}
\toprule
Method & BoS-Rep & CC-Circuit & Clean Up & Shown Avg & 
Raw residual energy$\downarrow$ \\
% $\mathrm{NonPot}(\omega)\downarrow$ \\
\midrule
IA2C  & 83.7 & 88.3 & 70.6 & 80.9 & -- \\
IPPO  & 100.0 & 30.7 & 84.4 & 71.7 & -- \\
MAPPO  & 86.3 & 96.9 & 86.2 & 89.8 & -- \\
HAPPO  & 93.6 & 82.7 & 81.1 & 85.8 & -- \\
COMA  & 0.0 & 0.0 & 0.0 & 0.0 & -- \\
\midrule
HPML-MAPPO (Graph)  & 100.0 & 100.0 & 100.0 & 100.0 & 0.074 \\
HPML-MAPPO (Neural)  & 88.7 & 99.9 & 98.7 & 95.8 & 411.326 \\
% IA2C  & -- & -- & -- & -- & -- \\
% IPPO  & -- & -- & -- & -- & -- \\
% MAPPO & -- & -- & -- & -- & -- \\
% HAPPO & -- & -- & -- & -- & -- \\
% COMA  & -- & -- & -- & -- & -- \\
% \midrule
% HPML-MAPPO (Graph)  & -- & -- & -- & -- & -- \\
% HPML-MAPPO (Neural) & -- & -- & -- & -- & -- \\
\bottomrule
\end{tabular}
\end{table}

We evaluate HPML in two regimes: controlled games where the potential and cyclic components are interpretable, and high-dimensional CTDE benchmarks where the projection is used as a plug-in actor-update layer.

\subsection{Mechanism Tests in Interpretable Games}
\label{subsec:exp_interpretable}

We first isolate the geometric mechanism in four controlled settings: cyclic matrix games, a two-dimensional potential-plus-skew field, a logistic $2\times2$ mixed-strategy game, and a linear 3D field with a known potential component. Details of the testbeds and metrics are given in Appendix~\ref{app:exp_interpretable_details}. Across these settings, we compare the raw update field with the HPML-projected field and track circulation, path length, graph non-potentiality $\mathrm{NonPot}(\omega)$, and a gap/exploitability proxy when available.

These tests verify the intended mechanism: HPML preserves the potential-flow component while reducing sample-graph circulation. The 3D test checks projection accuracy against a known potential component in addition to trajectory visualizations.

\subsection{Melting Pot CTDE Benchmarks}
\label{subsec:exp_meltingpot}

We next evaluate HPML in high-dimensional, partially observed MARL using a 10-scenario Melting Pot suite~\citep{leibo2021scalable,agapiou2022melting}. The suite covers convention selection, repeated coordination, social dilemmas, and embodied team coordination. 
The main text reports three representative scenarios---\texttt{bos\_repeated}, \texttt{collaborative\_cooking\_circuit}, and \texttt{clean\_up}---while the remaining scenarios, hyperparameters, and full tables are deferred to Appendix~\ref{app:exp_impl}--\ref{app:exp_diag}.

We compare IA2C, IPPO, MAPPO, and HAPPO against two HPML variants built on MAPPO: \textbf{HPML-MAPPO (Graph)} and \textbf{HPML-MAPPO (Neural)}. COMA is reported in the full appendix table for completeness. All methods use the same actor/critic backbone, rollout budget, and observation interface. Curves show mean return with standard-error bands over $S=5$ seeds; Table~\ref{tab:meltingpot_summary} reports final means.

The controlled tasks test the geometric mechanism directly, while the Melting Pot suite tests the same projection inside a high-dimensional CTDE pipeline. Because HPML-MAPPO uses the same backbone, rollout budget, and observation interface as MAPPO, the comparison isolates the effect of projecting the stacked actor-update field.

\section{Conclusion}

We introduced Hodge-Projected Multi-Agent Learning (HPML), a geometric projection method that separates a joint multi-agent update field into a metric-gradient potential component and a non-potential residual. The projected component admits Lyapunov-style improvement, and the residual captures cyclic structure that appears explicitly in the VI-gap bound. Graph-based and amortized implementations make HPML usable as a CTDE plug-in, and the experiments show reduced circulation in controlled games and improved stability and normalized return on the evaluated Melting Pot suite. Key limitations are the dependence on the chosen metric, sampling distribution, field-estimation quality, and projection approximation error.

{
\small
\nocite{*}
\bibliographystyle{unsrtnat}
\bibliography{ref}

@book{puterman2014markov,
  title={Markov decision processes: discrete stochastic dynamic programming},
  author={Puterman, Martin L},
  year={2014},
  publisher={John Wiley \& Sons}
}

@inproceedings{DBLP:conf/nips/ZinkevichGL05,
  author       = {Martin Zinkevich and
                  Amy Greenwald and
                  Michael L. Littman},
  title        = {Cyclic Equilibria in Markov Games},
  booktitle    = {Advances in Neural Information Processing Systems 18 [Neural Information
                  Processing Systems, {NIPS} 2005, December 5-8, 2005, Vancouver, British
                  Columbia, Canada]},
  pages        = {1641--1648},
  year         = {2005},
  url          = {https://proceedings.neurips.cc/paper/2005/hash/9752d873fa71c19dc602bf2a0696f9b5-Abstract.html},
  bibsource    = {dblp computer science bibliography, https://dblp.org}
}

@article{tang2026agent,
  title={Agent Alpha: Tree Search Unifying Generation, Exploration and Evaluation for Computer-Use Agents},
  author={Tang, Sizhe and Chen, Rongqian and Lan, Tian},
  journal={arXiv preprint arXiv:2602.02995},
  year={2026}
}

@article{li2026acdzero,
  title={Acdzero: Graph-embedding-based tree search for mastering automated cyber defense},
  author={Li, Yu and Tang, Sizhe and Chen, Rongqian and Yu, Fei Xu and Jiang, Guangyu and Imani, Mahdi and Bastian, Nathaniel D and Lan, Tian},
  journal={arXiv preprint arXiv:2601.02196},
  year={2026}
}

@article{zhang2026operator,
  title={Operator-Guided Invariance Learning for Continuous Reinforcement Learning},
  author={Zhang, Zuyuan and Yu, Fei Xu and Lan, Tian},
  journal={arXiv preprint arXiv:2605.06500},
  year={2026}
}

@article{tang2026nonzero,
  title={NonZero: Interaction-Guided Exploration for Multi-Agent Monte Carlo Tree Search},
  author={Tang, Sizhe and Zhang, Zuyuan and Imani, Mahdi and Lan, Tian},
  journal={arXiv preprint arXiv:2605.00751},
  year={2026}
}

@book{sutton1998reinforcement,
  title={Reinforcement learning: An introduction},
  author={Sutton, Richard S and Barto, Andrew G and others},
  volume={1},
  number={1},
  year={1998},
  publisher={MIT press Cambridge}
}

@article{shapley1953stochastic,
  title={Stochastic games},
  author={Shapley, Lloyd S},
  journal={Proceedings of the national academy of sciences},
  volume={39},
  number={10},
  pages={1095--1100},
  year={1953},
  publisher={National Academy of Sciences}
}

@incollection{littman1994markov,
  title={Markov games as a framework for multi-agent reinforcement learning},
  author={Littman, Michael L},
  booktitle={Machine learning proceedings 1994},
  pages={157--163},
  year={1994},
  publisher={Elsevier}
}

@article{bucsoniu2010multi,
  title={Multi-agent reinforcement learning: An overview},
  author={Bu{\c{s}}oniu, Lucian and Babu{\v{s}}ka, Robert and De Schutter, Bart},
  journal={Innovations in multi-agent systems and applications-1},
  pages={183--221},
  year={2010},
  publisher={Springer}
}

@article{bertsekas1997nonlinear,
  title={Nonlinear programming},
  author={Bertsekas, Dimitri P},
  journal={Journal of the Operational Research Society},
  volume={48},
  number={3},
  pages={334--334},
  year={1997},
  publisher={Taylor \& Francis}
}

@inproceedings{balduzzi2018mechanics,
  title={The mechanics of n-player differentiable games},
  author={Balduzzi, David and Racaniere, Sebastien and Martens, James and Foerster, Jakob and Tuyls, Karl and Graepel, Thore},
  booktitle={International Conference on Machine Learning},
  pages={354--363},
  year={2018},
  organization={PMLR}
}

@article{monderer1996potential,
  title={Potential games},
  author={Monderer, Dov and Shapley, Lloyd S},
  journal={Games and economic behavior},
  volume={14},
  number={1},
  pages={124--143},
  year={1996},
  publisher={Elsevier}
}

@inproceedings{schulman2015trust,
  title={Trust region policy optimization},
  author={Schulman, John and Levine, Sergey and Abbeel, Pieter and Jordan, Michael and Moritz, Philipp},
  booktitle={International conference on machine learning},
  pages={1889--1897},
  year={2015},
  organization={PMLR}
}

@article{schulman2015high,
  title={High-dimensional continuous control using generalized advantage estimation},
  author={Schulman, John and Moritz, Philipp and Levine, Sergey and Jordan, Michael and Abbeel, Pieter},
  journal={arXiv preprint arXiv:1506.02438},
  year={2015}
}

@article{schulman2017proximal,
  title={Proximal policy optimization algorithms},
  author={Schulman, John and Wolski, Filip and Dhariwal, Prafulla and Radford, Alec and Klimov, Oleg},
  journal={arXiv preprint arXiv:1707.06347},
  year={2017}
}

@article{yu2022surprising,
  title={The surprising effectiveness of ppo in cooperative multi-agent games},
  author={Yu, Chao and Velu, Akash and Vinitsky, Eugene and Gao, Jiaxuan and Wang, Yu and Bayen, Alexandre and Wu, Yi},
  journal={Advances in neural information processing systems},
  volume={35},
  pages={24611--24624},
  year={2022}
}

@article{lowe2017multi,
  title={Multi-agent actor-critic for mixed cooperative-competitive environments},
  author={Lowe, Ryan and Wu, Yi I and Tamar, Aviv and Harb, Jean and Pieter Abbeel, OpenAI and Mordatch, Igor},
  journal={Advances in neural information processing systems},
  volume={30},
  year={2017}
}

@inproceedings{foerster2018counterfactual,
  title={Counterfactual multi-agent policy gradients},
  author={Foerster, Jakob and Farquhar, Gregory and Afouras, Triantafyllos and Nardelli, Nantas and Whiteson, Shimon},
  booktitle={Proceedings of the AAAI conference on artificial intelligence},
  volume={32},
  number={1},
  year={2018}
}

@article{rashid2020monotonic,
  title={Monotonic value function factorisation for deep multi-agent reinforcement learning},
  author={Rashid, Tabish and Samvelyan, Mikayel and De Witt, Christian Schroeder and Farquhar, Gregory and Foerster, Jakob and Whiteson, Shimon},
  journal={Journal of Machine Learning Research},
  volume={21},
  number={178},
  pages={1--51},
  year={2020}
}

@article{sunehag2017value,
  title={Value-decomposition networks for cooperative multi-agent learning},
  author={Sunehag, Peter and Lever, Guy and Gruslys, Audrunas and Czarnecki, Wojciech Marian and Zambaldi, Vinicius and Jaderberg, Max and Lanctot, Marc and Sonnerat, Nicolas and Leibo, Joel Z and Tuyls, Karl and others},
  journal={arXiv preprint arXiv:1706.05296},
  year={2017}
}

@article{mescheder2017numerics,
  title={The numerics of gans},
  author={Mescheder, Lars and Nowozin, Sebastian and Geiger, Andreas},
  journal={Advances in neural information processing systems},
  volume={30},
  year={2017}
}

@article{gidel2018variational,
  title={A variational inequality perspective on generative adversarial networks},
  author={Gidel, Gauthier and Berard, Hugo and Vignoud, Ga{\"e}tan and Vincent, Pascal and Lacoste-Julien, Simon},
  journal={arXiv preprint arXiv:1802.10551},
  year={2018}
}

@article{daskalakis2017training,
  title={Training gans with optimism},
  author={Daskalakis, Constantinos and Ilyas, Andrew and Syrgkanis, Vasilis and Zeng, Haoyang},
  journal={arXiv preprint arXiv:1711.00141},
  year={2017}
}

@article{korpelevich1977extragradient,
  title={Extragradient method for finding saddle points and other problems},
  author={Korpelevich, Galina M},
  journal={Matekon},
  volume={13},
  number={4},
  pages={35--49},
  year={1977},
  publisher={ME SHARPE INC 80 BUSINESS PARK DR, ARMONK, NY 10504}
}

@article{nemirovski2004prox,
  title={Prox-method with rate of convergence O (1/t) for variational inequalities with Lipschitz continuous monotone operators and smooth convex-concave saddle point problems},
  author={Nemirovski, Arkadi},
  journal={SIAM Journal on Optimization},
  volume={15},
  number={1},
  pages={229--251},
  year={2004},
  publisher={SIAM}
}

@article{popov1980modification,
  title={A modification of the Arrow-Hurwicz method for search of saddle points},
  author={Popov, Leonid Denisovich},
  journal={Mathematical notes of the Academy of Sciences of the USSR},
  volume={28},
  number={5},
  pages={845--848},
  year={1980},
  publisher={Springer}
}

@article{mertikopoulos2018optimistic,
  title={Optimistic mirror descent in saddle-point problems: Going the extra (gradient) mile},
  author={Mertikopoulos, Panayotis and Lecouat, Bruno and Zenati, Houssam and Foo, Chuan-Sheng and Chandrasekhar, Vijay and Piliouras, Georgios},
  journal={arXiv preprint arXiv:1807.02629},
  year={2018}
}

@article{zhang2025lipschitz,
  title={Lipschitz lifelong monte carlo tree search for mastering non-stationary tasks},
  author={Zhang, Zuyuan and Lan, Tian},
  journal={arXiv preprint arXiv:2502.00633},
  year={2025}
}

@article{zhang2025tail,
  title={Tail-risk-safe monte carlo tree search under pac-level guarantees},
  author={Zhang, Zuyuan and Ghosh, Arnob and Lan, Tian},
  journal={arXiv preprint arXiv:2508.05441},
  year={2025}
}

@article{fang2026manifold,
  title={Manifold-Constrained Energy-Based Transition Models for Offline Reinforcement Learning},
  author={Fang, Zeyu and Zhang, Zuyuan and Imani, Mahdi and Lan, Tian},
  journal={arXiv preprint arXiv:2602.02900},
  year={2026}
}

@article{zhang2026structuring,
  title={Structuring Value Representations via Geometric Coherence in Markov Decision Processes},
  author={Zhang, Zuyuan and Fang, Zeyu and Lan, Tian},
  journal={arXiv preprint arXiv:2602.02978},
  year={2026}
}

@article{zhang2026geometry,
  title={Geometry of drifting mdps with path-integral stability certificates},
  author={Zhang, Zuyuan and Imani, Mahdi and Lan, Tian},
  journal={arXiv preprint arXiv:2601.21991},
  year={2026}
}

@article{zhang2026cochain,
  title={Cochain Perspectives on Temporal-Difference Signals for Learning Beyond Markov Dynamics},
  author={Zhang, Zuyuan and Tang, Sizhe and Lan, Tian},
  journal={arXiv preprint arXiv:2602.06939},
  year={2026}
}

@book{facchinei2003finite,
  title={Finite-dimensional variational inequalities and complementarity problems},
  author={Facchinei, Francisco and Pang, Jong-Shi},
  year={2003},
  publisher={Springer}
}

@book{hodge1989theory,
  title={The theory and applications of harmonic integrals},
  author={Hodge, William Vallance Douglas},
  year={1989},
  publisher={CUP Archive}
}

@article{eckmann1944harmonische,
  title={Harmonische funktionen und randwertaufgaben in einem komplex},
  author={Eckmann, Beno},
  journal={Commentarii Mathematici Helvetici},
  volume={17},
  number={1},
  pages={240--255},
  year={1944},
  publisher={Springer}
}

@book{chung1997spectral,
  title={Spectral graph theory},
  author={Chung, Fan RK},
  volume={92},
  year={1997},
  publisher={American Mathematical Soc.}
}

@article{von2007tutorial,
  title={A tutorial on spectral clustering},
  author={Von Luxburg, Ulrike},
  journal={Statistics and computing},
  volume={17},
  number={4},
  pages={395--416},
  year={2007},
  publisher={Springer}
}

@inproceedings{spielman2010algorithms,
  title={Algorithms, graph theory, and linear equations in Laplacian matrices},
  author={Spielman, Daniel A},
  booktitle={Proceedings of the International Congress of Mathematicians 2010 (ICM 2010) (In 4 Volumes) Vol. I: Plenary Lectures and Ceremonies Vols. II--IV: Invited Lectures},
  pages={2698--2722},
  year={2010},
  organization={World Scientific}
}

@misc{hawrylycz2012discrete,
  title={Discrete Calculus: Applied Analysis on Graphs for Computational Science, Leo Grady, Jonathan Polimeni, Springer (2010), \$129.00, ISBN: 978-1-84996-289-6},
  author={Hawrylycz, Mike},
  year={2012},
  publisher={Pergamon}
}

@article{lim2020hodge,
  title={Hodge Laplacians on graphs},
  author={Lim, Lek-Heng},
  journal={Siam Review},
  volume={62},
  number={3},
  pages={685--715},
  year={2020},
  publisher={SIAM}
}

@inproceedings{leibo2021scalable,
  title={Scalable evaluation of multi-agent reinforcement learning with melting pot},
  author={Leibo, Joel Z and Due{\~n}ez-Guzman, Edgar A and Vezhnevets, Alexander and Agapiou, John P and Sunehag, Peter and Koster, Raphael and Matyas, Jayd and Beattie, Charlie and Mordatch, Igor and Graepel, Thore},
  booktitle={International conference on machine learning},
  pages={6187--6199},
  year={2021},
  organization={PMLR}
}

@article{agapiou2022melting,
  title={Melting Pot 2.0},
  author={Agapiou, John P and Vezhnevets, Alexander Sasha and Du{\'e}{\~n}ez-Guzm{\'a}n, Edgar A and Matyas, Jayd and Mao, Yiran and Sunehag, Peter and K{\"o}ster, Raphael and Madhushani, Udari and Kopparapu, Kavya and Comanescu, Ramona and others},
  journal={arXiv preprint arXiv:2211.13746},
  year={2022}
}

@inproceedings{zhang2025learning,
  title={Learning to collaborate with unknown agents in the absence of reward},
  author={Zhang, Zuyuan and Zhou, Hanhan and Imani, Mahdi and Lee, Taeyoung and Lan, Tian},
  booktitle={Proceedings of the AAAI Conference on Artificial Intelligence},
  volume={39},
  number={13},
  pages={14502--14511},
  year={2025}
}

@article{zhang2024modeling,
  title={Modeling other players with bayesian beliefs for games with incomplete information},
  author={Zhang, Zuyuan and Imani, Mahdi and Lan, Tian},
  journal={arXiv preprint arXiv:2405.14122},
  year={2024}
}

@inproceedings{zhang2025network,
  title={Network diffuser for placing-scheduling service function chains with inverse demonstration},
  author={Zhang, Zuyuan and Aggarwal, Vaneet and Lan, Tian},
  booktitle={IEEE INFOCOM 2025-IEEE Conference on Computer Communications},
  pages={1--10},
  year={2025},
  organization={IEEE}
}

@inproceedings{qiao2024br,
  title={Br-defedrl: Byzantine-robust decentralized federated reinforcement learning with fast convergence and communication efficiency},
  author={Qiao, Jing and Zhang, Zuyuan and Yue, Sheng and Yuan, Yuan and Cai, Zhipeng and Zhang, Xiao and Ren, Ju and Yu, Dongxiao},
  booktitle={IEEE infocom 2024-IEEE conference on computer communications},
  pages={141--150},
  year={2024},
  organization={IEEE}
}

@inproceedings{ravari2024adversarial,
  title={Adversarial inverse learning of defense policies conditioned on human factor models},
  author={Ravari, Amirhossein and Jiang, Guangyu and Zhang, Zuyuan and Imani, Mahdi and Thomson, Robert H and Pyke, Aryn A and Bastian, Nathaniel D and Lan, Tian},
  booktitle={2024 58th Asilomar Conference on Signals, Systems, and Computers},
  pages={188--195},
  year={2024},
  organization={IEEE}
}

@article{zhang2026lisfc,
  title={Lisfc-search: Lifelong search for network sfc optimization under non-stationary drifts},
  author={Zhang, Zuyuan and Aggarwal, Vaneet and Lan, Tian},
  journal={arXiv preprint arXiv:2602.14360},
  year={2026}
}

@article{tang2025malinzero,
  title={Malinzero: Efficient low-dimensional search for mastering complex multi-agent planning},
  author={Tang, Sizhe and Chen, Jiayu and Lan, Tian},
  journal={arXiv preprint arXiv:2511.06142},
  year={2025}
}
}

%%%%%%%%%%%%%%%%%%%%%%%%%%%%%%%%%%%%%%%%%%%%%%%%%%%%%%%%%%%%

\appendix
\newpage

\section{Additional Experimental Details and Full Results}
\label{app:exp_full}

\subsection{Mechanism-test details}
\label{app:exp_interpretable_details}

The controlled experiments in Section~\ref{subsec:exp_interpretable} use four diagnostic settings.
First, cyclic matrix games, including Rock--Paper--Scissors and generalized cyclic games on $\Delta_K\times\Delta_K$, expose cyclic interaction dynamics in a familiar simplex geometry.
Second, the continuous field $g(z)=-z+\rho Jz$ separates a contractive potential component from a tunable skew component; increasing $\rho$ increases transient swirl while preserving the radial potential part.
Third, the logistic $2\times 2$ mixed-strategy game parameterizes strategies by logits $(a,b)$ and visualizes the induced field on $(p,q)=(\sigma(a),\sigma(b))$.
Fourth, the linear 3D field $g(z)=-z+\rho Sz$ uses a known skew-symmetric matrix $S$, so the exact potential component is $-z$ and projection accuracy can be measured directly.

For each setting, we compare the raw update field with the HPML-projected field computed on a sample graph.
The main diagnostics are windowed orbit diameter, total path length, the graph residual ratio $\mathrm{NonPot}(\omega)$, and a dual-gap or exploitability proxy when the game admits one.
These experiments are intended to verify the mechanism rather than to serve as a large-scale benchmark: increasing $\rho$ should increase the measured non-potential component and the amount of cycling in the raw dynamics, while HPML should reduce both by retaining only the potential-flow component.

\subsection{Implementation details}
\label{app:exp_impl}

This appendix provides the implementation details omitted from the main text,
including Melting Pot scenario handling, observation preprocessing,
CTDE training configuration, baseline-specific settings, HPML graph construction,
neural-projection optimization, and evaluation protocol.

\paragraph{Melting Pot specifics.}
For Melting Pot, we follow a unified preprocessing and rollout interface across all scenarios.
We keep the actor/critic backbone fixed across scenarios unless otherwise stated,
and only adjust environment-dependent quantities such as action masking,
episode horizon, and reward normalization.

\paragraph{Architectures.}
All methods share the same actor/critic backbone unless otherwise stated.
For CTDE methods, we use decentralized actors with a centralized critic;
parameter counts are matched as closely as possible across baselines.

\paragraph{Training protocol.}
We train each method for a fixed environment-step budget on every Melting Pot scenario,
using $S$ seeds (default $S{=}5$).
We report mean and standard error across seeds.
Learning curves are smoothed only for visualization.

\paragraph{HPML specifics.}
For HPML-Graph, we build a $k$-NN graph over a buffer of recent policy/parameter points $\{x^i\}$ and estimate edge flows
$\omega_{ij}=\langle M\frac{F_i+F_j}{2}, x^j-x^i\rangle$ (Section~\ref{sec:discrete}).
We solve the graph Poisson system with a fixed gauge and lift the potential to node directions via local least squares.
For HPML-Neural, we fit a scalar potential network $\Phi_\theta$ by minimizing the empirical projection loss
(Section~\ref{sec:neural-projection}) and compute projected directions by autograd.

\paragraph{Hyperparameters.}
We sweep a small grid for each method and report the best validation configuration.
Table~\ref{tab:hparams} lists the final hyperparameters used in the main experiments.

\begin{table}[t]
\centering
\small
\caption{Hyperparameters used in Melting Pot experiments.}
\label{tab:hparams}
\begin{tabular}{lcc}
\toprule\toprule
Hyperparameter & Value & Notes \\
\midrule\midrule
Total env steps & $10{,}000{,}000$ & per scenario \\
Rollout length / batch size & $128$ / $256 \times n_{\mathrm{agents}}$ & 2 vectorized envs; minibatch size $4096$ \\
Actor LR / critic LR & $3\times10^{-4}$ / $3\times10^{-4}$ & Adam, $\epsilon=10^{-5}$ \\
PPO clip $\epsilon$ & $0.2$ & for PPO-style methods \\
Entropy coef & $0.01$ & \\
HPML graph $k$ & $4$ & kNN degree \\
HPML refresh rate & $8$ & actor update/minibatch steps between projections \\
HPML ridge $\lambda$ (lifting) & $10^{-4}$ & stability \\
Neural potential LR / inner steps & $10^{-3}$ / $2$ & if HPML-Neural \\
% Total env steps & -- & per scenario \\
% Rollout length / batch size & -- & \\
% Actor LR / critic LR & -- & \\
% PPO clip $\epsilon$ & -- & for PPO-style methods \\
% Entropy coef & -- & \\
% HPML graph $k$ & -- & kNN degree \\
% HPML refresh rate & -- & steps between projections \\
% HPML ridge $\lambda$ (lifting) & -- & stability \\
% Neural potential LR / inner steps & -- & if HPML-Neural \\
\bottomrule\bottomrule
\end{tabular}
\end{table}

\paragraph{Compute resources and software.}
All Melting Pot experiments were run on a single machine equipped with four NVIDIA RTX 6000 Ada Generation GPUs, each with 49,140 MiB of memory, an AMD EPYC 7713 64-Core Processor CPU, and 1.0 TiB of RAM. Software versions were Python 3.13.9, NumPy 2.4.2, and CUDA 13.0 as reported by NVIDIA-SMI.

\paragraph{Existing assets and licenses.}
The experiments use the Melting Pot benchmark \citep{leibo2021scalable,agapiou2022melting}, which is distributed under the Apache-2.0 license. We use the benchmark only for evaluation and do not redistribute modified benchmark assets. The anonymized supplementary code package includes its own license and documentation.

\subsection{Additional Melting Pot learning curves (remaining seven scenarios)}
\label{app:exp_curves}

\providecommand{\curveincludeorplaceholder}[1]{%
  \IfFileExists{#1}{\includegraphics[width=\linewidth]{#1}}{\fbox{\rule{0pt}{1.15in}\rule{0.95\linewidth}{0pt}}}%
}
\newcommand{\curveplaceholder}[3]{%
\begin{figure}[t]
  \centering
  \curveincludeorplaceholder{#3}
  \caption{Normalized return versus environment steps on \textbf{#1}.}
  \label{fig:curve:#2}
\end{figure}
}

Figure~\ref{fig:meltingpot_main_curves} in the main text shows three representative scenarios (\texttt{bos\_repeated}, \texttt{collaborative\_cooking\_circuit}, and \texttt{clean\_up}). Here we report the remaining seven scenarios.

\curveplaceholder{Pure Coordination Repeated}{pc_rep}{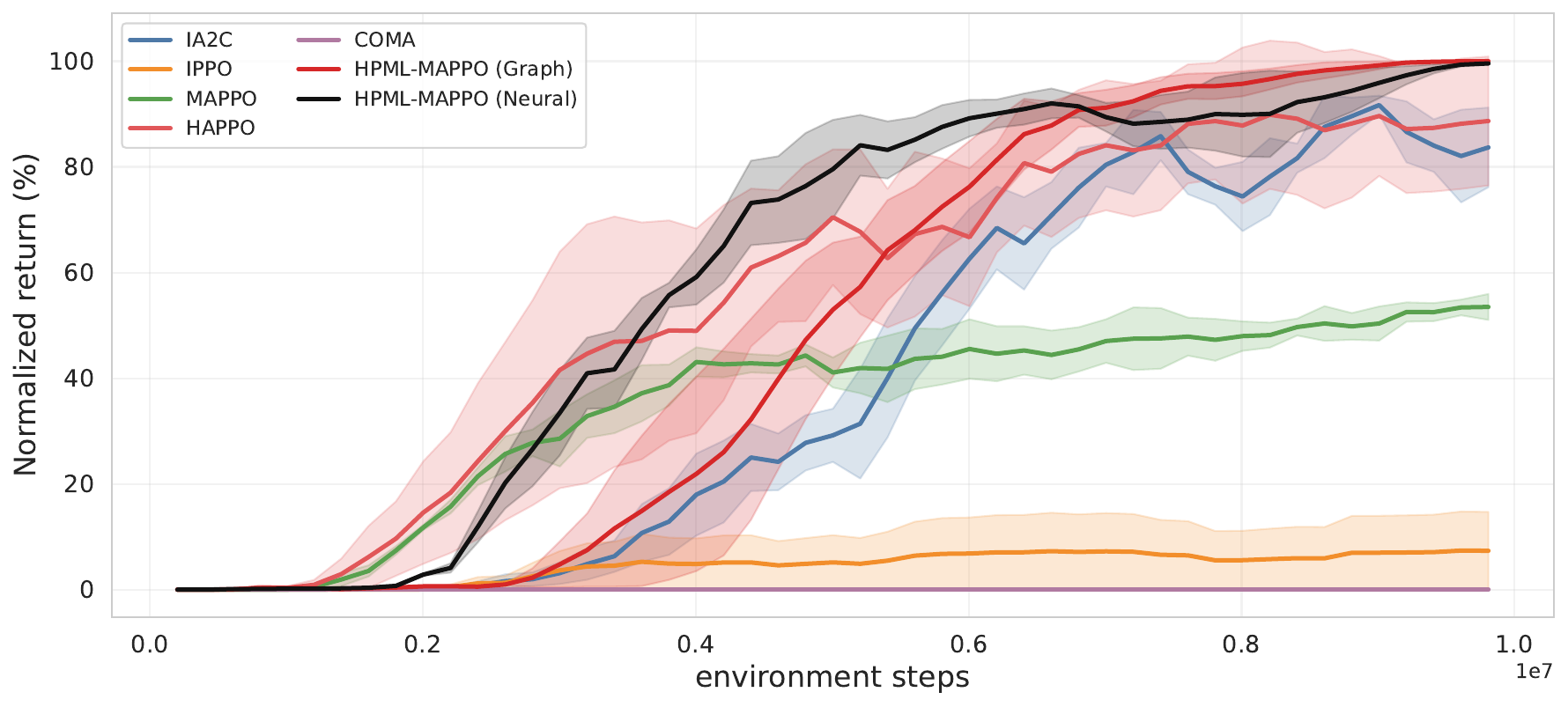}
\curveplaceholder{Rationalizable Coordination Repeated}{rc_rep}{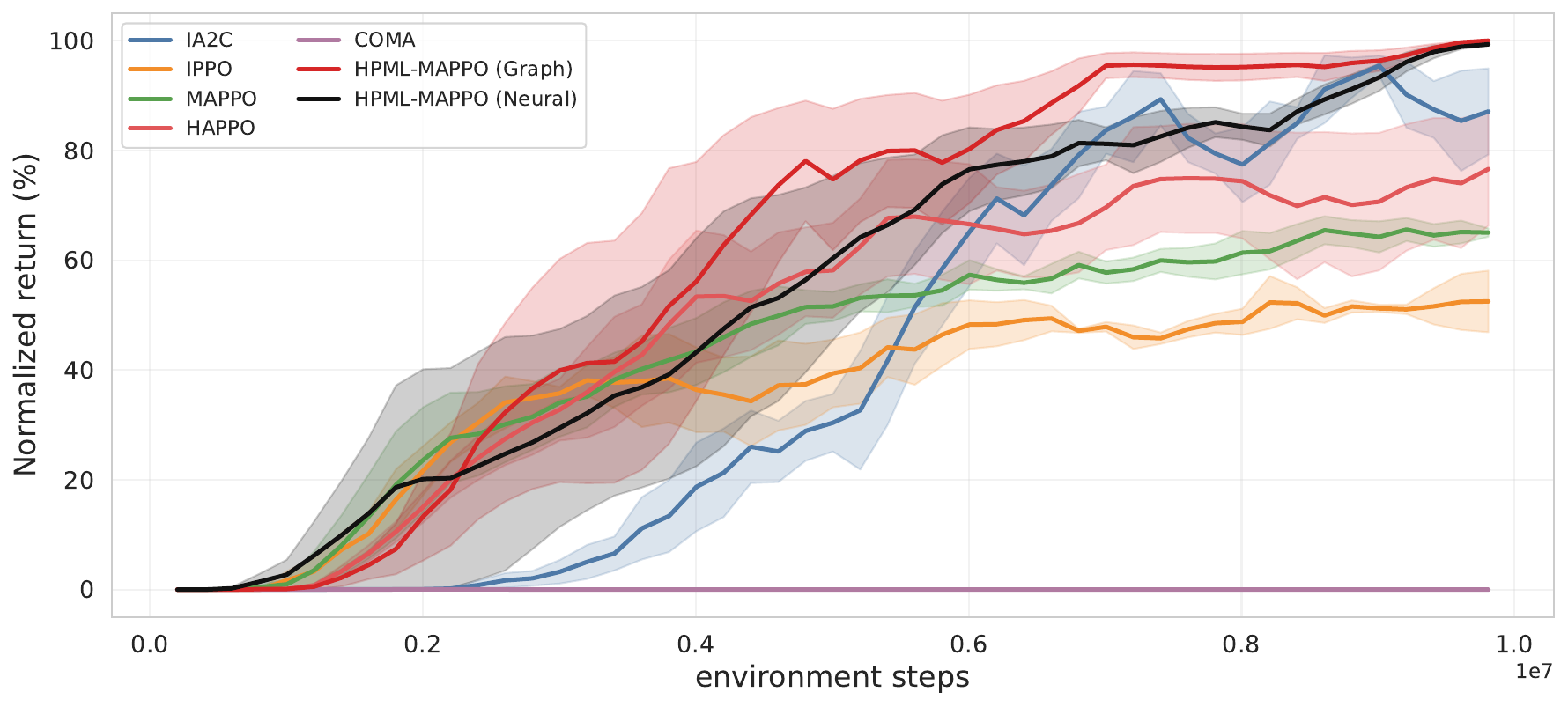}
\curveplaceholder{Stag Hunt Repeated}{sh_rep}{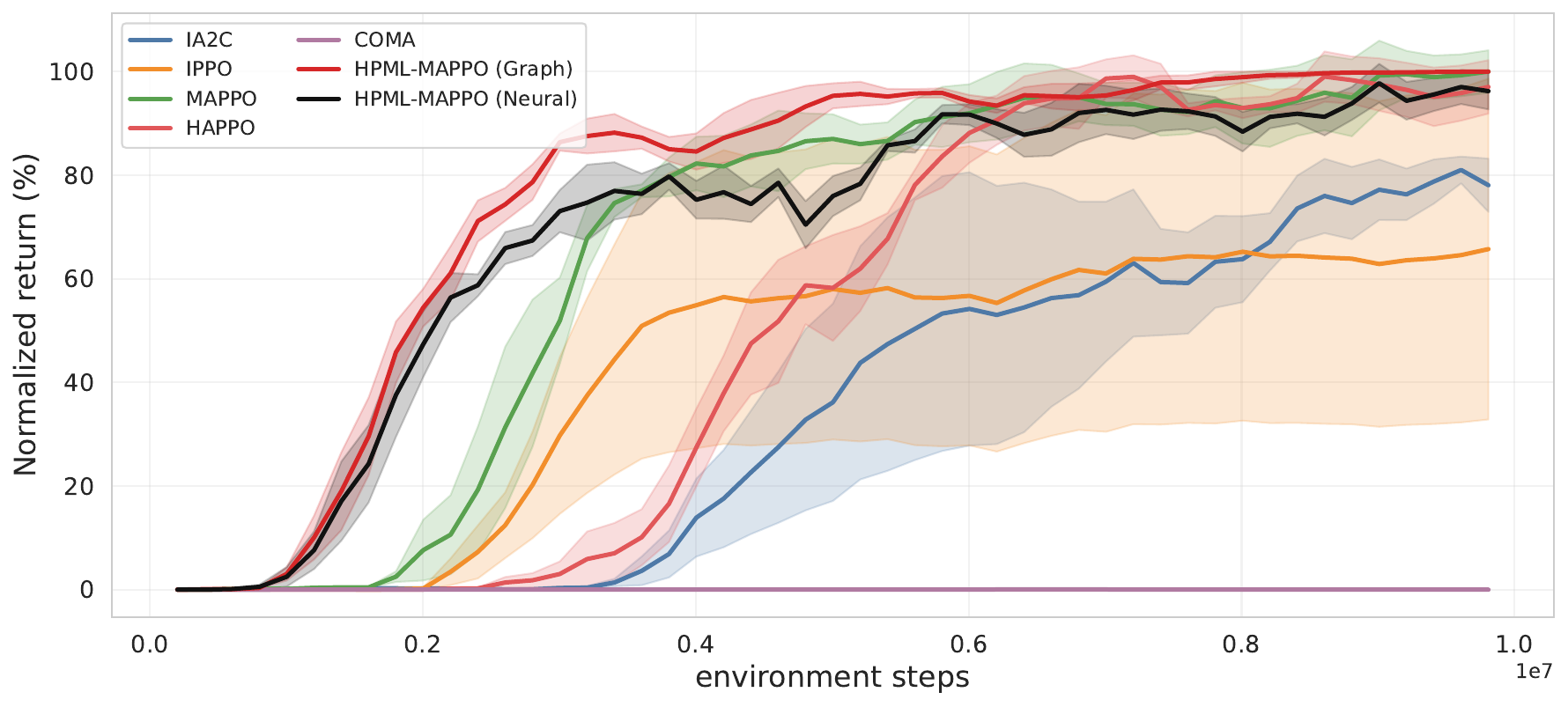}
\curveplaceholder{Prisoners Dilemma Repeated}{pd_rep}{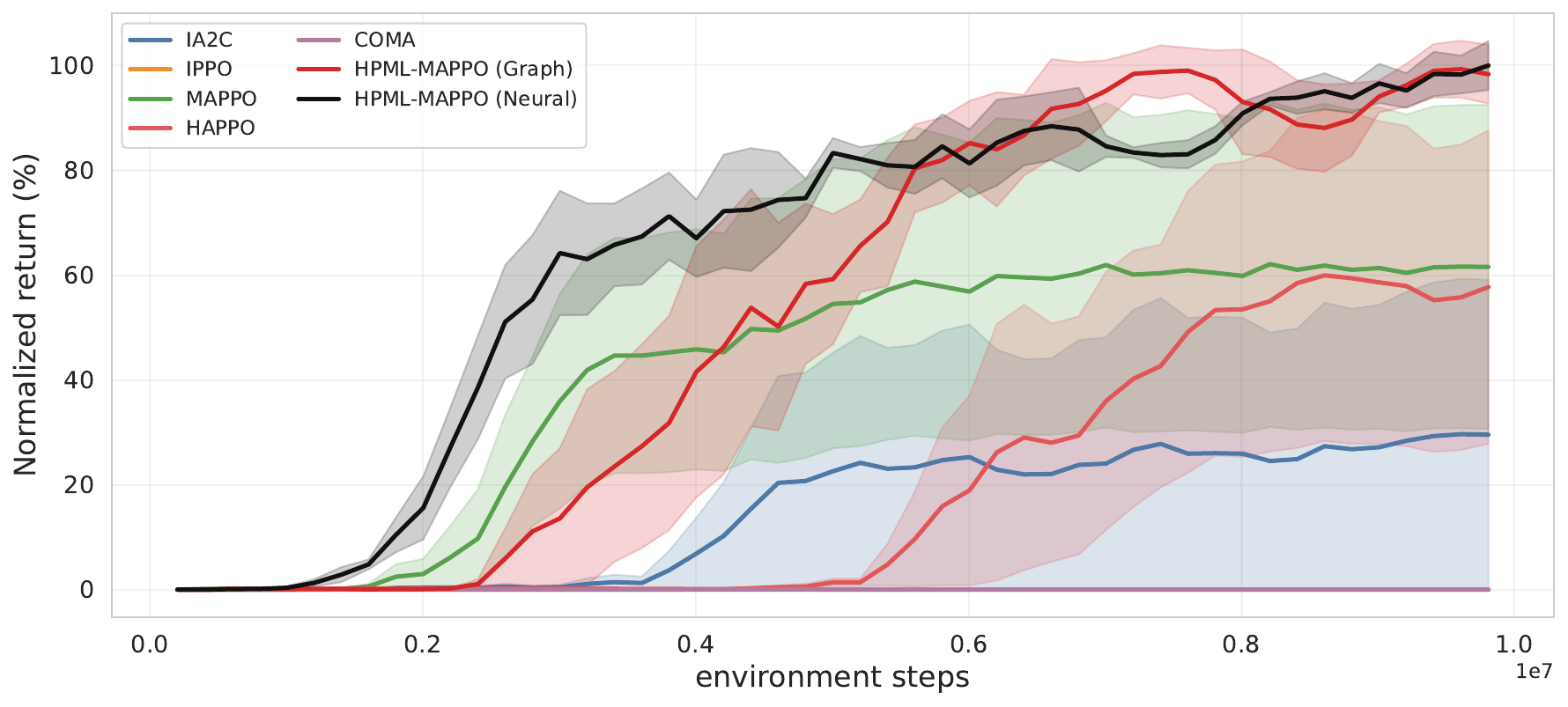}
\curveplaceholder{Collaborative Cooking: Crowded}{cc_crowded}{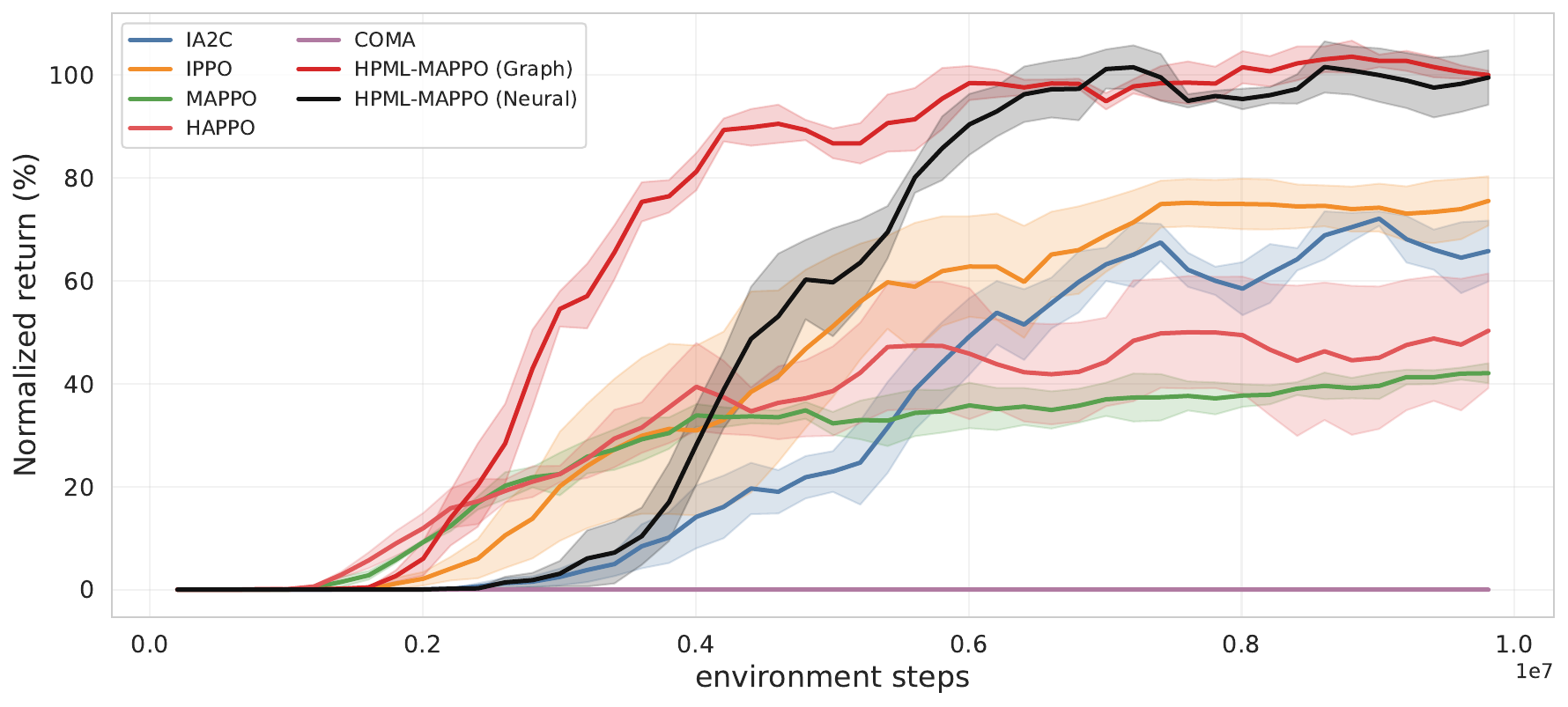}
\curveplaceholder{Collaborative Cooking: Figure Eight}{cc_figure8}{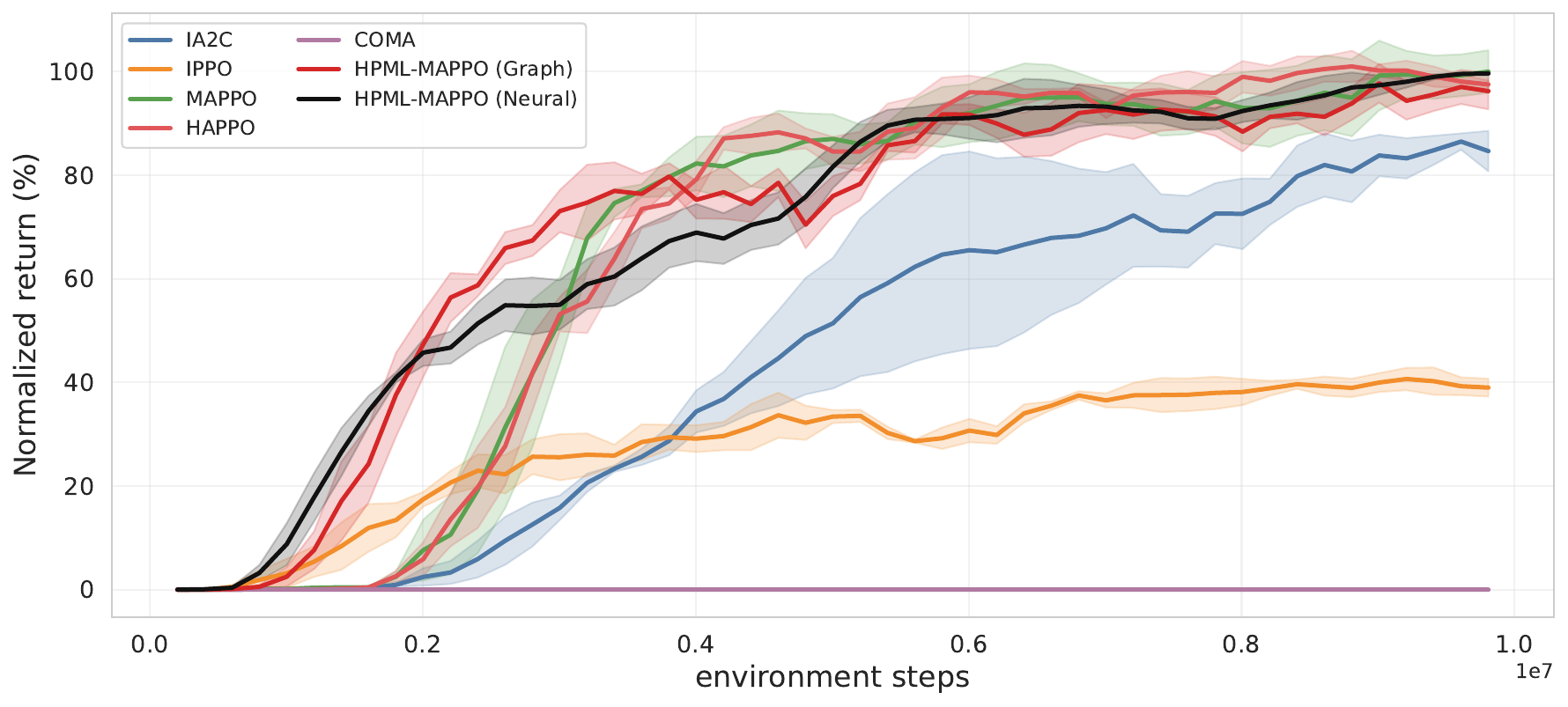}
\curveplaceholder{Commons Harvest: Partnership}{ch_partnership}{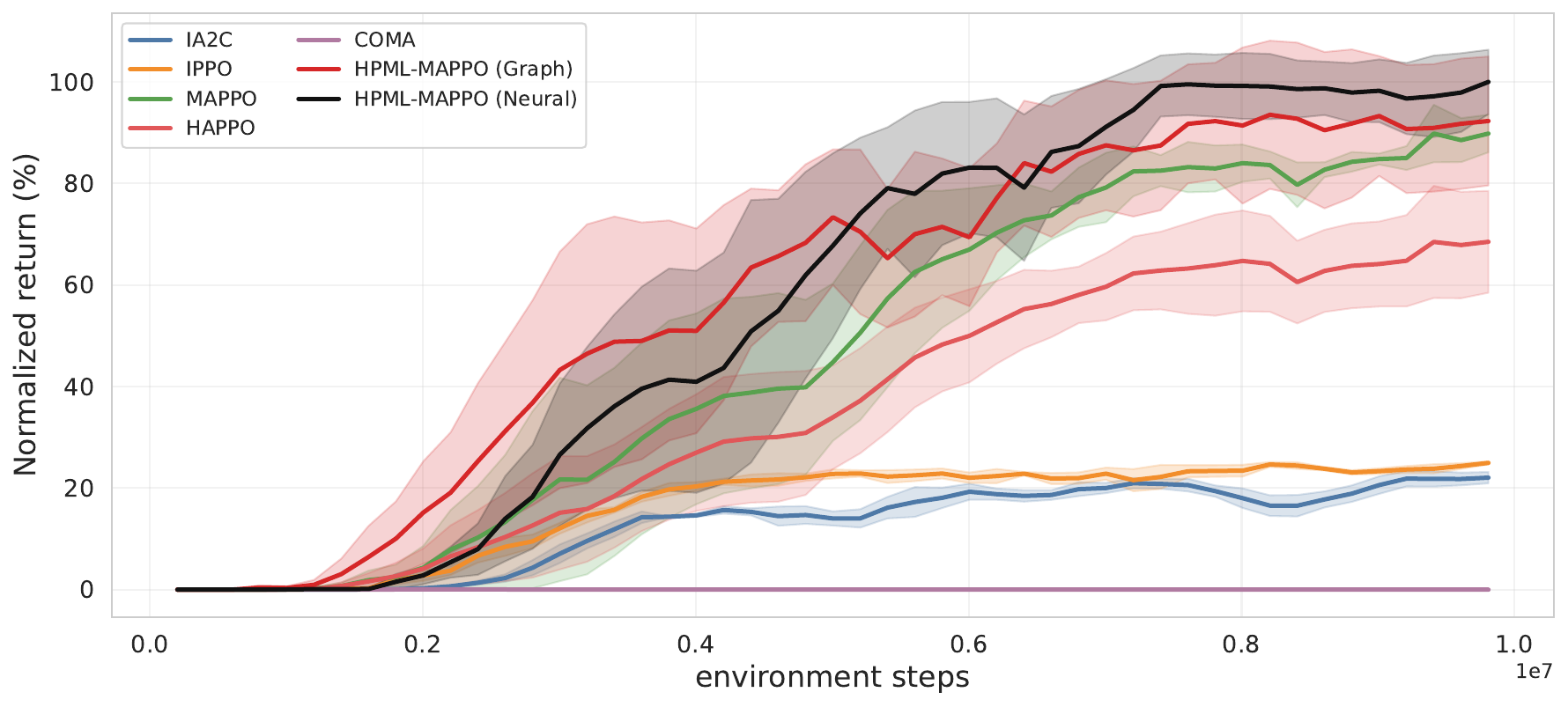}

\subsection{Full Melting Pot table (per-scenario)}
\label{app:exp_table_full}

\begin{table*}[t]
\centering
\scriptsize
\caption{Full Melting Pot results (mean normalized return \% $\pm$ s.e.) on all 10 scenarios.
Abbreviations: BoS-R = Bach or Stravinsky Repeated,
PC-R = Pure Coordination Repeated,
RC-R = Rationalizable Coordination Repeated,
SH-R = Stag Hunt Repeated,
PD-R = Prisoners Dilemma Repeated,
CC-Cir = Collaborative Cooking Circuit,
CC-Crowd = Collaborative Cooking Crowded,
CC-F8 = Collaborative Cooking Figure Eight,
CH-Part = Commons Harvest Partnership.}
\label{tab:meltingpot_full}
\resizebox{\textwidth}{!}{%
\begin{tabular}{lccccccccccc}
\toprule
Method
& BoS-R & PC-R & RC-R & SH-R & PD-R
& CC-Cir & CC-Crowd & CC-F8 & CH-Part & Clean Up & Macro Avg \\
\midrule
IA2C & 83.7 $\pm$ 7.5 & 83.7 $\pm$ 7.5 & 87.1 $\pm$ 7.8 & 78.1 $\pm$ 5.2 & 29.6 $\pm$ 29.6 & 88.3 $\pm$ 4.5 & 65.8 $\pm$ 5.9 & 84.6 $\pm$ 3.9 & 22.1 $\pm$ 1.1 & 70.6 $\pm$ 3.0 & 69.4 $\pm$ 3.4 \\
IPPO & 100.0 $\pm$ 4.1 & 7.4 $\pm$ 7.4 & 52.5 $\pm$ 5.6 & 65.7 $\pm$ 32.9 & 0.0 $\pm$ 0.0 & 30.7 $\pm$ 30.7 & 75.5 $\pm$ 4.8 & 39.0 $\pm$ 1.7 & 25.0 $\pm$ 0.3 & 84.4 $\pm$ 4.2 & 48.0 $\pm$ 4.7 \\
MAPPO & 86.3 $\pm$ 3.6 & 53.5 $\pm$ 2.5 & 65.0 $\pm$ 0.8 & 100.0 $\pm$ 4.1 & 61.6 $\pm$ 30.9 & 96.9 $\pm$ 4.9 & 42.1 $\pm$ 1.9 & 100.0 $\pm$ 4.1 & 89.8 $\pm$ 3.7 & 86.2 $\pm$ 3.6 & 78.1 $\pm$ 3.3 \\
HAPPO & 93.6 $\pm$ 3.9 & 88.7 $\pm$ 12.2 & 76.6 $\pm$ 10.3 & 97.0 $\pm$ 5.2 & 57.7 $\pm$ 29.9 & 82.7 $\pm$ 7.2 & 50.3 $\pm$ 11.1 & 97.5 $\pm$ 0.8 & 68.5 $\pm$ 10.0 & 81.1 $\pm$ 4.5 & 79.4 $\pm$ 3.9 \\
COMA & 0.0 $\pm$ 0.0 & 0.0 $\pm$ 0.0 & 0.0 $\pm$ 0.0 & 0.0 $\pm$ 0.0 & 0.0 $\pm$ 0.0 & 0.0 $\pm$ 0.0 & 0.0 $\pm$ 0.0 & 0.0 $\pm$ 0.0 & 0.0 $\pm$ 0.0 & 0.0 $\pm$ 0.0 & 0.0 $\pm$ 0.0 \\
\midrule
HPML-MAPPO (Graph) & 100.0 $\pm$ 0.0 & 100.0 $\pm$ 0.0 & 100.0 $\pm$ 0.0 & 100.0 $\pm$ 0.0 & 98.3 $\pm$ 5.6 & 100.0 $\pm$ 0.0 & 100.0 $\pm$ 0.8 & 96.2 $\pm$ 3.5 & 92.3 $\pm$ 12.7 & 100.0 $\pm$ 25.9 & 98.7 $\pm$ 3.0 \\
HPML-MAPPO (Neural) & 88.7 $\pm$ 12.2 & 99.6 $\pm$ 0.2 & 99.3 $\pm$ 0.2 & 96.2 $\pm$ 3.5 & 100.0 $\pm$ 4.7 & 99.9 $\pm$ 0.0 & 99.5 $\pm$ 5.3 & 99.6 $\pm$ 0.2 & 100.0 $\pm$ 6.3 & 98.7 $\pm$ 0.2 & 98.2 $\pm$ 1.6 \\
% Method
% & BoS-R & PC-R & RC-R & SH-R & PD-R
% & CC-Cir & CC-Crowd & CC-F8 & CH-Part & Clean Up & Macro Avg \\
% \midrule
% IA2C  & -- & -- & -- & -- & -- & -- & -- & -- & -- & -- & -- \\
% IPPO  & -- & -- & -- & -- & -- & -- & -- & -- & -- & -- & -- \\
% MAPPO & -- & -- & -- & -- & -- & -- & -- & -- & -- & -- & -- \\
% HAPPO & -- & -- & -- & -- & -- & -- & -- & -- & -- & -- & -- \\
% COMA  & -- & -- & -- & -- & -- & -- & -- & -- & -- & -- & -- \\
% \midrule
% HPML-MAPPO (Graph)  & -- & -- & -- & -- & -- & -- & -- & -- & -- & -- & -- \\
% HPML-MAPPO (Neural) & -- & -- & -- & -- & -- & -- & -- & -- & -- & -- & -- \\
\bottomrule
\end{tabular}}
\end{table*}

\subsection{Additional diagnostics: non-potentiality and convention switching}
\label{app:exp_diag}

We additionally track $\mathrm{NonPot}(\omega)$ during training for all scenarios. For the convention-heavy tasks---Bach or Stravinsky Repeated, Pure Coordination Repeated, and Rationalizable Coordination Repeated---we also report an empirical convention-switch rate to test whether HPML changes the geometry of learning rather than only the reward scale.

\section{Additional Notes for Section~\ref{sec:prelim}}
\label{app:prelim_notes}

\begin{remark}[Equilibrium operator and algorithmic update field]
\label{rem:operator_app}
For equilibrium analysis, the canonical VI operator is the negative pseudo-gradient
\[
F_{\mathrm{VI}}(\pi)=\big(-\nabla_{\pi_1}J_1(\pi),\ldots,-\nabla_{\pi_n}J_n(\pi)\big).
\]
The corresponding simultaneous ascent/update field is $U=-F_{\mathrm{VI}}$.
HPML projects the supplied update field $U$; if the available object is a VI operator, one first changes sign and applies the projection to $-F_{\mathrm{VI}}$.
Some sections use $F$ as a generic field to avoid excessive notation, but the intended sign is always determined by whether the object is being used for VI gap evaluation or for an actual update.
\end{remark}

\begin{remark}[Primal policies and unconstrained parameterizations]
\label{rem:param_app}
The theory is stated on the compact convex feasible set $\cX$ of stationary policies in primal coordinates.
In implementations, one often optimizes unconstrained parameters such as logits and maps them back to $\cX$
through a retraction or mirror map such as softmax.
This changes the coordinate representation of the iterate, but not the conceptual role of $F$
as a field defined on feasible policies.
\end{remark}

\begin{remark}[Stampacchia and Minty gaps]
\label{rem:minty_app}
Definition~\ref{def:gap} uses the Stampacchia duality gap
\[
\Gap(\bar x)=\max_{x\in\cX}\ip{F(\bar x)}{\bar x-x},
\]
which evaluates the operator at the reference point $\bar x$.
Another common quantity is the Minty gap
\[
\Gap_{\mathrm{Minty}}(\bar x)
:=
\max_{x\in\cX}\ip{F(x)}{\bar x-x},
\]
which corresponds to the Minty VI formulation.
The present paper uses the Stampacchia gap because it aligns directly with the projected operator analysis developed later.
\end{remark}

\begin{remark}[Gap and monotonicity]
\label{rem:monotone_gap_app}
Lemma~\ref{lem:gap_zero_vi} is purely algebraic and does not require monotonicity.
When $F$ is monotone, however, the duality gap becomes a standard merit function that quantitatively certifies closeness to the VI solution set and supports classical convergence arguments.
In general-sum games, the induced operator is typically non-monotone, which is precisely why cycling may persist and why a structural modification of the field is useful.
\end{remark}

\begin{example}[Canonical bilinear cyclic-interaction field]
\label{ex:bilinear_app}
Consider the bilinear min--max problem
\[
\min_u\max_v\; u^\top v
\]
with associated field
\[
F(u,v)=(v,-u).
\]
Its Jacobian is
\[
J=\nabla F=
\begin{bmatrix}
0 & I\\
-I & 0
\end{bmatrix},
\qquad
S=0,
\qquad
A=J.
\]
Hence the field is purely antisymmetric and has no potential component in the local Jacobian sense.
This is the canonical example showing that simultaneous first-order updates may follow cyclic orbits around the origin rather than converge.
\end{example}

\begin{remark}[Local and global non-potentiality]
\label{rem:localglobal_app}
Proposition~\ref{prop:grad_symm} identifies antisymmetry of the Jacobian as a \emph{local} obstruction to integrability.
This is not yet a global statement: even if a field looks locally potential-like, global or sample-level obstructions may remain.
The role of HPML in the main text is to move from this local Jacobian diagnostic to a global $L^2$ projection of the full joint field, thereby producing a residual that measures non-potentiality at the vector-field level.
\end{remark}

\section{Variational and Geometric Details for HPML}
\label{app:hpml_variational}

This appendix records the continuous Euler--Lagrange characterization of the projection,
its circulation interpretation, and several geometric remarks that are not needed in the main text
but clarify the structure behind HPML.

\begin{lemma}[Weighted integration by parts]
\label{lem:adjoint}
Let $\cX\subset\R^d$ be a bounded Lipschitz domain or a flat torus.
Let $d\mu(x)=\rho(x)\,dx$ with $0<\rho_{\min}\le \rho(x)\le \rho_{\max}<\infty$ a.e.\ on $\cX$.
Assume boundary conditions that eliminate boundary terms
(e.g., periodic boundary conditions on the torus; or $\rho\,U\cdot n=0$ on $\partial\cX$).
Then for any $\psi\in H^1(\cX)$ and any vector field $U$ for which $\div_\mu U$ is well-defined in $H^{-1}(\cX)$,
\[
\int_{\cX} U^\top \nabla\psi \, d\mu
=
-\int_{\cX} \psi\, \div_\mu U \, d\mu.
\]
Equivalently, $\nabla$ and $-\div_\mu$ are adjoint under the $L^2(\mu)$ pairing.
\end{lemma}

\begin{proposition}[Poisson-type characterization of the projection]
\label{prop:poisson}
Assume $M\succ 0$ is constant and $d\mu(x)=\rho(x)\,dx$ with $\rho>0$.
Assume boundary conditions that remove boundary terms in Lemma~\ref{lem:adjoint}.
Let $\Phi^\star$ be any minimizer in Definition~\ref{def:proj}.
Then the projected field is $\Pi_{\mathrm{grad}}^{M}F=\nabla_M\Phi^\star$ and $\Phi^\star$ satisfies the weak equation
\[
\cL_{M,\mu}\Phi^\star
=
\div_\mu F,
\qquad\text{i.e.,}\qquad
\div_\mu\!\big(M^{-1}\nabla\Phi^\star\big)
=
\div_\mu F.
\]
Moreover, $\Phi^\star$ is unique up to an additive constant; one may fix a gauge such as $\E_\mu[\Phi^\star]=0$.
A natural boundary condition implied by the variational optimality is the zero-flux residual condition
\[
\rho\,\big(M^{-1}\nabla\Phi^\star-F\big)\cdot n
=
0
\quad\text{on }\partial\cX,
\]
whenever traces are well-defined.
\end{proposition}

\begin{remark}[Local Jacobian diagnostic and global projection]
\label{rem:jac_bridge}
The Jacobian split in Section~\ref{sec:prelim} is a local diagnostic:
antisymmetry in $\nabla F$ signals local cyclic interaction behavior.
The HPML projection is a global $L^2(M,\mu)$ construction on the full field.
For a locally linear field $F(x)\approx Bx$, Euclidean metric $M=I$, and isotropic sampling,
the closest gradient field corresponds to the symmetric part of $B$, namely
$B\mapsto \tfrac12(B+B^\top)$.
\end{remark}

\begin{remark}[Relation to Hodge decomposition]
\label{rem:hodge_app}
Lemma~\ref{lem:orth} is the structural statement used in the main text.
On domains or complexes with appropriate topology and boundary conditions,
the residual can often be refined into curl-like and harmonic components.
The discrete construction used by HPML mirrors this picture on graphs or sample complexes,
but the main text only requires the projected gradient component and its orthogonal residual.
\end{remark}

\paragraph{Convention-level cyclic interaction as circulation.}
Cyclic interaction dynamics need not come from explicit symmetry.
In CTDE training, one often observes repeated revisits to a small family of recurring joint behaviors.
Geometrically, such behavior corresponds to nontrivial circulation of the joint update field along approximately closed loops in the visited region of $\cX$.

For a piecewise $C^1$ closed curve $\gamma:[0,1]\to\cX$ with $\gamma(0)=\gamma(1)$, define
\[
\Circ_F(\gamma)
:=
\int_0^1 \langle F(\gamma(t)),\dot\gamma(t)\rangle_M\,dt
=
\int_0^1 \dot\gamma(t)^\top M F(\gamma(t))\,dt.
\]

\begin{proposition}[Closed-loop circulation is carried by the non-potential residual]
\label{prop:circulation_residual}
Let $G^\star := \Pi_{\mathrm{grad}}^{M}F$ and $R:=F-G^\star$.
Assume $G^\star=\nabla_M\Phi^\star$ with $\Phi^\star$ sufficiently regular along the curve under consideration.
For any piecewise $C^1$ closed curve $\gamma$ in $\cX$,
\[
\Circ_{G^\star}(\gamma)=0,
\qquad\text{and hence}\qquad
\Circ_F(\gamma)=\Circ_R(\gamma).
\]
Moreover, if $\mu_\gamma$ denotes the pushforward of the uniform measure on $[0,1]$ by $\gamma$, then
\[
|\Circ_F(\gamma)|
\le
\Big(\E_{t\sim\mathrm{Unif}[0,1]}\|R(\gamma(t))\|_M^2\Big)^{1/2}
\Big(\E_{t\sim\mathrm{Unif}[0,1]}\|\dot\gamma(t)\|_{M^{-1}}^2\Big)^{1/2}.
\]
\end{proposition}

\begin{remark}[Operational reading of failed collaboration loops]
\label{rem:failed_loops}
Proposition~\ref{prop:circulation_residual} formalizes the intuition that recurring convention cycles correspond to nonzero circulation, which cannot be generated by any potential flow.
Accordingly, the non-potentiality in Definition~\ref{def:nonpot}, or its empirical discrete proxy, can be read as a quantitative certificate of such looping behavior.
\end{remark}

\begin{remark}[Smoothly varying metrics]
\label{rem:variable_metric}
If the metric varies smoothly as $M(x)$, the variational definition in Definition~\ref{def:proj} remains unchanged:
one still projects onto metric-gradients under the local inner product induced by $M(x)$.
The Euler--Lagrange condition becomes a variable-coefficient elliptic equation in weak form.
The main text keeps the constant-$M$ presentation for clarity.
\end{remark}

\begin{remark}[Discrete and amortized implementations]
\label{rem:discrete_impl_app}
The practical implementations of HPML do not solve the continuous Poisson problem directly.
Instead, they approximate the same variational objective either on a graph or sample complex,
or through an amortized potential model.
The concrete procedures appear in Algorithm~\ref{alg:hodge-proj} and Algorithm~\ref{alg:neural-hpml}.
\end{remark}

\section{Graph Construction and Discrete Recovery Details}
\label{app:discrete_details}

This appendix records the implementation-level details of the sample-graph construction,
the graph interpretation of the residual, the local lifting formulas, and the graph-based
projection subroutine used by HPML.

\paragraph{Graph construction.}
Let $\{x^1,\dots,x^N\}\subset\cX$ be sampled joint-policy points.
We build an undirected $k$-NN graph $G=(\cV,\cE)$ on the samples, where $\cV=\{1,\dots,N\}$ and
$(i,j)\in\cE$ indicates that $x^j$ is among the $k$ nearest neighbors of $x^i$ (symmetrized).
If a metric $M\succ 0$ is used in Section~\ref{sec:hpml}, neighbors may be computed using
\[
\dist_M(x^i,x^j):=\sqrt{(x^i-x^j)^\top M (x^i-x^j)};
\]
otherwise Euclidean distance is used.
Each undirected edge is assigned a positive weight $w_{ij}=w_{ji}>0$.

\begin{remark}[Cycle-space interpretation on a graph]
\label{rem:graph_cycle_space}
On a pure graph, the space of edge flows decomposes into the cut space (gradients) and the cycle space (circulations).
Thus $\omega_{\mathrm{cyc}}$ is precisely the cycle-space component of $\omega$.
If one augments the graph to a higher-order simplicial complex, one can further refine the decomposition into gradient, curl, and harmonic components, but the graph-only variant already removes all non-gradient circulation supported by cycles.
\end{remark}

\begin{remark}[Convention graphs and identifying which patterns form the cycle]
\label{rem:convention_graph}
When samples $\{x_i\}$ cluster around a small number of recurring joint behaviors,
the $k$-NN graph can be viewed as a convention graph whose nodes represent these patterns and whose edges capture local transitions induced by the update field.
In this view, the residual $\omega_{\mathrm{cyc}}$ in Definition~\ref{def:disc_proj} is exactly the cycle-space component on the convention graph, and $\mathrm{NonPot}(\omega)$ in Definition~\ref{def:disc_nonpot} quantifies how much of the observed update energy is explained only by cycling among conventions.

Moreover, one can localize which conventions participate in the loop by expanding $\omega_{\mathrm{cyc}}$ on a cycle basis of the graph, e.g.\ fundamental cycles from a spanning tree.
Large coefficients identify the specific subset of conventions responsible for the recurrent dynamics.
\end{remark}

\subsection{Node-wise lifting formulas}
\label{app:discrete_lifting}

The graph projection outputs a scalar potential $\phi^\star$.
To obtain a vector direction at node $i$, define the neighborhood
\[
\cN(i):=\{j:(i,j)\in\cE\}.
\]
We fit $h^i\in\R^d$ as an approximation of $\nabla\Phi$ by solving
\begin{equation}
h^i \in \arg\min_{h\in\R^d}\;
\sum_{j\in\cN(i)} w_{ij}
\Big(
\langle h,x^j-x^i\rangle-(\phi^\star_j-\phi^\star_i)
\Big)^2
+\lambda\|h\|_2^2,
\label{eq:lifting_ls_h}
\end{equation}
where $\lambda\ge 0$ is a ridge parameter.
We then set
\begin{equation}
g^i:=M^{-1}h^i.
\label{eq:lifting_metric}
\end{equation}

In matrix form,
\[
h^i=(X_i^\top W_i X_i+\lambda I)^{-1}X_i^\top W_i y_i,
\]
where rows of $X_i$ are $(x^j-x^i)^\top$ and entries of $y_i$ are $(\phi^\star_j-\phi^\star_i)$.

\begin{remark}[Updating a query point not in the sample set]
\label{rem:query_point_graph}
If the current iterate $x_t$ is not one of the sampled nodes, one may insert $x_t$ into the graph by connecting it to its $k$ nearest sampled neighbors, compute local potential differences, and apply the same lifting step to obtain $g_t$.
\end{remark}

\begin{proposition}[Metric-consistent discretization of potential differences]
\label{prop:metric_discrete}
Let $M\in\R^{d\times d}$ be constant SPD and let $\Phi\in C^2(\cX)$.
For any $x\in\cX$ and any displacement $\Delta x$ with $x+\Delta x\in\cX$,
\[
\Phi(x+\Delta x)-\Phi(x)
=
\langle \nabla\Phi(x),\Delta x\rangle
+
O(\|\Delta x\|^2)
=
\langle M\nabla_M\Phi(x),\Delta x\rangle
+
O(\|\Delta x\|^2).
\]
Consequently, on a graph with node positions $\{x^i\}$ and edges $(i,j)$, the metric-consistent edge $1$-form
associated with a vector field $G$ intended to approximate $\nabla_M\Phi$ is
\[
\omega_{ij}
:=
\Big\langle
M\,\frac{G(x^i)+G(x^j)}{2},
\,x^j-x^i
\Big\rangle,
\]
and if $G=\nabla_M\Phi$ then
\[
\omega_{ij}
=
\Phi(x^j)-\Phi(x^i)+O(\|x^j-x^i\|^2).
\]
\end{proposition}

\subsection{Graph-based projection algorithm}
\label{app:discrete_algorithm}

\begin{algorithm}[t]
\caption{Discrete Hodge Projection (Graph Poisson Solve)}
\label{alg:hodge-proj}
\begin{algorithmic}[1]
\REQUIRE Samples $\{x^i\}_{i=1}^N\subset\cX$, field estimates $\{F_i=F(x^i)\}$, metric $M\succ 0$ (optional), weighted $k$-NN graph $G=(\cV,\cE)$, ridge $\lambda\ge 0$.
\ENSURE Node potential $\phi^\star$, potential edge flow $\omega_{\mathrm{pot}}$, residual $\omega_{\mathrm{cyc}}$, and optionally lifted node directions $\{g^i\}$.
\STATE Choose an orientation for each undirected edge and form $\cE^\rightarrow$.
\STATE Construct the incidence matrix $B$ and the diagonal weight matrix $W=\diag(w_e)$.
\FOR{each oriented edge $e=(i\to j)\in\cE^\rightarrow$}
    \STATE $\Delta x_e \gets x^j-x^i$
    \STATE $\omega_e \gets \left\langle M\,\tfrac12(F_i+F_j),\,\Delta x_e\right\rangle$
\ENDFOR
\STATE Form $L\gets B^\top W B$ and $b\gets B^\top W\omega$.
\STATE Solve $L\phi^\star=b$ with a gauge condition.
\STATE $\omega_{\mathrm{pot}}\gets B\phi^\star$, \quad $\omega_{\mathrm{cyc}}\gets \omega-\omega_{\mathrm{pot}}$.
\IF{node directions are needed}
    \FOR{$i=1$ to $N$}
        \STATE Fit $h^i$ by \eqref{eq:lifting_ls_h}.
        \STATE Set $g^i\gets M^{-1}h^i$.
    \ENDFOR
\ENDIF
\end{algorithmic}
\end{algorithm}

\section{Algorithmic Procedures for HPML}
\label{app:alg_realizations}

This appendix records the generic outer-loop realization of HPML, the amortized neural realization, and the CTDE plug-in realization.
The graph-based projection subroutine itself is given in Algorithm~\ref{alg:hodge-proj}.

\begin{algorithm}[t]
\caption{Hodge-Projected Multi-Agent Learning (HPML)}
\label{alg:hpml}
\begin{algorithmic}[1]
\REQUIRE Initial $x_0\in\cX$, stepsizes $\{\eta_t\}$, update-field oracle/estimator $F(\cdot)$,
metric $M\succ 0$ (optional), feasibility operator $\Retr(\cdot)$,
projection routine (discrete graph or amortized neural projection).
\FOR{$t=0,1,2,\dots$}
  \STATE Estimate $F(x_t)$.
  \STATE Obtain a projected direction $g_t \approx \Pi_{\mathrm{grad}}^{M}F(x_t)$:
  \STATE \hspace{1em} either (i) build/update a sample graph and run Algorithm~\ref{alg:hodge-proj} then lift to get $g_t$, or
  \STATE \hspace{1em} (ii) use an amortized potential network to compute $g_t$ by autograd.
  \STATE $x_{t+1} \gets \Retr(x_t + \eta_t g_t)$.
\ENDFOR
\end{algorithmic}
\end{algorithm}

\begin{algorithm}[t]
\caption{Neural HPML (Amortized Potential Projection)}
\label{alg:neural-hpml}
\begin{algorithmic}[1]
\REQUIRE Initial $x_0\in\cX$, stepsizes $\{\eta_t\}$, feasibility operator $\Retr(\cdot)$,
field estimator $\widehat F(\cdot)$, buffer $\mathcal{B}$, inner steps $K$, batch size $B$,
(optional) metric $M\succ 0$.
\STATE Initialize potential-network parameters $\theta$.
\FOR{$t=0,1,2,\ldots$}
  \STATE Append $x_t$ (and auxiliary rollout data, if any) to buffer $\mathcal B$.
  \STATE Sample $\{x_k\}_{k=1}^B\sim\mathcal B$ and compute $\{\widehat F(x_k)\}$.
  \FOR{$k=1$ to $K$}
     \STATE Update $\theta$ by SGD on $\widehat{\mathcal L}(\theta)$ in \eqref{eq:nn-proj-loss}.
  \ENDFOR
  \STATE Compute $g_t \gets M^{-1}\nabla_x\Phi_\theta(x_t)$ by autograd.
  \STATE $x_{t+1} \gets \Retr(x_t + \eta_t g_t)$.
\ENDFOR
\end{algorithmic}
\end{algorithm}

\begin{remark}[Practical schedules for amortized projection]
\label{rem:neural_practical_app}
The amortized projection can be updated either online (a few inner SGD steps per outer iteration) or periodically (larger refreshes every $H$ outer steps).
If $M$ varies with $x$, one may use a diagonal or low-rank approximation of $M(x)$, or provide metric-related features to the potential network.
These choices affect approximation quality and computational overhead, but do not change the projection objective itself.
\end{remark}

\begin{algorithm}[t]
\caption{CTDE-HPML (HPML as a Plug-in Projection Layer for MARL)}
\label{alg:ctde-hpml}
\begin{algorithmic}[1]
\REQUIRE Decentralized actors $\{\pi_{\theta_i}(a_i\mid o_i)\}_{i=1}^n$, centralized critic $V_\psi$ (or $Q_\psi$),
stepsizes $\{\eta_t\}$, projection routine $\Pi_{\mathrm{grad}}^{M}(\cdot)$ (discrete or neural),
buffer $\mathcal B$, and an optional feasibility operator $\Retr(\cdot)$.
\FOR{$t=0,1,2,\dots$}
  \STATE Roll out trajectories with decentralized execution and append transitions to $\mathcal B$.
  \STATE Sample a minibatch from $\mathcal B$ and update the centralized critic parameters $\psi$.
  \STATE Compute a centralized advantage (or TD-based signal) $\widehat A(s,a)$.
  \STATE For each agent $i$, form a stochastic ascent direction
  \[
  \widehat g_i
  \leftarrow
  \E\big[\nabla_{\theta_i}\log\pi_{\theta_i}(a_i\mid o_i)\,\widehat A(s,a)\big].
  \]
  \STATE Stack the joint field $\widehat F \leftarrow (\widehat g_1,\dots,\widehat g_n)$.
  \STATE Compute projected direction $g_t \leftarrow \Pi_{\mathrm{grad}}^{M}\widehat F$.
  \STATE Update each actor:
  \[
  \theta_i \leftarrow \Retr(\theta_i+\eta_t[g_t]_i),
  \qquad i=1,\dots,n.
  \]
\ENDFOR
\end{algorithmic}
\end{algorithm}

\begin{remark}[Parameter-space realization under CTDE]
\label{rem:ctde_param_space_app}
The main theory is stated on a convex feasible set of policy coordinates, whereas practical CTDE implementations often operate directly in actor-parameter space.
HPML is used there as an operator-level projection layer on the stacked update field.
This does not change the conceptual role of the residual: it still measures the non-integrable part of the joint improvement dynamics, but now in the chosen parameterization.
\end{remark}

\begin{remark}[Relation to Bellman residuals and the multi-agent source of circulation]
\label{rem:bellman_bridge}
The role of Proposition~\ref{prop:residual_bellman_obstruction} is interpretive rather than evaluative.
It does \emph{not} say that single-agent Bellman or TD operators are ill-defined.
The issue appears only after stacking several critic-driven local actor-improvement directions into one joint field.
Each block $g_i$ may be a valid local greedy direction, yet the collection $(g_1,\dots,g_n)$ may fail to be jointly integrable because the cross-agent response terms need not be symmetric.
That is the sense in which there is a ``mismatch'': changing agent $j$ can alter agent $i$'s improvement direction without a matching reciprocity in the reverse block.

This phenomenon is natural in multi-agent CTDE because the joint control problem is often factorized into simultaneously updated local blocks for scalability. The decisive distinction is again exact centralized optimization versus centralized evaluation alone. A centralized critic and even a single shared reward do not by themselves remove the obstruction if the implemented actor update is still assembled from local or approximate blocks. The obstruction disappears only when all agents are optimized as one super-agent by the exact gradient of a single centralized scalar objective in the full joint parameter space. Sampling other agents' actions then contributes stochasticity, but not a structural non-potential cyclic component at the level of the exact expected field. Partial observability and approximation error can exacerbate the effect but are not necessary. Mathematically, non-potentiality is not exclusive to multi-agent systems, but multi-agent blockwise updates are a particularly natural mechanism generating it.

So the residual is not a Bellman-evaluation error. It is a structural inconsistency in the joint greedy-improvement dynamics induced by decentralized multi-agent updates.
\end{remark}

\subsection{CTDE integrability and residual-only cyclic dynamics}
\label{app:ctde_integrability}

\begin{remark}[Origin of the non-potential cyclic component]
\label{rem:residual_quantitative}
The nonzero loop integral in Proposition~\ref{prop:residual_bellman_obstruction} is generated by a local asymmetry in the Jacobian of the one-form $MF$. Writing
$
Y(x):=M F(x),
\qquad
K(x):=\nabla Y(x)=M\nabla F(x),
$
joint integrability requires $K(x)=K(x)^\top$.
Equivalently, the antisymmetric part
$
B(x):=\tfrac12\bigl(M\nabla F(x)-\nabla F(x)^\top M\bigr)
$
must vanish.
If $B(x)\neq 0$, then for any sufficiently small parallelogram $C_\varepsilon(x;u,v)\subset U$,
$
\oint_{C_\varepsilon(x;u,v)} \langle M F(z),dz\rangle
=
2\varepsilon^2 \langle B(x)u,v\rangle + o(\varepsilon^2),
$
so the field has a local circulation component.

In the two-agent Euclidean case $M=I$ with $x=(x_1,x_2)$ and $F=(g_1,g_2)$, this becomes
$
\oint_{C_\varepsilon} \langle F(z),dz\rangle
=
\varepsilon^2 h^\top\!\bigl(\nabla_{x_2}g_1(x)-\nabla_{x_1}g_2(x)^\top\bigr)k
+
o(\varepsilon^2).
$
Thus the relevant ``mismatch'' is not merely that agents prefer different actions. It is the failure of cross-agent response symmetry: changing agent $2$ alters agent $1$'s greedy improvement direction in a way that is not matched by the reverse perturbation. Once this occurs, the stacked field cannot be the gradient of a single common potential.

This mismatch is typically created by strategic coupling in the environment together with decentralized or blockwise actor updates. Importantly, the issue is not determined by reward sharing alone. A single shared team reward and a centralized critic remove the non-potential cyclic component only when the realized joint actor-update field is exactly the gradient of the centralized team objective. If the same MARL system is trained through separate local or approximate improvement blocks, then one agent's update is still formed against the changing behavior of the others, and the induced cross-agent response terms need not be symmetric. Shared rewards, coupled transitions, role-allocation constraints, convention selection, or mixed cooperative/competitive effects can all make one agent's local improvement direction reshape another agent's landscape. Partial observability and function-approximation error can strengthen the asymmetry, but they are not required; the effect can already appear in fully observed matrix or Markov games. Hyperparameters that change the update operator itself can also change the size of the residual. By contrast, the numerical step size and the initial condition do not create the residual field; they determine whether training visits a region where the asymmetry is strong and whether the resulting cyclic motion is amplified in discrete time.

This also clarifies the multi-agent aspect. The decisive distinction is exact centralized optimization, not merely centralized evaluation. If all agents were merged into a single super-agent and updated by exact gradient ascent on one centralized scalar return, then the resulting field would be integrable and this non-potential cyclic component would disappear. By contrast, MARL with a single shared reward can still exhibit the same obstruction whenever the implemented actor update is factorized into local or approximate blocks whose cross-agent responses are not jointly consistent. Sampling the others' actions contributes stochasticity, but by itself does not imply a structural non-potential cyclic component at the level of the exact expected field; the obstruction appears when the realized update rule departs from exact joint-gradient structure. Mathematically, however, the phenomenon is not exclusive to multi-agent systems; any non-potential update rule may contain such a component. Finally, Proposition~\ref{prop:circulation_residual} implies that any nonzero loop circulation yields nonzero residual energy, so the obstruction is not only geometric but also quantitatively detectable.
\end{remark}

The next proposition isolates the effect of this residual cyclic component by considering the canonical residual-only skew field from Example~\ref{ex:bilinear_app}.

\begin{proposition}[Updating along the residual cyclic component alone need not approach the fixed point]
\label{prop:residual_rotation_fp}
Consider the residual-only field on $\R^2$,
$
R(u,v)=(v,-u),
$
whose unique fixed point is $(0,0)$.
(i) Under the continuous-time dynamics $(\dot u,\dot v)=R(u,v)$, every non-stationary trajectory satisfies
$
u(t)^2+v(t)^2 = u(0)^2+v(0)^2
$
for all $t\ge 0$ and therefore never reaches the fixed point. In particular, every circle
$
\gamma_r(s)=r(\cos 2\pi s,\sin 2\pi s),
\qquad r>0,
$
is an invariant orbit.
(ii) Under simultaneous explicit-Euler updates
$
z_{t+1}=z_t+\eta R(z_t),
\qquad z_t=(u_t,v_t),
$
one has
$
\|z_{t+1}\|^2=(1+\eta^2)\|z_t\|^2
$
for every step size $\eta>0$. Hence every nonzero iterate moves strictly away from the fixed point, and
$
\|z_t\|=(1+\eta^2)^{t/2}\|z_0\|.
$
Thus a pure residual cyclic component can make the fixed point unreachable in continuous time and repelling in discrete time.
\end{proposition}

Proposition~\ref{prop:residual_rotation_fp} isolates the role of the residual cyclic component: in continuous time it sustains recurrent motion, while in simultaneous discrete-time updates any positive step size turns the same skew field into outward drift.

\section{Auxiliary Results for Projected Dynamics and Gap Bounds}
\label{app:theory_details}

This appendix records the projection inequality used by Theorem~\ref{thm:lyapunov},
the $L^2$ output-law gap bound, an inexact projected-step lemma used by Theorem~\ref{thm:inexact},
and several interpretive remarks omitted from the main text.

\begin{lemma}[Projection variational inequality and nonexpansiveness]
\label{lem:proj_vi}
Let $\cX\subset\mathbb{R}^d$ be nonempty, closed, and convex, and let $P:=\Proj_{\cX}$ be the Euclidean projection.
Then for any $y\in\mathbb{R}^d$ and any $z\in\cX$,
\[
\langle y-P(y),\,z-P(y)\rangle \le 0.
\]
Moreover, $P$ is firmly nonexpansive: for any $y,y'\in\mathbb{R}^d$,
\[
\|P(y)-P(y')\|^2 \le \langle P(y)-P(y'),\,y-y'\rangle,
\qquad\text{and hence}\qquad
\|P(y)-P(y')\|\le \|y-y'\|.
\]
\end{lemma}

\begin{remark}[Residual circulation and the additive gap term]
\label{rem:gap_interp_app}
By Proposition~\ref{prop:circulation_residual}, any closed-loop circulation of the original operator $F$
must be carried by the residual $R$.
Accordingly, the additive terms $D\varepsilon$ in Theorem~\ref{thm:gap} and $DE_T$ in Theorem~\ref{thm:gap_L2}
can be read as the unavoidable contribution of non-potential circulation in the original joint field.
\end{remark}

\begin{theorem}[$L^2$ non-potentiality with respect to the output-iterate law]
\label{thm:gap_L2}
Assume Assumption~\ref{ass:smooth} and suppose
\[
F(x)=-\nabla\Phi(x)+R(x),
\]
with $R$ square-integrable under the law of the output iterate defined below.
Let $\eta\le 1/L$ and run the exact projected ascent
\[
x_{t+1}=\Proj_{\cX}(x_t+\eta\nabla\Phi(x_t)),
\qquad t=0,\dots,T-1.
\]
Let $\tilde t$ be uniformly distributed over $\{0,1,\dots,T-1\}$, independent of the iterates, and output $\tilde x_T:=x_{\tilde t}$.
Define
\[
E_T:=\Bigl(\mathbb{E}\bigl[\|R(\tilde x_T)\|^2\bigr]\Bigr)^{1/2}.
\]
Then
\[
\mathbb{E}\bigl[\Gap(\tilde x_T)\bigr]
\le
(D+\eta G)\sqrt{\frac{2(\Phi_{\max}-\Phi_{\min})}{T\eta}}
+
DE_T.
\]
\end{theorem}

\begin{remark}[Rate extensions under stronger structure]
\label{rem:rate_ext_app}
The $O(1/\sqrt{T})$ rate in Theorem~\ref{thm:gap} is the most robust guarantee under the minimal assumptions used in the main text.
Under stronger structure, such as a PL condition for $\Phi$ or stronger operator regularity, one can sharpen the stationarity rate and therefore the induced gap rate.
These refinements are not needed for the main narrative.
\end{remark}

\begin{lemma}[Gap bound via inexact projected-step mapping]
\label{lem:gap_mapping_inexact}
Under Assumption~\ref{ass:smooth}, let
\[
x^+=\Proj_{\cX}(x+\eta g),
\qquad
\widetilde{\cG}_\eta(x):=\frac{1}{\eta}(x^+-x),
\]
and let $e:=g-\nabla\Phi(x)$.
Then
\[
\Gap_\Phi(x)
\le
\bigl(D+\eta\|g\|\bigr)\,\|\widetilde{\cG}_\eta(x)\|
+
D\,\|e\|.
\]
\end{lemma}

\begin{remark}[Interpretation of the inexact-projection bound]
\label{rem:inexact_interp_app}
In Theorem~\ref{thm:inexact}, the residual term $D\varepsilon$ is inherited from the original non-potential component,
while the additional $D\bar\delta$ term and the $\bar\delta^2$ floor arise from projection error under the stated pointwise accuracy assumption.
Thus the theorem separates two effects: structural non-potentiality in the original game field and approximation error in the projection routine itself.
\end{remark}

\section{Broader impacts.}
HPML is a general-purpose MARL optimization tool. Its potential positive impact is improved stability and reproducibility in cooperative and mixed-motive multi-agent systems. A potential negative impact is that more stable multi-agent optimization could also improve systems deployed in competitive, surveillance, or resource-allocation settings; responsible use therefore depends on the application context, evaluation protocol, and deployment safeguards.

\section{Proof}

\subsection{Proof of Lemma~\ref{lem:gap_zero_vi}}

\begin{proof}
Fix any $\bar x\in\cX$. By Definition~\ref{def:gap},
\[
\Gap(\bar x)=\max_{x\in\cX}\;\ip{F(\bar x)}{\bar x-x}.
\]
First observe that choosing the feasible point $x=\bar x$ yields
\[
\ip{F(\bar x)}{\bar x-\bar x}=\ip{F(\bar x)}{0}=0,
\]
hence the maximum is at least this value and therefore
\[
\Gap(\bar x)\;\ge\;0.
\]
Now assume $\Gap(\bar x)=0$. For an arbitrary $x\in\cX$, since $\Gap(\bar x)$ is the maximum of the map
$x\mapsto \ip{F(\bar x)}{\bar x-x}$ over the feasible set $\cX$, it dominates every feasible value, i.e.,
\[
\ip{F(\bar x)}{\bar x-x}\;\le\;\max_{y\in\cX}\ip{F(\bar x)}{\bar x-y}
\;=\;\Gap(\bar x)\;=\;0.
\]
Multiplying both sides by $-1$ reverses the inequality and gives
\[
-\ip{F(\bar x)}{\bar x-x}\;\ge\;0.
\]
Using bilinearity of the inner product and the identity $-(\bar x-x)=x-\bar x$, we rewrite the left-hand side as
\[
-\ip{F(\bar x)}{\bar x-x}=\ip{F(\bar x)}{-(\bar x-x)}=\ip{F(\bar x)}{x-\bar x}.
\]
Therefore, for every $x\in\cX$ we have
\[
\ip{F(\bar x)}{x-\bar x}\;\ge\;0,
\]
which is exactly the Stampacchia variational inequality condition in Definition~\ref{def:vi}. Hence $\bar x$
solves $\mathrm{VI}(\cX,F)$.
\end{proof}

\subsection{Proof of Proposition~\ref{prop:grad_symm}}

\begin{proof}
For (i), assume there exists $\phi$ that is twice continuously differentiable on an open set $V\supseteq\cX$
and satisfies $F=\nabla\phi$ on $V$. Writing components, for each $i\in\{1,\dots,d\}$ we have
$F_i(x)=\partial_i \phi(x)$ for all $x\in V$. By definition of the Jacobian $J(x)=\nabla F(x)$,
its entries are
\[
J_{ij}(x)=\partial_j F_i(x)=\partial_j\big(\partial_i \phi(x)\big)=\partial_{j}\partial_{i}\phi(x).
\]
Similarly,
\[
J_{ji}(x)=\partial_i F_j(x)=\partial_i\big(\partial_j \phi(x)\big)=\partial_{i}\partial_{j}\phi(x).
\]
Since $\phi\in C^2(V)$, Schwarz/Clairaut's theorem applies and yields $\partial_{j}\partial_{i}\phi(x)=
\partial_{i}\partial_{j}\phi(x)$ for all $x\in V$, hence $J_{ij}(x)=J_{ji}(x)$ for every pair $(i,j)$,
so $J(x)=J(x)^\top$ on $V$. With the symmetric--antisymmetric split
$A(x)=\tfrac12\big(J(x)-J(x)^\top\big)$ (Definition~2.4 in the paper's notation),
this gives $A(x)\equiv 0$ on $V$, and consequently on $\cX$ we have
\[
\mathcal R(x)=\|A(x)\|_F^2=0.
\]

For (ii), assume $U\subset\R^d$ is open and simply connected, $F\in C^1(U)$, and $A(x)\equiv 0$ on $U$.
Then for every $x\in U$,
\[
0=A(x)=\tfrac12\big(J(x)-J(x)^\top\big)\qquad\Longrightarrow\qquad J(x)=J(x)^\top.
\]
Writing this entrywise and recalling $J_{ij}(x)=\partial_j F_i(x)$, we obtain the cross-partial equalities
\[
\partial_j F_i(x)=\partial_i F_j(x),\qquad \forall\,x\in U,\ \forall\,i,j\in\{1,\dots,d\}.
\]

\begin{lemma}[Closed-loop integral vanishes under $J=J^\top$]
\label{lem:curlfree_loop}
Let $U\subset\R^d$ be open and simply connected and let $F\in C^1(U)$ satisfy $\partial_jF_i=\partial_iF_j$
on $U$. Then for every piecewise $C^1$ closed curve $\gamma:[0,1]\to U$ with $\gamma(0)=\gamma(1)$,
\[
\int_0^1 \ip{F(\gamma(t))}{\gamma'(t)}\,dt \;=\;0.
\]
\end{lemma}
\begin{proof}
Fix a piecewise $C^1$ closed curve $\gamma$ in $U$. Since $U$ is simply connected, $\gamma$ is homotopic
to the constant loop at $\gamma(0)$ through loops in $U$; concretely, there exists a (piecewise) $C^1$ map
$H:[0,1]\times[0,1]\to U$ such that $H(0,t)=\gamma(t)$ for all $t$, $H(1,t)=\gamma(0)$ for all $t$,
and $H(s,0)=H(s,1)=\gamma(0)$ for all $s$.
Define, for each $s\in[0,1]$, the loop integral
\[
I(s):=\int_0^1 \ip{F(H(s,t))}{\partial_t H(s,t)}\,dt.
\]
We compute $I'(s)$ by differentiating under the integral sign. Using the chain rule and that
$\nabla F = J$, we have
\[
\frac{d}{ds}\,\ip{F(H(s,t))}{\partial_t H(s,t)}
=
\ip{J(H(s,t))\,\partial_s H(s,t)}{\partial_t H(s,t)}
+
\ip{F(H(s,t))}{\partial_{st} H(s,t)}.
\]
Integrating this identity over $t\in[0,1]$ gives
\[
I'(s)=\int_0^1 \ip{J(H)\,\partial_s H}{\partial_t H}\,dt
+\int_0^1 \ip{F(H)}{\partial_{st} H}\,dt,
\]
where for readability $H=H(s,t)$ inside the integrals.
For the second term, integrate by parts in $t$. Since $F\in C^1$ and $H$ is (piecewise) $C^1$ in both
variables, the map $t\mapsto \ip{F(H(s,t))}{\partial_s H(s,t)}$ is (piecewise) $C^1$ and
\[
\partial_t \ip{F(H)}{\partial_s H}
=
\ip{J(H)\,\partial_t H}{\partial_s H}+\ip{F(H)}{\partial_{ts}H}.
\]
Rearranging yields
\[
\ip{F(H)}{\partial_{st}H}
=
\partial_t \ip{F(H)}{\partial_s H}-\ip{J(H)\,\partial_t H}{\partial_s H}.
\]
Integrating from $t=0$ to $t=1$ gives
\[
\int_0^1 \ip{F(H)}{\partial_{st}H}\,dt
=
\Big[\ip{F(H(s,t))}{\partial_s H(s,t)}\Big]_{t=0}^{t=1}
-\int_0^1 \ip{J(H)\,\partial_t H}{\partial_s H}\,dt.
\]
Because $H(s,0)=H(s,1)=\gamma(0)$ for all $s$, the boundary term vanishes:
both endpoints coincide, and $\partial_s H(s,0)=\partial_s H(s,1)=0$ for the fixed endpoint loop,
so $\big[\ip{F(H)}{\partial_s H}\big]_{0}^{1}=0$. Substituting back into the formula for $I'(s)$ yields
\[
I'(s)=\int_0^1 \ip{J(H)\,\partial_s H}{\partial_t H}\,dt
-\int_0^1 \ip{J(H)\,\partial_t H}{\partial_s H}\,dt.
\]
Using $\ip{a}{b}=\ip{b}{a}$ and the identity
$\ip{J\,u}{v}-\ip{J\,v}{u}=\ip{(J-J^\top)u}{v}$, we obtain
\[
I'(s)=\int_0^1 \ip{(J(H)-J(H)^\top)\,\partial_s H}{\partial_t H}\,dt.
\]
Under the hypothesis $\partial_jF_i=\partial_iF_j$, we have $J=J^\top$ everywhere on $U$, hence
$J(H)-J(H)^\top=0$ along the homotopy and therefore $I'(s)=0$ for all $s$.
Thus $I(s)$ is constant in $s$. Evaluating at $s=1$ gives $I(1)=0$ because $H(1,\cdot)$ is the constant loop,
so $I(0)=I(1)=0$. Since $I(0)$ is exactly the integral along $\gamma$, the claim follows.
\end{proof}

Fix a base point $x_0\in U$. For any $x\in U$, choose a piecewise $C^1$ curve $\gamma_{x_0\to x}:[0,1]\to U$
with $\gamma_{x_0\to x}(0)=x_0$ and $\gamma_{x_0\to x}(1)=x$, and define
\[
\phi(x):=\int_0^1 \ip{F(\gamma_{x_0\to x}(t))}{\gamma'_{x_0\to x}(t)}\,dt.
\]
We show $\phi$ is well-defined (independent of the chosen path). If $\gamma_1$ and $\gamma_2$ are two such paths
from $x_0$ to $x$, consider the closed curve obtained by traversing $\gamma_1$ and then traversing $\gamma_2$
backwards; denote this closed loop by $\gamma_1\circ\gamma_2^{-1}$. The integral along $\gamma_2^{-1}$ is the
negative of the integral along $\gamma_2$, so
\[
\int_{\gamma_1\circ\gamma_2^{-1}} \ip{F}{d\ell}
=
\int_{\gamma_1} \ip{F}{d\ell} - \int_{\gamma_2}\ip{F}{d\ell}.
\]
By Lemma~\ref{lem:curlfree_loop}, the left-hand side is $0$, hence
$\int_{\gamma_1}\ip{F}{d\ell}=\int_{\gamma_2}\ip{F}{d\ell}$, proving that $\phi(x)$ does not depend on the choice
of $\gamma_{x_0\to x}$.

It remains to show $\nabla\phi(x)=F(x)$. Fix $x\in U$ and a coordinate direction $e_k$.
Because $U$ is open, there exists $\delta>0$ such that $x+h e_k\in U$ for all $|h|<\delta$.
Using path-independence, we may compute $\phi(x+h e_k)-\phi(x)$ by choosing for $\phi(x+h e_k)$ the path
$\gamma_{x_0\to x}$ followed by the straight segment from $x$ to $x+h e_k$. Parameterizing that segment by
$\sigma_h(t)=x+t e_k$ for $t\in[0,h]$, we obtain
\[
\phi(x+h e_k)-\phi(x)=\int_0^{h} \ip{F(\sigma_h(t))}{\sigma_h'(t)}\,dt
=\int_0^{h} \ip{F(x+t e_k)}{e_k}\,dt
=\int_0^{h} F_k(x+t e_k)\,dt.
\]
Dividing by $h\neq 0$ gives
\[
\frac{\phi(x+h e_k)-\phi(x)}{h}
=\frac{1}{h}\int_0^{h} F_k(x+t e_k)\,dt.
\]
Since $F$ is continuous (indeed $C^1$) on $U$, the integrand $F_k(x+t e_k)$ converges to $F_k(x)$ as $t\to 0$,
and therefore the right-hand side converges to $F_k(x)$ as $h\to 0$ (it is the average value of a continuous
function over an interval shrinking to a point). Hence the partial derivative exists and satisfies
\[
\partial_k \phi(x)=\lim_{h\to 0}\frac{\phi(x+h e_k)-\phi(x)}{h}=F_k(x).
\]
Because this holds for every $k$, we conclude $\nabla\phi(x)=F(x)$ for all $x\in U$.

For uniqueness up to an additive constant, suppose $\tilde\phi$ is another scalar potential on $U$ with
$\nabla\tilde\phi=F$. Then $\nabla(\phi-\tilde\phi)=0$ on $U$. For any $x,y\in U$, choose a piecewise $C^1$ path
$\gamma$ from $x$ to $y$; applying the fundamental theorem of calculus along curves yields
\[
(\phi-\tilde\phi)(y)-(\phi-\tilde\phi)(x)
=\int_0^1 \ip{\nabla(\phi-\tilde\phi)(\gamma(t))}{\gamma'(t)}\,dt
=\int_0^1 \ip{0}{\gamma'(t)}\,dt=0,
\]
so $\phi-\tilde\phi$ is constant on the (path-)connected set $U$. This proves uniqueness up to an additive constant.
\end{proof}

\subsection{Proof of Lemma~\ref{lem:adjoint}}
\begin{proof}
We write $d\mu(x)=\rho(x)\,dx$ and interpret all expressions in the natural sense in which they are assumed
to be well-defined; in particular, $\psi\in H^1(\cX)$ has $\nabla\psi\in L^2(\cX)$, the left-hand side
$\int_\cX U^\top\nabla\psi\,d\mu$ is understood as $\int_\cX \rho\,U\cdot\nabla\psi\,dx$, and the right-hand side
$\int_\cX \psi\,\div_\mu U\,d\mu$ is understood as the $H^{-1}$--$H^1$ dual pairing whenever $\div_\mu U\notin L^2(\mu)$.

Recall the weighted divergence is given (in the distributional sense) by
\[
\div_\mu U \;:=\;\frac{1}{\rho}\,\div(\rho U),
\qquad\text{equivalently}\qquad
\div(\rho U)=\rho\,\div_\mu U.
\]
We first justify the identity for smooth $\psi$; the extension to $\psi\in H^1(\cX)$ will then follow by density.
Fix $\psi\in C^\infty(\cX)$ (periodic if $\cX$ is a torus, or smooth up to the boundary if $\cX$ is a Lipschitz domain),
and set
\[
W(x):=\rho(x)\,U(x).
\]
Then the left-hand side becomes
\[
\int_{\cX} U^\top \nabla\psi\,d\mu
=
\int_{\cX} \rho\,U\cdot\nabla\psi\,dx
=
\int_{\cX} W\cdot\nabla\psi\,dx.
\]
We use the pointwise product rule for the divergence of a scalar times a vector field:
for every $x$ where the derivatives make sense,
\[
\div(\psi W)(x)
=
\sum_{i=1}^d \partial_i\big(\psi(x)\,W_i(x)\big)
=
\sum_{i=1}^d \big(\partial_i\psi(x)\big)\,W_i(x)\;+\;\sum_{i=1}^d \psi(x)\,\partial_i W_i(x)
=
W(x)\cdot\nabla\psi(x)\;+\;\psi(x)\,\div W(x).
\]
Rearranging this identity gives, pointwise,
\[
W\cdot\nabla\psi
=
\div(\psi W)-\psi\,\div W.
\]
Integrating over $\cX$ with respect to $dx$ yields
\[
\int_{\cX} W\cdot\nabla\psi\,dx
=
\int_{\cX} \div(\psi W)\,dx
-
\int_{\cX} \psi\,\div W\,dx.
\]
If $\cX$ is a bounded Lipschitz domain, the divergence theorem applies to the vector field $\psi W$ and gives
\[
\int_{\cX} \div(\psi W)\,dx
=
\int_{\partial\cX} \psi\,W\cdot n\,dS
=
\int_{\partial\cX} \psi\,\rho\,U\cdot n\,dS.
\]
Under the stated boundary condition eliminating boundary terms (e.g.\ $\rho\,U\cdot n=0$ on $\partial\cX$),
the boundary integral is $0$, hence
\[
\int_{\cX} W\cdot\nabla\psi\,dx
=
-\int_{\cX} \psi\,\div W\,dx.
\]
If $\cX$ is a flat torus, it has no boundary; equivalently, if one represents it by a fundamental periodic cell,
the boundary contributions cancel pairwise due to periodicity, and one again has
$\int_{\cX}\div(\psi W)\,dx=0$, so the same identity holds:
\[
\int_{\cX} W\cdot\nabla\psi\,dx
=
-\int_{\cX} \psi\,\div W\,dx.
\]
Substituting back $W=\rho U$ gives
\[
\int_{\cX} \rho\,U\cdot\nabla\psi\,dx
=
-\int_{\cX} \psi\,\div(\rho U)\,dx.
\]
Using $\div(\rho U)=\rho\,\div_\mu U$, we rewrite the right-hand side as
\[
-\int_{\cX} \psi\,\div(\rho U)\,dx
=
-\int_{\cX} \psi\,\rho\,\div_\mu U\,dx
=
-\int_{\cX} \psi\,\div_\mu U\,d\mu,
\]
and therefore, for smooth $\psi$,
\[
\int_{\cX} U^\top\nabla\psi\,d\mu
=
-\int_{\cX} \psi\,\div_\mu U\,d\mu.
\]

To pass from smooth $\psi$ to $\psi\in H^1(\cX)$, take a sequence $\{\psi_k\}$ of smooth functions (periodic on the torus,
or smooth up to the boundary on a Lipschitz domain) such that $\psi_k\to\psi$ in $H^1(\cX)$.
For each $k$ the identity already proved yields
\[
\int_{\cX} \rho\,U\cdot\nabla\psi_k\,dx
=
-\int_{\cX} \psi_k\,\rho\,\div_\mu U\,dx.
\]
Because $\rho$ is essentially bounded above and below, $\nabla\psi_k\to\nabla\psi$ in $L^2(\cX)$ implies
$\nabla\psi_k\to\nabla\psi$ in $L^2(\mu)$, and the left-hand side converges by Cauchy--Schwarz:
\[
\Big|\int_{\cX} U\cdot(\nabla\psi_k-\nabla\psi)\,d\mu\Big|
\le
\|U\|_{L^2(\mu)}\,\|\nabla\psi_k-\nabla\psi\|_{L^2(\mu)}
\longrightarrow 0.
\]
On the right-hand side, the assumption that $\div_\mu U$ is well-defined in $H^{-1}(\cX)$ means the map
$\psi\mapsto \langle \div_\mu U,\psi\rangle_{H^{-1},H^1}$ is continuous with respect to the $H^1$ norm, so
\[
\langle \div_\mu U,\psi_k\rangle_{H^{-1},H^1}\longrightarrow
\langle \div_\mu U,\psi\rangle_{H^{-1},H^1}.
\]
Interpreting $-\int_{\cX}\psi_k\,\div_\mu U\,d\mu$ as this dual pairing when needed, we may pass to the limit and obtain
\[
\int_{\cX} U^\top \nabla\psi\,d\mu
=
-\int_{\cX} \psi\,\div_\mu U\,d\mu
\qquad\text{for all }\psi\in H^1(\cX).
\]
Reading the identity as an $L^2(\mu)$ (or $H^{-1}$--$H^1$) pairing shows precisely that $\nabla$ and $-\div_\mu$
are adjoint under the $L^2(\mu)$ pairing.
\end{proof}

\subsection{Proof of Proposition~\ref{prop:poisson}}

\begin{proof}
Write $d\mu(x)=\rho(x)\,dx$ and recall $M\succ0$ is constant. Using $\nabla_M\Phi=M^{-1}\nabla\Phi$ and
$\|v\|_M^2=v^\top M v$, the objective in Definition~\ref{def:proj} can be written as the quadratic functional
\[
J(\Phi)
:=\int_{\cX}\|F(x)-\nabla_M\Phi(x)\|_M^2\,d\mu(x)
=\int_{\cX}\big(F(x)-M^{-1}\nabla\Phi(x)\big)^\top M\big(F(x)-M^{-1}\nabla\Phi(x)\big)\,d\mu(x).
\]
Since $M$ is constant, expanding the integrand gives
\[
\big(F-M^{-1}\nabla\Phi\big)^\top M\big(F-M^{-1}\nabla\Phi\big)
=F^\top M F-2\,F^\top \nabla\Phi+\nabla\Phi^\top M^{-1}\nabla\Phi,
\]
so $J(\Phi)$ depends on $\Phi$ only through $\nabla\Phi$ and is convex (indeed quadratic) in $\nabla\Phi$.

\begin{lemma}[First variation of $J$]
\label{lem:first_var_J}
For any $\Phi\in H^1(\cX)$ and any direction $\psi\in H^1(\cX)$, the map $\varepsilon\mapsto J(\Phi+\varepsilon\psi)$
is differentiable and
\[
\frac{d}{d\varepsilon}J(\Phi+\varepsilon\psi)\Big|_{\varepsilon=0}
=
2\int_{\cX}\big(M^{-1}\nabla\Phi(x)-F(x)\big)^\top\nabla\psi(x)\,d\mu(x).
\]
\end{lemma}
\begin{proof}
Set $\Phi_\varepsilon:=\Phi+\varepsilon\psi$, so $\nabla\Phi_\varepsilon=\nabla\Phi+\varepsilon\nabla\psi$ and
\[
F-M^{-1}\nabla\Phi_\varepsilon
=
\big(F-M^{-1}\nabla\Phi\big)-\varepsilon\,M^{-1}\nabla\psi.
\]
Substituting into $J$ and using bilinearity yields
\begin{align*}
J(\Phi_\varepsilon)
&=\int_{\cX}\Big(\big(F-M^{-1}\nabla\Phi\big)-\varepsilon\,M^{-1}\nabla\psi\Big)^\top
M\Big(\big(F-M^{-1}\nabla\Phi\big)-\varepsilon\,M^{-1}\nabla\psi\Big)\,d\mu \\
&=\int_{\cX}\big(F-M^{-1}\nabla\Phi\big)^\top M\big(F-M^{-1}\nabla\Phi\big)\,d\mu
-2\varepsilon\int_{\cX}\big(F-M^{-1}\nabla\Phi\big)^\top M\big(M^{-1}\nabla\psi\big)\,d\mu \\
&\qquad\qquad+\varepsilon^2\int_{\cX}\big(M^{-1}\nabla\psi\big)^\top M\big(M^{-1}\nabla\psi\big)\,d\mu.
\end{align*}
Because $M(M^{-1}\nabla\psi)=\nabla\psi$ and $\big(M^{-1}\nabla\psi\big)^\top M\big(M^{-1}\nabla\psi\big)
=\nabla\psi^\top M^{-1}\nabla\psi$, differentiating in $\varepsilon$ gives
\[
\frac{d}{d\varepsilon}J(\Phi_\varepsilon)
=
-2\int_{\cX}\big(F-M^{-1}\nabla\Phi\big)^\top\nabla\psi\,d\mu
+2\varepsilon\int_{\cX}\nabla\psi^\top M^{-1}\nabla\psi\,d\mu,
\]
and evaluating at $\varepsilon=0$ yields
\[
\frac{d}{d\varepsilon}J(\Phi+\varepsilon\psi)\Big|_{\varepsilon=0}
=
-2\int_{\cX}\big(F-M^{-1}\nabla\Phi\big)^\top\nabla\psi\,d\mu
=
2\int_{\cX}\big(M^{-1}\nabla\Phi-F\big)^\top\nabla\psi\,d\mu,
\]
as claimed.
\end{proof}

Let $\Phi^\star$ be any minimizer. Since $J$ is (Gateaux) differentiable and $\Phi^\star$ is an unconstrained minimizer
over the affine space $H^1(\cX)$, the first-order optimality condition is
\[
\frac{d}{d\varepsilon}J(\Phi^\star+\varepsilon\psi)\Big|_{\varepsilon=0}=0
\qquad\forall\,\psi\in H^1(\cX).
\]
Applying Lemma~\ref{lem:first_var_J} at $\Phi^\star$ yields the weak orthogonality condition
\[
\int_{\cX}\big(M^{-1}\nabla\Phi^\star-F\big)^\top\nabla\psi\,d\mu=0
\qquad\forall\,\psi\in H^1(\cX).
\]
Equivalently,
\[
\int_{\cX}\big(M^{-1}\nabla\Phi^\star\big)^\top\nabla\psi\,d\mu
=
\int_{\cX}F^\top\nabla\psi\,d\mu
\qquad\forall\,\psi\in H^1(\cX).
\]
Now apply Lemma~\ref{lem:adjoint} twice (with $U=M^{-1}\nabla\Phi^\star$ and $U=F$): under the stated boundary
assumptions removing boundary terms, for every $\psi\in H^1(\cX)$ we have
\[
\int_{\cX}\big(M^{-1}\nabla\Phi^\star\big)^\top\nabla\psi\,d\mu
=
-\int_{\cX}\psi\,\div_\mu\!\big(M^{-1}\nabla\Phi^\star\big)\,d\mu,
\qquad
\int_{\cX}F^\top\nabla\psi\,d\mu
=
-\int_{\cX}\psi\,\div_\mu F\,d\mu.
\]
Substituting these identities into the weak orthogonality equality gives
\[
-\int_{\cX}\psi\,\div_\mu\!\big(M^{-1}\nabla\Phi^\star\big)\,d\mu
=
-\int_{\cX}\psi\,\div_\mu F\,d\mu
\qquad\forall\,\psi\in H^1(\cX),
\]
hence
\[
\int_{\cX}\psi\Big(\div_\mu\!\big(M^{-1}\nabla\Phi^\star\big)-\div_\mu F\Big)\,d\mu=0
\qquad\forall\,\psi\in H^1(\cX).
\]
This is precisely the statement that
\[
\div_\mu\!\big(M^{-1}\nabla\Phi^\star\big)=\div_\mu F
\quad\text{in }H^{-1}(\cX),
\]
i.e.\ $\cL_{M,\mu}\Phi^\star=\div_\mu F$ in the weak sense, and the projected field is
$\Pi_{\mathrm{grad}}^M F=\nabla_M\Phi^\star=M^{-1}\nabla\Phi^\star$ by Definition~\ref{def:proj}.

For uniqueness up to an additive constant, let $\Phi_1^\star$ and $\Phi_2^\star$ be two minimizers and set
$w:=\Phi_1^\star-\Phi_2^\star\in H^1(\cX)$. Subtracting the two weak equations yields
\[
\div_\mu\!\big(M^{-1}\nabla w\big)=0\quad\text{in }H^{-1}(\cX).
\]
Testing this equality with $\psi=w\in H^1(\cX)$ and applying Lemma~\ref{lem:adjoint} gives
\[
0=\int_{\cX} w\,\div_\mu\!\big(M^{-1}\nabla w\big)\,d\mu
=-\int_{\cX}\big(M^{-1}\nabla w\big)^\top\nabla w\,d\mu
=-\int_{\cX}\nabla w^\top M^{-1}\nabla w\,d\mu.
\]
Therefore
\[
\int_{\cX}\nabla w^\top M^{-1}\nabla w\,d\mu=0.
\]
Since $M^{-1}\succ0$, $\nabla w^\top M^{-1}\nabla w\ge \lambda_{\min}(M^{-1})\|\nabla w\|_2^2$ pointwise, hence
the integral can vanish only if $\nabla w=0$ $\mu$-a.e. on $\cX$. Consequently,
\[
M^{-1}\nabla\Phi_1^\star=M^{-1}\nabla\Phi_2^\star\quad\mu\text{-a.e.},
\]
so the projected field $\nabla_M\Phi^\star$ is unique $\mu$-a.e., while $\Phi^\star$ is determined only up to an
additive constant (more precisely, up to constants on each connected component). One may fix a gauge such as
$\E_\mu[\Phi^\star]=0$ to select a representative.

Finally, the stated boundary condition arises as the natural boundary condition of the variational problem when one
does not impose boundary constraints on $\Phi$. Writing $U:=M^{-1}\nabla\Phi^\star-F$ and assuming sufficient trace
regularity so that classical integration by parts applies, one has
\[
\int_{\cX} U^\top\nabla\psi\,d\mu
=
-\int_{\cX}\psi\,\div_\mu U\,d\mu
+\int_{\partial\cX}\psi\,\rho\,U\cdot n\,dS.
\]
At optimality the first variation gives $\int_{\cX}U^\top\nabla\psi\,d\mu=0$ for all $\psi$, while the interior Euler--Lagrange
equation gives $\div_\mu U=0$ in $\cX$; therefore the boundary term must vanish for arbitrary boundary traces of $\psi$,
forcing
\[
\rho\,U\cdot n=\rho\,(M^{-1}\nabla\Phi^\star-F)\cdot n=0\quad\text{on }\partial\cX,
\]
which is exactly the zero-flux residual condition.
\end{proof}

\subsection{Proof of Lemma~\ref{lem:orth}}
\begin{proof}
By Definition~\ref{def:proj}, pick a minimizer $\Phi^\star\in H^1(\cX)$ and set
\[
G^\star:=\Pi_{\mathrm{grad}}^{M}F=\nabla_M\Phi^\star,
\qquad
R:=F-G^\star=F-\nabla_M\Phi^\star.
\]
Consider the least-squares functional on potentials
\[
J(\Psi):=\E_{x\sim\mu}\big[\|F(x)-\nabla_M\Psi(x)\|_M^2\big],
\qquad \Psi\in H^1(\cX).
\]
Fix an arbitrary test potential $\Phi\in H^1(\cX)$ and form the one-dimensional perturbation
\[
\Psi_\varepsilon := \Phi^\star+\varepsilon\Phi,\qquad \varepsilon\in\R.
\]
Using linearity of $\nabla_M$, we have
\[
\nabla_M\Psi_\varepsilon
=
\nabla_M(\Phi^\star+\varepsilon\Phi)
=
\nabla_M\Phi^\star+\varepsilon\nabla_M\Phi
=
G^\star+\varepsilon\nabla_M\Phi.
\]
Therefore the residual under this perturbation is
\[
F-\nabla_M\Psi_\varepsilon
=
F-(G^\star+\varepsilon\nabla_M\Phi)
=
(F-G^\star)-\varepsilon\nabla_M\Phi
=
R-\varepsilon\nabla_M\Phi.
\]
Now expand the squared $M$-norm using $\|v\|_M^2=\langle v,v\rangle_M$ and bilinearity of the $M$-inner product:
\begin{align*}
\|R-\varepsilon\nabla_M\Phi\|_M^2
&=\langle R-\varepsilon\nabla_M\Phi,\;R-\varepsilon\nabla_M\Phi\rangle_M \\
&=\langle R,R\rangle_M
-\varepsilon\langle R,\nabla_M\Phi\rangle_M
-\varepsilon\langle \nabla_M\Phi,R\rangle_M
+\varepsilon^2\langle \nabla_M\Phi,\nabla_M\Phi\rangle_M.
\end{align*}
Because $\langle\cdot,\cdot\rangle_M$ is symmetric (since $M\succ0$), we have
$\langle R,\nabla_M\Phi\rangle_M=\langle \nabla_M\Phi,R\rangle_M$, hence
\[
\|R-\varepsilon\nabla_M\Phi\|_M^2
=
\|R\|_M^2
-2\varepsilon\,\langle R,\nabla_M\Phi\rangle_M
+\varepsilon^2\,\|\nabla_M\Phi\|_M^2.
\]
Taking expectation with respect to $\mu$ gives
\[
J(\Psi_\varepsilon)
=
\E_\mu[\|R\|_M^2]
-2\varepsilon\,\E_\mu[\langle R,\nabla_M\Phi\rangle_M]
+\varepsilon^2\,\E_\mu[\|\nabla_M\Phi\|_M^2].
\]
Differentiate this scalar quadratic function of $\varepsilon$:
\[
\frac{d}{d\varepsilon}J(\Psi_\varepsilon)
=
-2\,\E_\mu[\langle R,\nabla_M\Phi\rangle_M]
+2\varepsilon\,\E_\mu[\|\nabla_M\Phi\|_M^2].
\]
Since $\Phi^\star$ is a minimizer of $J$ over the whole (affine) space $H^1(\cX)$, the first-order optimality
condition for every direction $\Phi\in H^1(\cX)$ is
\[
\frac{d}{d\varepsilon}J(\Psi_\varepsilon)\Big|_{\varepsilon=0}=0.
\]
Evaluating the derivative at $\varepsilon=0$ yields
\[
-2\,\E_\mu[\langle R,\nabla_M\Phi\rangle_M]=0,
\]
and therefore
\[
\E_{x\sim\mu}\big[\langle R(x),\nabla_M\Phi(x)\rangle_M\big]=0
\qquad\forall\,\Phi\in H^1(\cX).
\]
This is exactly the stated $L^2(M,\mu)$-orthogonality of the residual $R$ to every metric-gradient field
$\nabla_M\Phi$.
\end{proof}
% \begin{proof}
% This is the normal-equation condition for an $L^2(M,\mu)$ orthogonal projection:
% the residual is orthogonal to the subspace $\cG_{\mathrm{grad}}^M$.
% \end{proof}

\subsection{Proof of Corollary~\ref{cor:pythag}}
\begin{proof}
Recall the decomposition $F=G^\star+R$, where $G^\star=\Pi_{\mathrm{grad}}^{M}F=\nabla_M\Phi^\star$ for some minimizer
$\Phi^\star\in H^1(\cX)$ and $R:=F-G^\star$. For each $x\in\cX$, expand the squared $M$-norm using
$\|v\|_M^2=\langle v,v\rangle_M$ and bilinearity of $\langle\cdot,\cdot\rangle_M$:
\begin{align*}
\|F(x)\|_M^2
&=\langle F(x),F(x)\rangle_M \\
&=\langle G^\star(x)+R(x),\,G^\star(x)+R(x)\rangle_M \\
&=\langle G^\star(x),G^\star(x)\rangle_M
+\langle G^\star(x),R(x)\rangle_M
+\langle R(x),G^\star(x)\rangle_M
+\langle R(x),R(x)\rangle_M.
\end{align*}
Because $\langle\cdot,\cdot\rangle_M$ is symmetric, $\langle G^\star,R\rangle_M=\langle R,G^\star\rangle_M$, so
\[
\|F(x)\|_M^2
=
\|G^\star(x)\|_M^2
+2\,\langle R(x),G^\star(x)\rangle_M
+\|R(x)\|_M^2.
\]
Taking expectation with respect to $x\sim\mu$ yields
\[
\E_\mu[\|F(x)\|_M^2]
=
\E_\mu[\|G^\star(x)\|_M^2]
+2\,\E_\mu[\langle R(x),G^\star(x)\rangle_M]
+\E_\mu[\|R(x)\|_M^2].
\]
It remains to show the cross term vanishes. By Lemma~\ref{lem:orth}, for every test potential $\Phi\in H^1(\cX)$,
\[
\E_\mu[\langle R(x),\nabla_M\Phi(x)\rangle_M]=0.
\]
Choose $\Phi=\Phi^\star$ (a minimizer inducing $G^\star=\nabla_M\Phi^\star$). Then
\[
\E_\mu[\langle R(x),G^\star(x)\rangle_M]
=
\E_\mu[\langle R(x),\nabla_M\Phi^\star(x)\rangle_M]
=
0.
\]
Substituting this into the previous expansion gives
\[
\E_{x\sim\mu}\big[\|F(x)\|_M^2\big]
=
\E_{x\sim\mu}\big[\|G^\star(x)\|_M^2\big]
+
\E_{x\sim\mu}\big[\|R(x)\|_M^2\big],
\]
which is the claimed Pythagorean identity.
\end{proof}

\subsection{Proof of Proposition~\ref{prop:circulation_residual}}
\begin{proof}
Because $G^\star=\nabla_M\Phi^\star$, by definition of the metric-gradient we have
\[
\nabla_M\Phi^\star(x)=M^{-1}\nabla\Phi^\star(x),
\qquad\text{so in particular}\qquad
M\,G^\star(x)=\nabla\Phi^\star(x)\ \ \text{for all }x\in\cX.
\]
By the definition of circulation and the $M$-inner product $\langle u,v\rangle_M=u^\top M v$, for any piecewise $C^1$ closed curve $\gamma:[0,1]\to\cX$ we can write
\[
\Circ_{G^\star}(\gamma)
=
\int_0^1 \langle G^\star(\gamma(t)),\dot\gamma(t)\rangle_M\,dt
=
\int_0^1 G^\star(\gamma(t))^\top M\,\dot\gamma(t)\,dt.
\]
Substituting $G^\star(\gamma(t))=M^{-1}\nabla\Phi^\star(\gamma(t))$ into the integrand yields
\[
G^\star(\gamma(t))^\top M\,\dot\gamma(t)
=
\big(M^{-1}\nabla\Phi^\star(\gamma(t))\big)^\top M\,\dot\gamma(t)
=
\nabla\Phi^\star(\gamma(t))^\top \dot\gamma(t),
\]
where we used that $(M^{-1})^\top=M^{-1}$ and $M^{-1}M=I$ because $M$ is symmetric positive definite. Hence
\[
\Circ_{G^\star}(\gamma)
=
\int_0^1 \nabla\Phi^\star(\gamma(t))^\top \dot\gamma(t)\,dt.
\]
Since $\gamma$ is piecewise $C^1$ and $\Phi^\star$ is differentiable, the chain rule applies on each smooth sub-interval of $[0,1]$ and gives
\[
\frac{d}{dt}\Phi^\star(\gamma(t))
=
\nabla\Phi^\star(\gamma(t))^\top \dot\gamma(t)
\quad\text{for all }t\text{ where }\gamma\text{ is differentiable}.
\]
Therefore, integrating over $[0,1]$ (equivalently, summing over the finitely many smooth pieces) and applying the fundamental theorem of calculus yields
\[
\Circ_{G^\star}(\gamma)
=
\int_0^1 \frac{d}{dt}\Phi^\star(\gamma(t))\,dt
=
\Phi^\star(\gamma(1))-\Phi^\star(\gamma(0)).
\]
Because $\gamma$ is closed, $\gamma(1)=\gamma(0)$, so $\Phi^\star(\gamma(1))-\Phi^\star(\gamma(0))=0$ and thus $\Circ_{G^\star}(\gamma)=0$. Now using $F=G^\star+R$ and the linearity of both the inner product and the integral, we have
\[
\Circ_F(\gamma)
=
\int_0^1 \langle F(\gamma(t)),\dot\gamma(t)\rangle_M\,dt
=
\int_0^1 \langle G^\star(\gamma(t))+R(\gamma(t)),\dot\gamma(t)\rangle_M\,dt
=
\Circ_{G^\star}(\gamma)+\Circ_R(\gamma)
=
\Circ_R(\gamma),
\]
which proves $\Circ_F(\gamma)=\Circ_R(\gamma)$.

For the quantitative bound, start from the identity just shown and write
\[
|\Circ_F(\gamma)|
=
|\Circ_R(\gamma)|
=
\left|\int_0^1 \langle R(\gamma(t)),\dot\gamma(t)\rangle_M\,dt\right|
=
\left|\int_0^1 R(\gamma(t))^\top M\,\dot\gamma(t)\,dt\right|.
\]
Let $M^{1/2}$ denote the (symmetric) principal square root of $M$, so that $M^{1/2}M^{1/2}=M$ and $(M^{1/2})^\top=M^{1/2}$. Then for each $t$ where $\gamma$ is differentiable,
\[
R(\gamma(t))^\top M\,\dot\gamma(t)
=
R(\gamma(t))^\top M^{1/2}M^{1/2}\dot\gamma(t)
=
\big(M^{1/2}R(\gamma(t))\big)^\top \big(M^{1/2}\dot\gamma(t)\big).
\]
Applying the Euclidean Cauchy--Schwarz inequality to the last inner product gives the pointwise bound
\[
\big|R(\gamma(t))^\top M\,\dot\gamma(t)\big|
\le
\|M^{1/2}R(\gamma(t))\|_2\,\|M^{1/2}\dot\gamma(t)\|_2
=
\|R(\gamma(t))\|_M\,\|\dot\gamma(t)\|_M,
\]
where we used $\|x\|_M^2=x^\top M x=\|M^{1/2}x\|_2^2$. Using this bound inside the integral and then applying Cauchy--Schwarz in $L^2([0,1])$ yields
\[
|\Circ_F(\gamma)|
\le
\int_0^1 \|R(\gamma(t))\|_M\,\|\dot\gamma(t)\|_M\,dt
\le
\left(\int_0^1 \|R(\gamma(t))\|_M^2\,dt\right)^{1/2}
\left(\int_0^1 \|\dot\gamma(t)\|_M^2\,dt\right)^{1/2}.
\]
Finally, because $t\sim\mathrm{Unif}[0,1]$ satisfies $\E[g(t)]=\int_0^1 g(t)\,dt$ for any integrable $g$, we can rewrite the last display as
\[
|\Circ_F(\gamma)|
\le
\Big(\E_{t\sim \mathrm{Unif}[0,1]} \|R(\gamma(t))\|_M^2\Big)^{1/2}
\cdot
\Big(\E_{t\sim \mathrm{Unif}[0,1]} \|\dot\gamma(t)\|_M^2\Big)^{1/2}.
\]
In the Euclidean-metric case $M=I$ (so $M^{-1}=I$), the second factor satisfies $\|\dot\gamma(t)\|_M=\|\dot\gamma(t)\|_{M^{-1}}$, which matches the stated form.
\end{proof}

\subsection{Proof of Proposition~\ref{prop:residual_bellman_obstruction}}
\begin{proof}
Recall the HPML decomposition
\[
F \;=\; G^\star + R,
\qquad
G^\star := \Pi_{\mathrm{grad}}^{M}F,
\qquad
R := F - G^\star .
\]
By definition of the projection, there exists a scalar potential $\Phi^\star$ such that
\[
G^\star = \nabla_M \Phi^\star = M^{-1}\nabla \Phi^\star .
\]

Assume $R\equiv 0$ on $U$.
Then, by the decomposition above,
\[
F = G^\star = \Pi_{\mathrm{grad}}^{M}F = \nabla_M \Phi^\star
\qquad\text{on }U.
\]
Hence the joint field is integrable on $U$: all agents' first-order update directions are aligned with ascent
of the single scalar potential $\Phi^\star$.

In particular, when HPML is run with exact projected ascent on the feasible set $\cX$,
\[
x_{t+1}=\Proj_{\cX}\!\bigl(x_t+\eta \nabla \Phi^\star(x_t)\bigr),
\]
the Lyapunov-improvement result already established in Theorem~\ref{thm:lyapunov}
(or the corresponding projected-gradient theorem in the main text) applies directly.
Therefore $\Phi^\star$ increases monotonically along non-stationary iterates, and the projected dynamics
cannot sustain a nontrivial cycle.

Assume there exists a piecewise-$C^1$ closed curve $\gamma:[0,1]\to U$ with $\gamma(0)=\gamma(1)$ such that
\[
\oint_\gamma \langle M F(x), dx\rangle
\;=\;
\Circ_F(\gamma)
\;\neq\;0.
\]
We first show that no scalar potential $\Phi$ on $U$ can satisfy $F=\nabla_M\Phi$ on $U$.

Suppose, for contradiction, that such a $\Phi$ exists.
Then along the loop $\gamma$,
\[
\Circ_F(\gamma)
=
\int_0^1 \langle \nabla_M\Phi(\gamma(t)), \dot\gamma(t)\rangle_M\,dt
=
\int_0^1 \langle M\nabla_M\Phi(\gamma(t)), \dot\gamma(t)\rangle\,dt.
\]
Since $M\nabla_M\Phi=\nabla\Phi$, this becomes
\[
\Circ_F(\gamma)
=
\int_0^1 \langle \nabla\Phi(\gamma(t)), \dot\gamma(t)\rangle\,dt
=
\int_0^1 \frac{d}{dt}\Phi(\gamma(t))\,dt
=
\Phi(\gamma(1))-\Phi(\gamma(0))
=
0,
\]
because $\gamma$ is closed.
This contradicts the assumption $\Circ_F(\gamma)\neq 0$.
Hence no scalar potential on $U$ can generate $F$ as a metric-gradient field.

It remains to identify which component carries this circulation.
By Proposition~\ref{prop:circulation_residual}, for every closed curve,
\[
\Circ_{G^\star}(\gamma)=0,
\qquad\text{and hence}\qquad
\Circ_F(\gamma)=\Circ_R(\gamma).
\]
Therefore,
\[
\oint_\gamma \langle M F(x),dx\rangle
=
\oint_\gamma \langle M R(x),dx\rangle
\neq 0.
\]
So any nonzero closed-loop circulation of the original joint field is carried entirely by the residual $R$.

Consequently, $R$ is the obstruction to joint greedy integrability:
when such a loop exists, the simultaneous critic-driven actor updates encoded by $F$ cannot be interpreted as
ascent of a single Bellman-compatible scalar objective on the joint policy space.
Equivalently, the original joint improvement field retains an irreducible circulation component, which rules out a
global potential representation and permits recurrent convention/role-allocation loops.
\end{proof}

\subsection{Proof of Proposition~\ref{prop:residual_rotation_fp}}
\begin{proof}
The equation $R(u,v)=(0,0)$ implies $v=0$ and $u=0$, so $(0,0)$ is the unique fixed point.

For the continuous-time dynamics $(\dot u,\dot v)=(v,-u)$,
\[
\frac{d}{dt}\bigl(u(t)^2+v(t)^2\bigr)
=
2u(t)\dot u(t)+2v(t)\dot v(t)
=
2u(t)v(t)-2v(t)u(t)
=
0.
\]
Hence $u(t)^2+v(t)^2$ is constant along trajectories. Therefore every non-stationary solution stays on the circle determined by its initial radius and never reaches the origin. In particular, each circle
\[
\gamma_r(s)=r(\cos 2\pi s,\sin 2\pi s)
\]
is an invariant orbit.

For the explicit-Euler update,
\[
u_{t+1}=u_t+\eta v_t,
\qquad
v_{t+1}=v_t-\eta u_t.
\]
Thus
\[
u_{t+1}^2+v_{t+1}^2
=
(u_t+\eta v_t)^2+(v_t-\eta u_t)^2
=
(1+\eta^2)(u_t^2+v_t^2).
\]
Equivalently,
\[
\|z_{t+1}\|^2=(1+\eta^2)\|z_t\|^2.
\]
Since $1+\eta^2>1$ for every $\eta>0$, every nonzero iterate moves strictly away from the origin, and iteration yields
\[
\|z_t\|=(1+\eta^2)^{t/2}\|z_0\|.
\]
This proves both claims.
\end{proof}

\subsection{Proof of Lemma~\ref{lem:disc_orth}}
\begin{proof}
Endow $\R^m$ (edge-flow space) with the weighted inner product
$\langle a,b\rangle_W:=a^\top W b$ and weighted norm $\|a\|_W^2:=a^\top W a$ (here $W$ is symmetric positive definite,
typically diagonal with positive weights). Consider the weighted least-squares projection of the flow $\omega$
onto the discrete gradient subspace $\mathrm{Range}(B)$, i.e. define the objective
\[
J(\phi):=\frac12\|\omega-B\phi\|_W^2=\frac12(\omega-B\phi)^\top W(\omega-B\phi),\qquad \phi\in\R^N.
\]
Let $\phi^\star$ be any minimizer of $J$, and define the residual (cyclic component)
\[
\omega_{\mathrm{cyc}}:=\omega-B\phi^\star.
\]
Fix an arbitrary direction $\psi\in\R^N$ and consider the one-dimensional perturbation
$\phi_\varepsilon:=\phi^\star+\varepsilon\psi$. Then
\[
\omega-B\phi_\varepsilon=\omega-B(\phi^\star+\varepsilon\psi)=(\omega-B\phi^\star)-\varepsilon B\psi
=\omega_{\mathrm{cyc}}-\varepsilon B\psi.
\]
Substituting into $J$ and expanding the quadratic form gives
\begin{align*}
J(\phi_\varepsilon)
&=\frac12(\omega_{\mathrm{cyc}}-\varepsilon B\psi)^\top W(\omega_{\mathrm{cyc}}-\varepsilon B\psi) \\
&=\frac12\Big(\omega_{\mathrm{cyc}}^\top W\omega_{\mathrm{cyc}}
-\varepsilon\,\omega_{\mathrm{cyc}}^\top W B\psi
-\varepsilon\,(B\psi)^\top W\omega_{\mathrm{cyc}}
+\varepsilon^2 (B\psi)^\top W(B\psi)\Big).
\end{align*}
Because $W$ is symmetric, $\omega_{\mathrm{cyc}}^\top W B\psi=(B\psi)^\top W\omega_{\mathrm{cyc}}$, hence
\[
J(\phi_\varepsilon)
=\frac12\,\omega_{\mathrm{cyc}}^\top W\omega_{\mathrm{cyc}}
-\varepsilon\,(B\psi)^\top W\omega_{\mathrm{cyc}}
+\frac{\varepsilon^2}{2}\,(B\psi)^\top W(B\psi).
\]
Differentiate with respect to $\varepsilon$:
\[
\frac{d}{d\varepsilon}J(\phi_\varepsilon)
=-(B\psi)^\top W\omega_{\mathrm{cyc}}+\varepsilon\,(B\psi)^\top W(B\psi).
\]
Since $\phi^\star$ is a minimizer of $J$ over the whole space $\R^N$, the first-order optimality condition along
every direction $\psi$ is
\[
\frac{d}{d\varepsilon}J(\phi^\star+\varepsilon\psi)\Big|_{\varepsilon=0}=0.
\]
Evaluating the derivative at $\varepsilon=0$ yields
\[
-(B\psi)^\top W\omega_{\mathrm{cyc}}=0,
\]
i.e.
\[
(B\psi)^\top W\,\omega_{\mathrm{cyc}}=0,\qquad \forall\,\psi\in\R^N.
\]
This is exactly the stated orthogonality: the residual $\omega_{\mathrm{cyc}}=\omega-B\phi^\star$ is orthogonal
(in the weighted inner product $\langle\cdot,\cdot\rangle_W$) to every discrete gradient flow $B\psi$.
Equivalently, since $(B\psi)^\top W\omega_{\mathrm{cyc}}=\psi^\top(B^\top W\omega_{\mathrm{cyc}})$ for all $\psi$,
one may write the normal equation $B^\top W\omega_{\mathrm{cyc}}=0$.
\end{proof}
% \begin{proof}
% This is the normal-equation condition $B^\top W(\omega-B\phi^\star)=0$ implied by \eqref{eq:graph_poisson}.
% \end{proof}

\subsection{Proof of Corollary~\ref{cor:disc_pythag}}
\begin{proof}
Recall that the weighted inner product and norm are
$\langle a,b\rangle_W:=a^\top W b$ and $\|a\|_W^2:=a^\top W a$.
By definition,
\[
\omega_{\mathrm{pot}}=B\phi^\star,
\qquad
\omega_{\mathrm{cyc}}=\omega-\omega_{\mathrm{pot}},
\qquad\text{hence}\qquad
\omega=\omega_{\mathrm{pot}}+\omega_{\mathrm{cyc}}.
\]
Expand the weighted energy using bilinearity and symmetry of $\langle\cdot,\cdot\rangle_W$ (since $W=W^\top$):
\begin{align*}
\|\omega\|_W^2
&=\langle \omega,\omega\rangle_W \\
&=\langle \omega_{\mathrm{pot}}+\omega_{\mathrm{cyc}},\,\omega_{\mathrm{pot}}+\omega_{\mathrm{cyc}}\rangle_W \\
&=\langle \omega_{\mathrm{pot}},\omega_{\mathrm{pot}}\rangle_W
+\langle \omega_{\mathrm{pot}},\omega_{\mathrm{cyc}}\rangle_W
+\langle \omega_{\mathrm{cyc}},\omega_{\mathrm{pot}}\rangle_W
+\langle \omega_{\mathrm{cyc}},\omega_{\mathrm{cyc}}\rangle_W \\
&=\|\omega_{\mathrm{pot}}\|_W^2
+2\,\langle \omega_{\mathrm{pot}},\omega_{\mathrm{cyc}}\rangle_W
+\|\omega_{\mathrm{cyc}}\|_W^2.
\end{align*}
It remains to show the cross term vanishes. By Lemma~\ref{lem:disc_orth}, for every $\psi\in\R^N$,
\[
(B\psi)^\top W\,\omega_{\mathrm{cyc}}=0.
\]
Choose $\psi=\phi^\star$. Then
\[
(B\phi^\star)^\top W\,\omega_{\mathrm{cyc}}=0.
\]
Since $B\phi^\star=\omega_{\mathrm{pot}}$, this is exactly
\[
\langle \omega_{\mathrm{pot}},\omega_{\mathrm{cyc}}\rangle_W
=
\omega_{\mathrm{pot}}^\top W\,\omega_{\mathrm{cyc}}=0.
\]
Substituting $\langle \omega_{\mathrm{pot}},\omega_{\mathrm{cyc}}\rangle_W=0$ into the expansion above yields
\[
\|\omega\|_W^2=\|\omega_{\mathrm{pot}}\|_W^2+\|\omega_{\mathrm{cyc}}\|_W^2,
\]
which is the claimed discrete energy decomposition.
\end{proof}

\subsection{Proof of Lemma~\ref{lem:nonpot_props}}
\begin{proof}
Recall the discrete non-potentiality proxy (Definition~4.5 in the paper):
for $\varepsilon>0$,
\[
\mathrm{NonPot}(\omega)
:=
\frac{\|\omega_{\mathrm{cyc}}\|_W^2}{\|\omega\|_W^2+\varepsilon},
\qquad
\|\eta\|_W^2 := \eta^\top W \eta,
\]
where $\omega_{\mathrm{pot}}=B\phi^\star$ is the weighted least-squares potential flow and
$\omega_{\mathrm{cyc}}:=\omega-\omega_{\mathrm{pot}}$ is the residual.

Since $W\succeq 0$ (in fact $W\succ 0$ in typical constructions), we have $\|\omega_{\mathrm{cyc}}\|_W^2\ge 0$.
Also $\|\omega\|_W^2\ge 0$ and $\varepsilon>0$, hence $\|\omega\|_W^2+\varepsilon>0$.
Therefore
\[
\mathrm{NonPot}(\omega)=\frac{\|\omega_{\mathrm{cyc}}\|_W^2}{\|\omega\|_W^2+\varepsilon}\ge 0.
\]
To show $\mathrm{NonPot}(\omega)\le 1$, we use the discrete energy decomposition
(Corollary~\ref{cor:disc_pythag}):
\[
\|\omega\|_W^2=\|\omega_{\mathrm{pot}}\|_W^2+\|\omega_{\mathrm{cyc}}\|_W^2.
\]
In particular, $\|\omega\|_W^2\ge \|\omega_{\mathrm{cyc}}\|_W^2$, and hence
\[
\|\omega\|_W^2+\varepsilon \;\ge\; \|\omega_{\mathrm{cyc}}\|_W^2.
\]
Dividing both sides by the positive quantity $\|\omega\|_W^2+\varepsilon$ yields
\[
\frac{\|\omega_{\mathrm{cyc}}\|_W^2}{\|\omega\|_W^2+\varepsilon}\le 1,
\]
i.e.\ $\mathrm{NonPot}(\omega)\le 1$.

Next, $\mathrm{NonPot}(\omega)=0$ if and only if $\|\omega_{\mathrm{cyc}}\|_W^2=0$ (because the denominator is
strictly positive). If $W\succ0$, then $\|\omega_{\mathrm{cyc}}\|_W^2=0$ implies $\omega_{\mathrm{cyc}}=0$, hence
\[
\omega=\omega_{\mathrm{pot}}=B\phi^\star\in \mathrm{Range}(B),
\]
so $\omega$ is exactly a potential (gradient/cut-space) flow on the graph. Conversely, if $\omega$ is a potential
flow, i.e.\ $\omega=B\phi$ for some $\phi$, then the best-fit potential flow is $\omega_{\mathrm{pot}}=\omega$ and
$\omega_{\mathrm{cyc}}=0$, hence $\mathrm{NonPot}(\omega)=0$.

The phrase ``$\mathrm{NonPot}(\omega)\approx 0$'' can be made precise as a relative-distance statement.
Since $\omega_{\mathrm{pot}}$ is the weighted least-squares projection onto $\mathrm{Range}(B)$, one has
\[
\mathrm{dist}_W(\omega,\mathrm{Range}(B)):=\inf_{\phi\in\R^N}\|\omega-B\phi\|_W=\|\omega-\omega_{\mathrm{pot}}\|_W
=\|\omega_{\mathrm{cyc}}\|_W.
\]
Thus
\[
\mathrm{NonPot}(\omega)=\frac{\mathrm{dist}_W(\omega,\mathrm{Range}(B))^2}{\|\omega\|_W^2+\varepsilon},
\]
so $\mathrm{NonPot}(\omega)$ is small exactly when the residual distance to the potential-flow subspace is small
relative to the total energy scale $\|\omega\|_W^2+\varepsilon$.

Finally, to formalize ``$\mathrm{NonPot}(\omega)$ increases with the energy of circulation components supported by
cycles'', note that on a graph (a 1-dimensional complex) the residual $\omega_{\mathrm{cyc}}$ is precisely the
cycle-space (circulation) component of $\omega$, and its circulation energy is $\|\omega_{\mathrm{cyc}}\|_W^2$.
Using the decomposition $\|\omega\|_W^2=\|\omega_{\mathrm{pot}}\|_W^2+\|\omega_{\mathrm{cyc}}\|_W^2$, define
\[
p:=\|\omega_{\mathrm{pot}}\|_W^2\ge 0,
\qquad
s:=\|\omega_{\mathrm{cyc}}\|_W^2\ge 0.
\]
Then
\[
\mathrm{NonPot}(\omega)=\frac{s}{p+s+\varepsilon}.
\]
For fixed $p$ and $\varepsilon>0$, the function $f(s):=\frac{s}{p+s+\varepsilon}$ is strictly increasing in $s$:
for $s\ge 0$,
\[
f'(s)=\frac{p+\varepsilon}{(p+s+\varepsilon)^2}>0.
\]
Hence, when the potential component is fixed, increasing the circulation (cycle-space) energy $s=\|\omega_{\mathrm{cyc}}\|_W^2$
strictly increases $\mathrm{NonPot}(\omega)$, matching the intended interpretation.
\end{proof}

\subsection{Proof of Proposition~\ref{prop:metric_discrete}}
\begin{proof}
Fix $x\in\cX$ and a displacement $\Delta x$ such that the segment $x+t\Delta x\in\cX$ for all $t\in[0,1]$.
Define the one-dimensional function
\[
g(t):=\Phi(x+t\Delta x),\qquad t\in[0,1].
\]
Since $\Phi\in C^2(\cX)$, we have $g\in C^2([0,1])$, and by the chain rule,
\[
g'(t)=\langle \nabla\Phi(x+t\Delta x),\,\Delta x\rangle,
\qquad
g''(t)=\Delta x^\top \nabla^2\Phi(x+t\Delta x)\,\Delta x.
\]
Using the second-order Taylor formula with integral remainder for $g$ around $t=0$ evaluated at $t=1$,
\[
g(1)=g(0)+g'(0)\cdot 1+\int_0^1(1-t)\,g''(t)\,dt,
\]
and substituting the expressions for $g(0),g(1),g'(0),g''(t)$ gives
\[
\Phi(x+\Delta x)-\Phi(x)
=
\langle \nabla\Phi(x),\,\Delta x\rangle
+
\int_0^1(1-t)\,\Delta x^\top \nabla^2\Phi(x+t\Delta x)\,\Delta x\,dt.
\]
The remainder term is quadratic in $\Delta x$: indeed, letting
\[
H_*:=\sup_{t\in[0,1]}\big\|\nabla^2\Phi(x+t\Delta x)\big\|_{\mathrm{op}}<\infty
\]
(which is finite by continuity of $\nabla^2\Phi$ on the compact segment), we have for every $t\in[0,1]$,
\[
\big|\Delta x^\top \nabla^2\Phi(x+t\Delta x)\,\Delta x\big|
\le
\big\|\nabla^2\Phi(x+t\Delta x)\big\|_{\mathrm{op}}\;\|\Delta x\|^2
\le
H_*\;\|\Delta x\|^2,
\]
hence
\[
\left|\int_0^1(1-t)\,\Delta x^\top \nabla^2\Phi(x+t\Delta x)\,\Delta x\,dt\right|
\le
\int_0^1(1-t)\,H_*\,\|\Delta x\|^2\,dt
=
\frac{H_*}{2}\,\|\Delta x\|^2
=
O(\|\Delta x\|^2).
\]
Therefore
\[
\Phi(x+\Delta x)-\Phi(x)=\langle \nabla\Phi(x),\,\Delta x\rangle + O(\|\Delta x\|^2).
\]
Since $M\succ0$ is constant and $\nabla_M\Phi:=M^{-1}\nabla\Phi$, we have the identity
\[
M\nabla_M\Phi(x)=M(M^{-1}\nabla\Phi(x))=\nabla\Phi(x),
\]
and thus
\[
\langle \nabla\Phi(x),\,\Delta x\rangle
=
\langle M\nabla_M\Phi(x),\,\Delta x\rangle,
\]
which gives the second displayed equality in the proposition.

For the graph statement, let nodes have positions $\{x^i\}$ and write $\Delta_{ij}:=x^j-x^i$.
Define the edge $1$-form
\[
\omega_{ij}:=\Big\langle M\,\frac{G(x^i)+G(x^j)}{2},\,\Delta_{ij}\Big\rangle.
\]
If $G=\nabla_M\Phi$, then $M G=\nabla\Phi$, so
\[
\omega_{ij}
=
\Big\langle \frac{\nabla\Phi(x^i)+\nabla\Phi(x^j)}{2},\,\Delta_{ij}\Big\rangle.
\]
On the other hand, by the fundamental theorem of calculus along the straight segment
$\gamma(t):=x^i+t\Delta_{ij}$, we have
\[
\Phi(x^j)-\Phi(x^i)=\int_0^1 \frac{d}{dt}\Phi(\gamma(t))\,dt
=\int_0^1 \langle \nabla\Phi(\gamma(t)),\,\gamma'(t)\rangle\,dt
=\int_0^1 \langle \nabla\Phi(x^i+t\Delta_{ij}),\,\Delta_{ij}\rangle\,dt.
\]
To compare this integral to the endpoint average, use the integral representation of the gradient.
For each $t\in[0,1]$, the map $s\mapsto \nabla\Phi(x^i+s\Delta_{ij})$ is $C^1$, and
\[
\nabla\Phi(x^i+t\Delta_{ij})
=
\nabla\Phi(x^i)+\int_0^t \nabla^2\Phi(x^i+s\Delta_{ij})\,\Delta_{ij}\,ds.
\]
Substituting into the line integral yields
\begin{align*}
\Phi(x^j)-\Phi(x^i)
&=\int_0^1 \Big\langle \nabla\Phi(x^i)+\int_0^t \nabla^2\Phi(x^i+s\Delta_{ij})\,\Delta_{ij}\,ds,\ \Delta_{ij}\Big\rangle\,dt \\
&=\int_0^1 \langle \nabla\Phi(x^i),\Delta_{ij}\rangle\,dt
+\int_0^1\int_0^t \Delta_{ij}^\top \nabla^2\Phi(x^i+s\Delta_{ij})\,\Delta_{ij}\,ds\,dt \\
&=\langle \nabla\Phi(x^i),\Delta_{ij}\rangle
+\int_0^1(1-s)\,\Delta_{ij}^\top \nabla^2\Phi(x^i+s\Delta_{ij})\,\Delta_{ij}\,ds,
\end{align*}
where the last line uses Fubini's theorem to swap the order of integration:
$\int_0^1\int_0^t(\cdot)\,ds\,dt=\int_0^1\int_s^1(\cdot)\,dt\,ds=\int_0^1(1-s)(\cdot)\,ds$.
Similarly, using the same identity at $t=1$ gives
\[
\nabla\Phi(x^j)=\nabla\Phi(x^i+\Delta_{ij})
=
\nabla\Phi(x^i)+\int_0^1 \nabla^2\Phi(x^i+s\Delta_{ij})\,\Delta_{ij}\,ds,
\]
hence
\[
\Big\langle \frac{\nabla\Phi(x^i)+\nabla\Phi(x^j)}{2},\,\Delta_{ij}\Big\rangle
=
\langle \nabla\Phi(x^i),\Delta_{ij}\rangle
+\frac12\int_0^1 \Delta_{ij}^\top \nabla^2\Phi(x^i+s\Delta_{ij})\,\Delta_{ij}\,ds.
\]
Subtracting the two expressions yields
\begin{align*}
\omega_{ij}-\big(\Phi(x^j)-\Phi(x^i)\big)
&=
\int_0^1\Big(\frac12-(1-s)\Big)\,\Delta_{ij}^\top \nabla^2\Phi(x^i+s\Delta_{ij})\,\Delta_{ij}\,ds \\
&=
\int_0^1\Big(s-\frac12\Big)\,\Delta_{ij}^\top \nabla^2\Phi(x^i+s\Delta_{ij})\,\Delta_{ij}\,ds.
\end{align*}
Let $H_{ij}:=\sup_{s\in[0,1]}\|\nabla^2\Phi(x^i+s\Delta_{ij})\|_{\mathrm{op}}<\infty$. Then
\[
\big|\omega_{ij}-\big(\Phi(x^j)-\Phi(x^i)\big)\big|
\le
\int_0^1\Big|s-\frac12\Big|\,H_{ij}\,\|\Delta_{ij}\|^2\,ds
=
H_{ij}\,\|\Delta_{ij}\|^2\int_0^1\Big|s-\frac12\Big|\,ds
=
\frac{H_{ij}}{4}\,\|\Delta_{ij}\|^2.
\]
Thus
\[
\omega_{ij}=\Phi(x^j)-\Phi(x^i)+O(\|x^j-x^i\|^2),
\]
which is the claimed second-order (metric-consistent) approximation of potential differences.
\end{proof}

\subsection{Proof of Lemma~\ref{lem:proj_vi}}
\begin{proof}
Fix $y\in\R^d$ and write $p:=P(y)\in\cX$. By definition of the Euclidean projection onto a nonempty closed convex set,
$p$ is the unique minimizer of the strongly convex problem
\[
\min_{x\in\cX}\; f(x):=\frac12\|x-y\|^2.
\]
Let $z\in\cX$ be arbitrary and consider the line segment
\[
x(t):=p+t(z-p),\qquad t\in[0,1].
\]
Because $\cX$ is convex, $x(t)\in\cX$ for all $t\in[0,1]$. Since $p$ minimizes $f$ over $\cX$, the one-variable function
$t\mapsto f(x(t))$ attains its minimum at $t=0$ on $[0,1]$, hence its right derivative at $t=0$ is nonnegative:
\[
\frac{d}{dt} f(x(t))\Big|_{t=0^+}\ge 0.
\]
We compute this derivative explicitly. First,
\[
f(x(t))=\frac12\|x(t)-y\|^2=\frac12\|p+t(z-p)-y\|^2.
\]
Differentiate using $\frac{d}{dt}\frac12\|u(t)\|^2=\langle u(t),u'(t)\rangle$ with
$u(t)=p+t(z-p)-y$ and $u'(t)=z-p$:
\[
\frac{d}{dt} f(x(t))
=
\big\langle p+t(z-p)-y,\;z-p\big\rangle.
\]
Evaluating at $t=0$ gives
\[
\frac{d}{dt} f(x(t))\Big|_{t=0}
=
\langle p-y,\;z-p\rangle.
\]
Combining with the necessary optimality condition yields
\[
\langle p-y,\;z-p\rangle\ge 0.
\]
Multiplying by $-1$ and recalling $p=P(y)$ gives the projection variational inequality
\[
\langle y-P(y),\;z-P(y)\rangle\le 0,\qquad \forall z\in\cX.
\]

For firm nonexpansiveness, take arbitrary $y,y'\in\R^d$ and denote
\[
p:=P(y),\qquad q:=P(y').
\]
Apply the projection variational inequality twice. First, with $(y,z)=(y,q)$ (note $q\in\cX$), we obtain
\[
\langle y-p,\;q-p\rangle\le 0.
\]
Second, with $(y',z)=(y',p)$ (note $p\in\cX$), we obtain
\[
\langle y'-q,\;p-q\rangle\le 0.
\]
Rewrite the second inequality by factoring out a minus sign from $(p-q)=-(q-p)$:
\[
\langle y'-q,\;-(q-p)\rangle\le 0
\qquad\Longleftrightarrow\qquad
\langle y'-q,\;q-p\rangle\ge 0.
\]
Now add this to the first inequality:
\[
\langle y-p,\;q-p\rangle + \langle y'-q,\;q-p\rangle \le 0.
\]
Combine the inner products:
\[
\langle (y-p)+(y'-q),\;q-p\rangle\le 0.
\]
Group terms as $(y-y')-(p-q)$ by noting $(y-p)+(y'-q)=(y-y')-(p-q)$ and $q-p=-(p-q)$:
\[
(y-p)+(y'-q)=y+y'-(p+q)=(y-y')-(p-q),\qquad q-p=-(p-q).
\]
Substitute these identities:
\[
\langle (y-y')-(p-q),\;-(p-q)\rangle\le 0.
\]
Use bilinearity to expand the left-hand side:
\[
\langle (y-y')-(p-q),\;-(p-q)\rangle
=
-\langle y-y',\;p-q\rangle + \langle p-q,\;p-q\rangle.
\]
Therefore the inequality becomes
\[
-\langle y-y',\;p-q\rangle + \|p-q\|^2 \le 0,
\]
which rearranges to the firm nonexpansiveness inequality
\[
\|P(y)-P(y')\|^2=\|p-q\|^2 \le \langle p-q,\;y-y'\rangle
=\langle P(y)-P(y'),\;y-y'\rangle.
\]
Finally, by Cauchy--Schwarz,
\[
\|p-q\|^2 \le \langle p-q,\;y-y'\rangle \le \|p-q\|\,\|y-y'\|.
\]
If $p=q$ then $\|p-q\|=0$ and the desired nonexpansiveness $\|p-q\|\le\|y-y'\|$ is trivial.
If $p\neq q$, divide both sides by the positive quantity $\|p-q\|$ to obtain
\[
\|P(y)-P(y')\|=\|p-q\| \le \|y-y'\|.
\]
\end{proof}

\subsection{Proof of Theorem~\ref{thm:lyapunov}}
\begin{proof}
Let $g_t:=\nabla\Phi(x_t)$ and denote the step $\Delta_t:=x_{t+1}-x_t$.
We will combine (a) the $L$-smooth lower quadratic bound for $\Phi$ and (b) the projection variational inequality.

\begin{lemma}[Two-sided smoothness (quadratic lower bound)]
\label{lem:smooth_lower}
If $\Phi$ has $L$-Lipschitz gradient on $\cX$, then for all $x,y\in\cX$,
\[
\Phi(y)\ge \Phi(x)+\langle \nabla\Phi(x),y-x\rangle-\frac{L}{2}\|y-x\|^2.
\]
\end{lemma}
\begin{proof}
Fix $x,y$ and define $h(t):=\Phi(x+t(y-x))$ for $t\in[0,1]$. Then $h$ is differentiable and
\[
h'(t)=\langle \nabla\Phi(x+t(y-x)),\,y-x\rangle.
\]
By the fundamental theorem of calculus,
\[
\Phi(y)-\Phi(x)=h(1)-h(0)=\int_0^1 h'(t)\,dt
=\int_0^1 \langle \nabla\Phi(x+t(y-x)),\,y-x\rangle\,dt.
\]
Add and subtract $\nabla\Phi(x)$ inside the integrand:
\[
\Phi(y)-\Phi(x)
=\int_0^1 \langle \nabla\Phi(x),\,y-x\rangle\,dt
+\int_0^1 \langle \nabla\Phi(x+t(y-x))-\nabla\Phi(x),\,y-x\rangle\,dt.
\]
The first integral equals $\langle \nabla\Phi(x),y-x\rangle$. For the second, apply Cauchy--Schwarz and
$L$-Lipschitzness of the gradient:
\begin{align*}
\left|\left\langle \nabla\Phi(x+t(y-x))-\nabla\Phi(x),\,y-x\right\rangle\right|
&\le \|\nabla\Phi(x+t(y-x))-\nabla\Phi(x)\|\,\|y-x\| \\
&\le L\,\|t(y-x)\|\,\|y-x\| \\
&= L\,t\,\|y-x\|^2.
\end{align*}
Therefore the second integral is bounded below by the negative of its absolute value:
\[
\int_0^1 \langle \nabla\Phi(x+t(y-x))-\nabla\Phi(x),\,y-x\rangle\,dt
\ge -\int_0^1 L\,t\,\|y-x\|^2\,dt
= -\frac{L}{2}\|y-x\|^2.
\]
Combining the identities yields
\[
\Phi(y)-\Phi(x)\ge \langle \nabla\Phi(x),y-x\rangle-\frac{L}{2}\|y-x\|^2,
\]
which is the claimed inequality.
\end{proof}

Apply Lemma~\ref{lem:smooth_lower} with $x=x_t$ and $y=x_{t+1}$:
\[
\Phi(x_{t+1})
\ge
\Phi(x_t)+\langle g_t,\Delta_t\rangle-\frac{L}{2}\|\Delta_t\|^2.
\]
It remains to lower bound $\langle g_t,\Delta_t\rangle$ in terms of $\|\Delta_t\|^2$ using the projection step.
Let
\[
y_t:=x_t+\eta g_t,
\qquad
x_{t+1}=\Proj_{\cX}(y_t).
\]
By Lemma~\ref{lem:proj_vi} (projection variational inequality), for any $z\in\cX$,
\[
\langle y_t-x_{t+1},\,z-x_{t+1}\rangle\le 0.
\]
Choose $z=x_t\in\cX$. Then
\[
\langle y_t-x_{t+1},\,x_t-x_{t+1}\rangle\le 0.
\]
Substitute $y_t=x_t+\eta g_t$ and $\Delta_t=x_{t+1}-x_t$:
\[
y_t-x_{t+1}=(x_t+\eta g_t)-x_{t+1}=-(x_{t+1}-x_t)+\eta g_t=-\Delta_t+\eta g_t,
\qquad
x_t-x_{t+1}=-(x_{t+1}-x_t)=-\Delta_t.
\]
Hence
\[
\langle -\Delta_t+\eta g_t,\,-\Delta_t\rangle\le 0.
\]
Expand the inner product using bilinearity:
\[
\langle -\Delta_t,\,-\Delta_t\rangle+\eta\langle g_t,\,-\Delta_t\rangle\le 0
\quad\Longleftrightarrow\quad
\|\Delta_t\|^2-\eta\langle g_t,\Delta_t\rangle\le 0.
\]
Rearranging yields the key inequality
\[
\langle g_t,\Delta_t\rangle\ge \frac{1}{\eta}\|\Delta_t\|^2.
\]
Substitute this into the smoothness lower bound:
\[
\Phi(x_{t+1})
\ge
\Phi(x_t)+\frac{1}{\eta}\|\Delta_t\|^2-\frac{L}{2}\|\Delta_t\|^2
=
\Phi(x_t)+\Big(\frac{1}{\eta}-\frac{L}{2}\Big)\|\Delta_t\|^2,
\]
which proves the first inequality in the theorem.

Finally, if $\eta\le 1/L$, then $L\le 1/\eta$, hence
\[
\frac{1}{\eta}-\frac{L}{2}\ge \frac{1}{\eta}-\frac{1}{2\eta}=\frac{1}{2\eta}.
\]
Therefore
\[
\Phi(x_{t+1})
\ge
\Phi(x_t)+\frac{1}{2\eta}\,\|x_{t+1}-x_t\|^2,
\]
which is the second inequality.
\end{proof}

\subsection{Proof of Lemma~\ref{lem:gap_residual}}
\begin{proof}
Fix any $x\in\cX$. By definition,
\[
\Gap(x)=\max_{u\in\cX}\,\langle -\nabla\Phi(x)+R(x),\,x-u\rangle.
\]
For an arbitrary $u\in\cX$, expand the inner product by bilinearity:
\[
\langle -\nabla\Phi(x)+R(x),\,x-u\rangle
=
\langle -\nabla\Phi(x),\,x-u\rangle+\langle R(x),\,x-u\rangle.
\]
We upper bound the residual term using Cauchy--Schwarz:
\[
\langle R(x),\,x-u\rangle
\le
\|R(x)\|\,\|x-u\|.
\]
Since $\cX$ has diameter $D:=\max_{a,b\in\cX}\|a-b\|$, in particular for this fixed $x\in\cX$ and any $u\in\cX$,
\[
\|x-u\|\le D.
\]
Therefore, for every $u\in\cX$,
\[
\langle R(x),\,x-u\rangle\le \|R(x)\|\,D.
\]
Substitute this bound back into the expansion:
\[
\langle -\nabla\Phi(x)+R(x),\,x-u\rangle
\le
\langle -\nabla\Phi(x),\,x-u\rangle + D\,\|R(x)\|.
\]
Now take the maximum over $u\in\cX$ on both sides. The right-hand side is a maximum of a function plus a constant
(independent of $u$), so
\begin{align*}
\Gap(x)
=\max_{u\in\cX}\langle -\nabla\Phi(x)+R(x),\,x-u\rangle
&\le
\max_{u\in\cX}\Big(\langle -\nabla\Phi(x),\,x-u\rangle + D\,\|R(x)\|\Big) \\
&=
\max_{u\in\cX}\langle -\nabla\Phi(x),\,x-u\rangle + D\,\|R(x)\| \\
&= \Gap_\Phi(x)+D\,\|R(x)\|.
\end{align*}
Finally, the alternative expression for $\Gap_\Phi(x)$ follows from pulling out the minus sign:
for every $u\in\cX$,
\[
\langle -\nabla\Phi(x),\,x-u\rangle=\langle \nabla\Phi(x),\,u-x\rangle,
\]
hence
\[
\Gap_\Phi(x)
=\max_{u\in\cX}\langle -\nabla\Phi(x),\,x-u\rangle
=\max_{u\in\cX}\langle \nabla\Phi(x),\,u-x\rangle.
\]
\end{proof}

\subsection{Proof of Lemma~\ref{lem:gap_mapping}}

\begin{proof}
Fix $x\in\cX$ and define
\[
y:=x+\eta\nabla\Phi(x),\qquad x^+:=\Proj_{\cX}(y),\qquad \cG_\eta(x):=\frac{1}{\eta}(x^+-x).
\]
Take an arbitrary $u\in\cX$. By the projection variational inequality in Lemma~\ref{lem:proj_vi},
\[
\langle y-x^+,\,u-x^+\rangle\le 0.
\]
Substitute $y=x+\eta\nabla\Phi(x)$:
\[
\langle x+\eta\nabla\Phi(x)-x^+,\,u-x^+\rangle\le 0.
\]
Split the inner product into two terms:
\[
\langle x-x^+,\,u-x^+\rangle+\eta\langle \nabla\Phi(x),\,u-x^+\rangle\le 0,
\]
hence
\[
\eta\langle \nabla\Phi(x),\,u-x^+\rangle \le \langle x^+-x,\,u-x^+\rangle.
\]
Now express $u-x^+$ as $(u-x)+(x-x^+)$:
\[
u-x^+=(u-x)+(x-x^+)= (u-x) - (x^+-x).
\]
Therefore
\[
\langle \nabla\Phi(x),\,u-x^+\rangle
=
\langle \nabla\Phi(x),\,u-x\rangle - \langle \nabla\Phi(x),\,x^+-x\rangle,
\]
and the inequality becomes
\[
\eta\langle \nabla\Phi(x),\,u-x\rangle
\le
\langle x^+-x,\,u-x^+\rangle + \eta\langle \nabla\Phi(x),\,x^+-x\rangle.
\]
We upper bound the last term on the right by using again the projection VI with the choice $z=x\in\cX$.
Lemma~\ref{lem:proj_vi} with $z=x$ gives
\[
\langle y-x^+,\,x-x^+\rangle\le 0.
\]
Substituting $y=x+\eta\nabla\Phi(x)$ and writing $x-x^+=-(x^+-x)$ yields
\[
\langle x+\eta\nabla\Phi(x)-x^+,\,x-x^+\rangle\le 0
\quad\Longleftrightarrow\quad
\langle - (x^+-x)+\eta\nabla\Phi(x),\, -(x^+-x)\rangle\le 0.
\]
Expanding gives
\[
\|x^+-x\|^2-\eta\langle \nabla\Phi(x),\,x^+-x\rangle\le 0,
\]
hence
\[
\eta\langle \nabla\Phi(x),\,x^+-x\rangle \ge \|x^+-x\|^2.
\]
Substitute this lower bound into the previous inequality (note it appears with a $+$ sign on the right):
\[
\eta\langle \nabla\Phi(x),\,u-x\rangle
\le
\langle x^+-x,\,u-x^+\rangle + \eta\langle \nabla\Phi(x),\,x^+-x\rangle
\le
\langle x^+-x,\,u-x^+\rangle + \|x^+-x\|^2.
\]
Now bound the remaining inner product by Cauchy--Schwarz:
\[
\langle x^+-x,\,u-x^+\rangle
\le
\|x^+-x\|\,\|u-x^+\|.
\]
Because both $u$ and $x^+$ lie in $\cX$ and $\cX$ has diameter $D$, we have $\|u-x^+\|\le D$, hence
\[
\langle x^+-x,\,u-x^+\rangle\le D\,\|x^+-x\|.
\]
Combining yields, for every $u\in\cX$,
\[
\eta\langle \nabla\Phi(x),\,u-x\rangle
\le
D\,\|x^+-x\|+\|x^+-x\|^2.
\]
Divide by $\eta>0$ and take the maximum over $u\in\cX$:
\begin{align*}
\Gap_\Phi(x)
=\max_{u\in\cX}\langle \nabla\Phi(x),\,u-x\rangle
&\le
\frac{D}{\eta}\|x^+-x\|+\frac{1}{\eta}\|x^+-x\|^2.
\end{align*}
Rewrite in terms of $\cG_\eta(x)=\frac{1}{\eta}(x^+-x)$:
\[
\frac{D}{\eta}\|x^+-x\|=D\,\|\cG_\eta(x)\|,
\qquad
\frac{1}{\eta}\|x^+-x\|^2=\eta\,\|\cG_\eta(x)\|^2.
\]
It remains to upper bound $\eta\,\|\cG_\eta(x)\|^2$ by $\eta G\,\|\cG_\eta(x)\|$.
By nonexpansiveness of the projection (Lemma~\ref{lem:proj_vi}), for any $a,b\in\R^d$,
$\|\Proj_{\cX}(a)-\Proj_{\cX}(b)\|\le \|a-b\|$. Apply this with $a=x+\eta\nabla\Phi(x)$ and $b=x$:
\[
\|x^+-x\|=\|\Proj_{\cX}(x+\eta\nabla\Phi(x))-\Proj_{\cX}(x)\|
\le
\|(x+\eta\nabla\Phi(x))-x\|
=
\eta\|\nabla\Phi(x)\|
\le
\eta G.
\]
Divide by $\eta$ to get $\|\cG_\eta(x)\|=\frac{1}{\eta}\|x^+-x\|\le G$ and thus
\[
\eta\,\|\cG_\eta(x)\|^2 \le \eta\,G\,\|\cG_\eta(x)\|.
\]
Substituting into the previous bound yields
\[
\Gap_\Phi(x)\le D\,\|\cG_\eta(x)\|+\eta G\,\|\cG_\eta(x)\|=(D+\eta G)\,\|\cG_\eta(x)\|,
\]
which is the claim.
\end{proof}

\subsection{Proof of Theorem~\ref{thm:gap}}
\begin{proof}
Fix any $x\in\cX$. By Assumption~\ref{ass:residual}, the VI operator decomposes as
$F(x)=-\nabla\Phi(x)+R(x)$ with $\|R(x)\|\le\varepsilon$. Applying Lemma~\ref{lem:gap_residual} gives
\[
\Gap(x)
=\max_{u\in\cX}\langle -\nabla\Phi(x)+R(x),\,x-u\rangle
\le
\Gap_\Phi(x)+D\|R(x)\|
\le
\Gap_\Phi(x)+D\varepsilon.
\]
Now apply Lemma~\ref{lem:gap_mapping} to bound $\Gap_\Phi$ in terms of the projected-step mapping
$\cG_\eta(x)=\frac{1}{\eta}(x^+-x)$ with $x^+=\Proj_{\cX}(x+\eta\nabla\Phi(x))$:
\[
\Gap_\Phi(x)\le (D+\eta G)\,\|\cG_\eta(x)\|.
\]
Combining the two inequalities yields, for every $x\in\cX$,
\[
\Gap(x)\le (D+\eta G)\,\|\cG_\eta(x)\|+D\varepsilon.
\]
We now specialize to the iterates $x_{t+1}=\Proj_{\cX}(x_t+\eta\nabla\Phi(x_t))$ with $\eta\le 1/L$.
For each $t$, define $x_t^+:=x_{t+1}$ and $\Delta_t:=x_{t+1}-x_t$, so that
\[
\cG_\eta(x_t)=\frac{1}{\eta}(x_t^+-x_t)=\frac{1}{\eta}\Delta_t,
\qquad\text{and hence}\qquad
\|\Delta_t\|^2=\eta^2\|\cG_\eta(x_t)\|^2.
\]
By Theorem~\ref{thm:lyapunov}, for every $t$,
\[
\Phi(x_{t+1})
\ge
\Phi(x_t)+\frac{1}{2\eta}\,\|x_{t+1}-x_t\|^2
=
\Phi(x_t)+\frac{1}{2\eta}\,\|\Delta_t\|^2
=
\Phi(x_t)+\frac{\eta}{2}\,\|\cG_\eta(x_t)\|^2.
\]
Rearranging gives
\[
\frac{\eta}{2}\,\|\cG_\eta(x_t)\|^2
\le
\Phi(x_{t+1})-\Phi(x_t).
\]
Sum this inequality from $t=0$ to $T-1$ and use telescoping:
\[
\frac{\eta}{2}\sum_{t=0}^{T-1}\|\cG_\eta(x_t)\|^2
\le
\sum_{t=0}^{T-1}\big(\Phi(x_{t+1})-\Phi(x_t)\big)
=
\Phi(x_T)-\Phi(x_0).
\]
By definition of $\Phi_{\max}$ and $\Phi_{\min}$ and the fact that all iterates lie in $\cX$ (since $\Proj_{\cX}(\cdot)\in\cX$),
\[
\Phi(x_T)\le \Phi_{\max},
\qquad
\Phi(x_0)\ge \Phi_{\min},
\qquad\text{so}\qquad
\Phi(x_T)-\Phi(x_0)\le \Phi_{\max}-\Phi_{\min}.
\]
Therefore
\[
\frac{\eta}{2}\sum_{t=0}^{T-1}\|\cG_\eta(x_t)\|^2
\le
\Phi_{\max}-\Phi_{\min},
\qquad\text{equivalently}\qquad
\sum_{t=0}^{T-1}\|\cG_\eta(x_t)\|^2
\le
\frac{2(\Phi_{\max}-\Phi_{\min})}{\eta}.
\]
Divide both sides by $T$:
\[
\frac{1}{T}\sum_{t=0}^{T-1}\|\cG_\eta(x_t)\|^2
\le
\frac{2(\Phi_{\max}-\Phi_{\min})}{T\,\eta}.
\]
Since $\min_{0\le t\le T-1}\|\cG_\eta(x_t)\|^2 \le \frac{1}{T}\sum_{t=0}^{T-1}\|\cG_\eta(x_t)\|^2$, we obtain
\[
\min_{0\le t\le T-1}\|\cG_\eta(x_t)\|^2
\le
\frac{2(\Phi_{\max}-\Phi_{\min})}{T\,\eta},
\]
and taking square roots (all terms are nonnegative) yields
\[
\min_{0\le t\le T-1}\|\cG_\eta(x_t)\|
\le
\sqrt{\frac{2(\Phi_{\max}-\Phi_{\min})}{T\,\eta}}.
\]
Let $\hat t\in\arg\min_{0\le t\le T-1}\|\cG_\eta(x_t)\|$ and $\hat x_T:=x_{\hat t}$. Then
\[
\|\cG_\eta(\hat x_T)\|
=
\|\cG_\eta(x_{\hat t})\|
=
\min_{0\le t\le T-1}\|\cG_\eta(x_t)\|
\le
\sqrt{\frac{2(\Phi_{\max}-\Phi_{\min})}{T\,\eta}}.
\]
Apply the pointwise gap bound $\Gap(x)\le (D+\eta G)\|\cG_\eta(x)\|+D\varepsilon$ at $x=\hat x_T$ to get
\[
\Gap(\hat x_T)
\le
(D+\eta G)\,\|\cG_\eta(\hat x_T)\|+D\varepsilon
\le
(D+\eta G)\sqrt{\frac{2(\Phi_{\max}-\Phi_{\min})}{T\,\eta}}+D\varepsilon,
\]
which is the claimed best-iterate guarantee.
\end{proof}

\subsection{Proof of Theorem~\ref{thm:gap_L2}}
\begin{proof}
Fix any deterministic $x\in\cX$. Using the decomposition $F(x)=-\nabla\Phi(x)+R(x)$ and applying
Lemma~\ref{lem:gap_residual}, we have
\[
\Gap(x)
=\max_{u\in\cX}\langle -\nabla\Phi(x)+R(x),\,x-u\rangle
\le
\Gap_\Phi(x)+D\|R(x)\|.
\]
By Lemma~\ref{lem:gap_mapping}, with $x^+=\Proj_{\cX}(x+\eta\nabla\Phi(x))$ and
$\cG_\eta(x)=\frac{1}{\eta}(x^+-x)$,
\[
\Gap_\Phi(x)\le (D+\eta G)\,\|\cG_\eta(x)\|.
\]
Combining the two bounds yields the pointwise inequality valid for every $x\in\cX$:
\[
\Gap(x)\le (D+\eta G)\,\|\cG_\eta(x)\|+D\,\|R(x)\|.
\]
Now specialize to the iterates generated by the algorithm and the random output $\tilde x_T:=x_{\tilde t}$,
where $\tilde t$ is uniform on $\{0,1,\dots,T-1\}$ and independent of the iterates. Taking expectations gives
\[
\mathbb{E}[\Gap(\tilde x_T)]
\le
(D+\eta G)\,\mathbb{E}\big[\|\cG_\eta(\tilde x_T)\|\big]
+
D\,\mathbb{E}\big[\|R(\tilde x_T)\|\big].
\]
We bound the residual term by Cauchy--Schwarz:
\[
\mathbb{E}\big[\|R(\tilde x_T)\|\big]
\le
\Big(\mathbb{E}\big[\|R(\tilde x_T)\|^2\big]\Big)^{1/2}
=
E_T.
\]
Next bound $\mathbb{E}\big[\|\cG_\eta(\tilde x_T)\|\big]$ using Jensen/Cauchy--Schwarz:
\[
\mathbb{E}\big[\|\cG_\eta(\tilde x_T)\|\big]
\le
\Big(\mathbb{E}\big[\|\cG_\eta(\tilde x_T)\|^2\big]\Big)^{1/2}.
\]
Because $\tilde t$ is uniform and independent, conditioning on the iterates yields
\[
\mathbb{E}\big[\|\cG_\eta(\tilde x_T)\|^2\ \big|\ x_0,\dots,x_{T-1}\big]
=
\frac{1}{T}\sum_{t=0}^{T-1}\|\cG_\eta(x_t)\|^2,
\]
and taking expectation over the iterates gives
\[
\mathbb{E}\big[\|\cG_\eta(\tilde x_T)\|^2\big]
=
\mathbb{E}\left[\frac{1}{T}\sum_{t=0}^{T-1}\|\cG_\eta(x_t)\|^2\right]
=
\frac{1}{T}\sum_{t=0}^{T-1}\mathbb{E}\big[\|\cG_\eta(x_t)\|^2\big].
\]
For the present (exact, deterministic) projected ascent dynamics, the iterates are deterministic given $x_0$,
so the expectations may be dropped; the same bound below holds verbatim even if randomness is present.

By Theorem~\ref{thm:lyapunov}, for each $t$,
\[
\Phi(x_{t+1})\ge \Phi(x_t)+\frac{\eta}{2}\|\cG_\eta(x_t)\|^2,
\]
equivalently,
\[
\frac{\eta}{2}\|\cG_\eta(x_t)\|^2\le \Phi(x_{t+1})-\Phi(x_t).
\]
Summing from $t=0$ to $T-1$ telescopes:
\[
\frac{\eta}{2}\sum_{t=0}^{T-1}\|\cG_\eta(x_t)\|^2
\le
\sum_{t=0}^{T-1}\big(\Phi(x_{t+1})-\Phi(x_t)\big)
=
\Phi(x_T)-\Phi(x_0).
\]
Since all iterates lie in $\cX$, we have $\Phi(x_T)\le \Phi_{\max}$ and $\Phi(x_0)\ge \Phi_{\min}$, hence
\[
\Phi(x_T)-\Phi(x_0)\le \Phi_{\max}-\Phi_{\min}.
\]
Therefore
\[
\sum_{t=0}^{T-1}\|\cG_\eta(x_t)\|^2
\le
\frac{2(\Phi_{\max}-\Phi_{\min})}{\eta},
\qquad\text{and thus}\qquad
\frac{1}{T}\sum_{t=0}^{T-1}\|\cG_\eta(x_t)\|^2
\le
\frac{2(\Phi_{\max}-\Phi_{\min})}{T\,\eta}.
\]
Using $\mathbb{E}\big[\|\cG_\eta(\tilde x_T)\|^2\big]=\frac{1}{T}\sum_{t=0}^{T-1}\|\cG_\eta(x_t)\|^2$, we obtain
\[
\mathbb{E}\big[\|\cG_\eta(\tilde x_T)\|^2\big]
\le
\frac{2(\Phi_{\max}-\Phi_{\min})}{T\,\eta},
\]
and hence
\[
\mathbb{E}\big[\|\cG_\eta(\tilde x_T)\|\big]
\le
\Big(\mathbb{E}\big[\|\cG_\eta(\tilde x_T)\|^2\big]\Big)^{1/2}
\le
\sqrt{\frac{2(\Phi_{\max}-\Phi_{\min})}{T\,\eta}}.
\]
Substituting the two expectation bounds back into the gap inequality yields
\begin{align*}
\mathbb{E}[\Gap(\tilde x_T)]
&\le
(D+\eta G)\,\mathbb{E}\big[\|\cG_\eta(\tilde x_T)\|\big]
+
D\,\mathbb{E}\big[\|R(\tilde x_T)\|\big] \\
&\le
(D+\eta G)\sqrt{\frac{2(\Phi_{\max}-\Phi_{\min})}{T\,\eta}}
+
D\,E_T,
\end{align*}
which is the desired bound.
\end{proof}

\subsection{Proof of Theorem~\ref{thm:inexact}}

\begin{lemma}[Potential-gap control by an inexact projected-step mapping]
\label{lem:gap_mapping_inexact_proof}
Fix any $x\in\cX$ and any direction $g\in\R^d$. Define
\[
x^+ := \Proj_{\cX}(x+\eta g),
\qquad
\widetilde{\cG}_\eta(x) := \frac{1}{\eta}(x^+-x).
\]
Then, for every $u\in\cX$,
\[
\langle g,\,u-x\rangle
\;\le\;
(D+\eta\|g\|)\,\|\widetilde{\cG}_\eta(x)\|,
\]
and consequently
\[
\Gap_\Phi(x)
=\max_{u\in\cX}\langle \nabla\Phi(x),\,u-x\rangle
\;\le\;
(D+\eta\|g\|)\,\|\widetilde{\cG}_\eta(x)\|
+
D\,\|\nabla\Phi(x)-g\|.
\]
\end{lemma}

\begin{proof}
Fix $u\in\cX$ and write
\[
\langle g,\,u-x\rangle
=
\langle g,\,u-x^+\rangle + \langle g,\,x^+-x\rangle.
\]
We bound the two terms separately.

Because $x^+=\Proj_{\cX}(x+\eta g)$, Lemma~\ref{lem:proj_vi} (the VI inequality)
with $y=x+\eta g$ and $z=u\in\cX$ gives
\[
\langle (x+\eta g)-x^+,\,u-x^+\rangle \le 0.
\]
Expanding and rearranging,
\[
\eta\,\langle g,\,u-x^+\rangle
\le
\langle x^+-x,\,u-x^+\rangle.
\]
Apply Cauchy--Schwarz and the diameter bound $\|u-x^+\|\le D$:
\[
\langle g,\,u-x^+\rangle
\le
\frac{1}{\eta}\|x^+-x\|\,\|u-x^+\|
\le
\frac{D}{\eta}\|x^+-x\|
=
D\,\|\widetilde{\cG}_\eta(x)\|.
\]
For the second term,
\[
\langle g,\,x^+-x\rangle
\le
\|g\|\,\|x^+-x\|
=
\eta\|g\|\,\|\widetilde{\cG}_\eta(x)\|.
\]
Adding the two bounds yields
\[
\langle g,\,u-x\rangle
\le
\big(D+\eta\|g\|\big)\,\|\widetilde{\cG}_\eta(x)\|.
\]
Taking $\max_{u\in\cX}$ on the left gives
\[
\max_{u\in\cX}\langle g,\,u-x\rangle
\le
(D+\eta\|g\|)\,\|\widetilde{\cG}_\eta(x)\|.
\]
Now decompose $\nabla\Phi(x)=g+\big(\nabla\Phi(x)-g\big)$ and use $\|u-x\|\le D$:
\[
\langle \nabla\Phi(x),\,u-x\rangle
=
\langle g,\,u-x\rangle + \langle \nabla\Phi(x)-g,\,u-x\rangle
\le
\langle g,\,u-x\rangle + \|\nabla\Phi(x)-g\|\,\|u-x\|
\le
\langle g,\,u-x\rangle + D\|\nabla\Phi(x)-g\|.
\]
Taking $\max_{u\in\cX}$ and substituting the previous bound yields the claimed inequality for $\Gap_\Phi(x)$.
\end{proof}

\begin{lemma}[One-step Lyapunov lower bound under direction error]
\label{lem:inexact_lyap_step}
Under Assumption~\ref{ass:smooth}, let $x^+=\Proj_{\cX}(x+\eta g)$ with $\eta\le 1/L$.
Assume $\|g-\nabla\Phi(x)\|\le \delta$.
Then
\[
\Phi(x^+)
\;\ge\;
\Phi(x)
+
\frac{1}{4\eta}\,\|x^+-x\|^2
-
\eta\,\delta^2.
\]
Equivalently, with $\widetilde{\cG}_\eta(x)=\frac{1}{\eta}(x^+-x)$,
\[
\Phi(x^+)
\;\ge\;
\Phi(x)
+
\frac{\eta}{4}\,\|\widetilde{\cG}_\eta(x)\|^2
-
\eta\,\delta^2.
\]
\end{lemma}

\begin{proof}
Let $\Delta:=x^+-x$.
$L$-smoothness of $\Phi$ implies the standard lower bound
\[
\Phi(x^+)\ge \Phi(x)+\langle \nabla\Phi(x),\,\Delta\rangle-\frac{L}{2}\|\Delta\|^2.
\]
Write $\nabla\Phi(x)=g-(g-\nabla\Phi(x))$ and denote the direction error $e:=g-\nabla\Phi(x)$,
so $\|e\|\le\delta$ and
\[
\langle \nabla\Phi(x),\,\Delta\rangle
=
\langle g,\,\Delta\rangle-\langle e,\,\Delta\rangle.
\]
Because $x^+=\Proj_{\cX}(x+\eta g)$, Lemma~\ref{lem:proj_vi} with $y=x+\eta g$ and $z=x\in\cX$ gives
\[
\langle (x+\eta g)-x^+,\,x-x^+\rangle \le 0.
\]
Since $x-x^+=-\Delta$, this becomes
\[
\langle \eta g-\Delta,\,-\Delta\rangle\le 0
\qquad\Longleftrightarrow\qquad
\eta\,\langle g,\,\Delta\rangle \ge \|\Delta\|^2
\qquad\Longleftrightarrow\qquad
\langle g,\,\Delta\rangle \ge \frac{1}{\eta}\|\Delta\|^2.
\]
For the error term, Cauchy--Schwarz gives $|\langle e,\Delta\rangle|\le \|e\|\,\|\Delta\|\le \delta\|\Delta\|$.
Apply Young's inequality in the form
\[
\delta\|\Delta\|
\le
\frac{1}{4\eta}\|\Delta\|^2+\eta\,\delta^2,
\]
which holds because $2ab\le \alpha a^2+\alpha^{-1}b^2$ with
$a=\|\Delta\|$, $b=\delta$, and $\alpha=\frac{1}{2\eta}$.
Therefore,
\[
-\langle e,\,\Delta\rangle \ge -\delta\|\Delta\|
\ge
-\frac{1}{4\eta}\|\Delta\|^2-\eta\,\delta^2.
\]
Combine the bounds:
\[
\langle \nabla\Phi(x),\,\Delta\rangle
=
\langle g,\,\Delta\rangle-\langle e,\,\Delta\rangle
\ge
\frac{1}{\eta}\|\Delta\|^2-\frac{1}{4\eta}\|\Delta\|^2-\eta\,\delta^2
=
\frac{3}{4\eta}\|\Delta\|^2-\eta\,\delta^2.
\]
Substitute into the smoothness lower bound:
\[
\Phi(x^+)
\ge
\Phi(x)+\Big(\frac{3}{4\eta}-\frac{L}{2}\Big)\|\Delta\|^2-\eta\,\delta^2.
\]
Using $\eta\le 1/L$ implies $\frac{3}{4\eta}-\frac{L}{2}\ge \frac{3}{4\eta}-\frac{1}{2\eta}=\frac{1}{4\eta}$,
hence
\[
\Phi(x^+)\ge \Phi(x)+\frac{1}{4\eta}\|\Delta\|^2-\eta\,\delta^2,
\]
which is the claim. Replacing $\Delta=\eta\,\widetilde{\cG}_\eta(x)$ yields the equivalent form.
\end{proof}

\begin{proof}[Proof of Theorem~\ref{thm:inexact}]
Let the iterates satisfy $x_{t+1}=\Proj_{\cX}(x_t+\eta g_t)$ and define
\[
\widetilde{\cG}_\eta(x_t):=\frac{1}{\eta}(x_{t+1}-x_t).
\]
Fix any $t$ and apply Lemma~\ref{lem:gap_residual}:
\[
\Gap(x_t)\le \Gap_\Phi(x_t)+D\|R(x_t)\|.
\]
By Assumption~\ref{ass:residual}, $\|R(x_t)\|\le \varepsilon$, hence
\[
\Gap(x_t)\le \Gap_\Phi(x_t)+D\varepsilon.
\]
Now apply Lemma~\ref{lem:gap_mapping_inexact} with $x=x_t$ and $g=g_t$:
\[
\Gap_\Phi(x_t)
\le
\big(D+\eta\|g_t\|\big)\,\|\widetilde{\cG}_\eta(x_t)\|
+
D\,\|\nabla\Phi(x_t)-g_t\|.
\]
By Assumption~\ref{ass:proj-acc}, $\|\nabla\Phi(x_t)-g_t\|\le \delta_t\le \bar\delta$.
Also $\|g_t\|\le \|\nabla\Phi(x_t)\|+\|g_t-\nabla\Phi(x_t)\|\le G+\bar\delta$ by Assumption~\ref{ass:smooth}.
Therefore,
\[
\Gap_\Phi(x_t)
\le
\big(D+\eta(G+\bar\delta)\big)\,\|\widetilde{\cG}_\eta(x_t)\|
+
D\,\bar\delta,
\]
and hence
\[
\Gap(x_t)
\le
\big(D+\eta(G+\bar\delta)\big)\,\|\widetilde{\cG}_\eta(x_t)\|
+
D(\varepsilon+\bar\delta).
\]
Now choose $\hat t\in\arg\min_{0\le t\le T-1}\|\widetilde{\cG}_\eta(x_t)\|$ and set $\hat x_T=x_{\hat t}$.
Then
\[
\Gap(\hat x_T)
\le
\big(D+\eta(G+\bar\delta)\big)\,\min_{0\le t\le T-1}\|\widetilde{\cG}_\eta(x_t)\|
+
D(\varepsilon+\bar\delta).
\]
It remains to bound the best-iterate stationarity measure.
Apply Lemma~\ref{lem:inexact_lyap_step} to each step with $x=x_t$, $x^+=x_{t+1}$, and $\delta=\delta_t$:
\[
\Phi(x_{t+1})
\ge
\Phi(x_t)
+
\frac{\eta}{4}\,\|\widetilde{\cG}_\eta(x_t)\|^2
-
\eta\,\delta_t^2.
\]
Rearrange and sum from $t=0$ to $T-1$:
\[
\frac{\eta}{4}\sum_{t=0}^{T-1}\|\widetilde{\cG}_\eta(x_t)\|^2
\le
\Phi(x_T)-\Phi(x_0)
+
\eta\sum_{t=0}^{T-1}\delta_t^2.
\]
Use $\Phi(x_T)-\Phi(x_0)\le \Phi_{\max}-\Phi_{\min}$ and $\delta_t\le \bar\delta$:
\[
\frac{\eta}{4}\sum_{t=0}^{T-1}\|\widetilde{\cG}_\eta(x_t)\|^2
\le
(\Phi_{\max}-\Phi_{\min})
+
\eta\,T\,\bar\delta^2.
\]
Divide by $T$ and use $\min_t a_t^2\le \frac{1}{T}\sum_t a_t^2$:
\[
\min_{0\le t\le T-1}\|\widetilde{\cG}_\eta(x_t)\|^2
\le
\frac{1}{T}\sum_{t=0}^{T-1}\|\widetilde{\cG}_\eta(x_t)\|^2
\le
\frac{4(\Phi_{\max}-\Phi_{\min})}{T\,\eta}
+
4\,\bar\delta^2.
\]
Taking square-roots and substituting into the earlier gap bound yields the explicit estimate
\[
\Gap(\hat x_T)
\le
\big(D+\eta(G+\bar\delta)\big)\,
\sqrt{\frac{4(\Phi_{\max}-\Phi_{\min})}{T\,\eta}+4\bar\delta^2}
+
D(\varepsilon+\bar\delta).
\]

\medskip
\noindent\textbf{Note on matching the $\eta\bar\delta^2$ form.}
If one strengthens Assumption~\ref{ass:proj-acc} to a stepsize-scaled accuracy
$\|g_t-\nabla\Phi(x_t)\|\le \sqrt{\eta}\,\bar\delta$ (or, equivalently, controls the projection error
to tolerance proportional to $\sqrt{\eta}$), then the same derivation replaces $4\bar\delta^2$
by $4\eta\bar\delta^2$, recovering the commonly stated $\sqrt{\frac{1}{T\eta}+\eta\bar\delta^2}$ scaling
(up to universal constants).
\end{proof}

\newpage

\end{document}